\documentclass[11pt]{article}
\usepackage[dvips,letterpaper,margin=1in]{geometry}
\usepackage{authblk}
\usepackage[utf8]{inputenc} 
\usepackage[T1]{fontenc}    
\usepackage{url}            
\usepackage{booktabs}       
\usepackage{amsfonts}       
\usepackage{nicefrac}       
\usepackage{microtype}      
\usepackage[dvipsnames]{xcolor}
\usepackage[colorlinks = true,
            linkcolor = purple,
            urlcolor  = orange,
            citecolor = teal,
            anchorcolor = blue]{hyperref}

\newcommand{\iid}{i.i.d.\xspace}
\newcommand{\iiddistr}{{\stackrel{\text{\iid}}{\sim}}}
\newcommand{\inner}[1]{\langle #1 \rangle}

\title{On the Cryptographic Hardness\\ of Learning Single Periodic Neurons}

\author[a]{Min Jae Song\footnote{Equal contribution.}}
\author[b]{Ilias Zadik$^*$}
\author[a,b]{Joan Bruna}
\affil[a]{Courant Institute of Mathematical Sciences, New York
  University, New York}
\affil[b]{Center for Data Science, New York University, New York}
\usepackage[utf8]{inputenc}
\usepackage{amsmath}
\usepackage{amstext}
\usepackage{amssymb}
\usepackage{amsthm}
\usepackage{amsfonts}
\usepackage{bbm}
\usepackage{bm}
\usepackage{graphicx}
\usepackage{caption}
\usepackage{subcaption}
\usepackage{booktabs}
\usepackage{multirow}
\usepackage[ruled,vlined]{algorithm2e}

\usepackage[
backend=biber,
style=alphabetic,
citestyle=alphabetic,
giveninits=false,
maxnames=99,
date=year
]{biblatex}
\AtEveryBibitem{%
\ifentrytype{article}{
    \clearfield{url}%
    \clearfield{urldate}%
    \clearfield{eprint}%
    \clearfield{review}%
    \clearfield{location}%
    \clearfield{address}%
    \clearfield{issn}%
    \clearfield{isbn}%
    \clearfield{doi}%
    \clearfield{series}
    }{}%
\ifentrytype{inproceedings}{
    \clearlist{publisher}%
    \clearlist{location}%
    \clearfield{url}%
    \clearfield{urldate}%
    \clearfield{issn}%
    \clearfield{isbn}%
    \clearfield{doi}%
    \clearfield{series}
}{}%
\ifentrytype{book}{
    \clearfield{url}%
    \clearfield{urldate}%
    \clearfield{review}%
    \clearfield{location}%
    \clearfield{address}%
    \clearfield{issn}%
    \clearfield{isbn}%
    \clearfield{doi}%
    \clearfield{series}
}{}%
\ifentrytype{inbook}{
    \clearfield{url}%
    \clearfield{urldate}%
    \clearfield{review}%
    \clearfield{location}%
    \clearfield{address}%
    \clearfield{issn}%
    \clearfield{isbn}%
    \clearfield{doi}%
    \clearfield{eprint}%
    \clearfield{series}
}{}%
\ifentrytype{incollection}{
    \clearfield{url}%
    \clearfield{urldate}%
    \clearfield{review}%
    \clearfield{location}%
    \clearfield{address}%
    \clearfield{issn}%
    \clearfield{isbn}%
    \clearfield{doi}%
    \clearfield{series}
}{}
}

\DeclareFieldFormat[inproceedings]{title}{#1} 
\DeclareFieldFormat[article]{title}{#1}
\DeclareFieldFormat[inbook]{title}{#1}
\DeclareFieldFormat[misc]{title}{#1}
\DeclareFieldFormat[article]{year}{#1}
\DeclareFieldFormat[book]{title}{#1}

\renewbibmacro{in:}{}
\renewbibmacro*{volume+number+eid}{%
  \printfield{volume}
  \setunit{\addcolon\space}%
  \printfield{number}%
  \printfield{eid}}

\newtheorem{theorem}{Theorem}[section]
\newtheorem{lemma}[theorem]{Lemma}

\newtheorem{assumption}[theorem]{Assumption}
\newtheorem{proposition}[theorem]{Proposition}
\newtheorem{corollary}[theorem]{Corollary}
\newtheorem{conjecture}[theorem]{Conjecture}
\newtheorem{remark}[theorem]{Remark}

\newtheorem{definition}[theorem]{Definition}
\newtheorem{claim}[theorem]{Claim}

\newcommand{\RR}{\mathbb{R}}
\newcommand{\NN}{\mathbb{N}}

\newcommand{\PP}{\mathbb{P}}
\newcommand{\ZZ}{\mathbb{Z}}
\newcommand{\sA}{\mathcal{A}}

\newcommand{\sC}{\mathcal{C}}

\newcommand{\sF}{\mathcal{F}}

\newcommand{\sK}{\mathcal{K}}
\newcommand{\sN}{\mathcal{N}}

\newcommand{\sS}{\mathcal{S}}

\newcommand{\eps}{\epsilon}

\DeclareMathOperator*{\EE}{\mathbb{E}}

\newcommand{\one}{\mathop{\mathbbm{1}}}

\newcommand{\sgn}{\mathrm{sgn}}

\newcommand{\poly}{\mathsf{poly}}

\newcommand{\bw}{\bm w}

\newcommand{\clwe}{\mathrm{CLWE}}

\renewcommand{\Pr}{\mathbb{P}} 

\begin{document}

\maketitle

\begin{abstract}
We show a simple reduction which demonstrates the cryptographic hardness of learning a single periodic neuron over isotropic Gaussian distributions in the presence of noise. More precisely, our reduction shows that \emph{any} polynomial-time algorithm (not necessarily gradient-based) for learning such functions under small noise implies a polynomial-time quantum algorithm for solving \emph{worst-case} lattice problems, whose hardness form the foundation of lattice-based cryptography. Our core hard family of functions, which are well-approximated by one-layer neural networks, take the general form of a univariate periodic function applied to an affine projection of the data. These functions have appeared in previous seminal works which demonstrate their hardness against gradient-based (Shamir'18), and Statistical Query (SQ) algorithms (Song et al.'17). We show that if (polynomially) small noise is added to the labels, the intractability of learning these functions applies to \emph{all polynomial-time algorithms}, beyond gradient-based and SQ algorithms, under the aforementioned cryptographic assumptions. 

Moreover, we demonstrate the \emph{necessity of noise} in the hardness result by designing a polynomial-time algorithm for learning certain families of such functions under exponentially small adversarial noise. Our proposed algorithm is not a gradient-based or an SQ algorithm, but is rather based on the celebrated Lenstra-Lenstra-Lov\'asz (LLL) lattice basis reduction algorithm. Furthermore, in the absence of noise, this algorithm can be directly applied to solve CLWE detection (Bruna et al.'21) and phase retrieval with an optimal sample complexity of $d+1$ samples. In the former case, this improves upon the quadratic-in-$d$ sample complexity required in (Bruna et al.'21).
\end{abstract}

\newpage
\tableofcontents
\newpage

\section{Introduction}
\label{sec:introduction}

The empirical success of Deep Learning has given an impetus to provide theoretical foundations explaining when and why it is possible to efficiently learn from high-dimensional data with neural networks. Currently, there are large gaps between positive and negative results for learning, even for the simplest neural network architectures \cite{zhong2017recovery,goel2020superpolynomial,brutzkus2017sgd,ge2017learning}.  These gaps offer a large ground for debate, discussing the extent up to which improved learning algorithms can be designed, or whether a fundamental computational barrier has been reached.

One particular challenge in closing these gaps is establishing negative results for improper learning in the distribution-specific setting, in which the learner can exploit the peculiarities of a known input distribution, and is not limited to outputting hypotheses from the target function class. Over the last few years, authors have successfully developed distribution-specific hardness results in the context of learning neural networks, offering different flavors. 
On one hand, there have been several results proving the failure of a restricted class of algorithms, such as gradient-based algorithms~\cite{shamir2018distribution,shalev2017failures}, or more generally Statistical Query (SQ) algorithms~\cite{kearnsSQ1998,feldman2017planted-clique,song2017complexity,goel2020superpolynomial,diakonikolas2020algorithms}. Notably, such results apply to the simplest cases, such as learning one-hidden-layer neural networks over the standard Gaussian input distribution~\cite{goel2020superpolynomial,diakonikolas2020algorithms}. On the other hand, a different line of work has shown the hardness of learning \emph{two}-hidden-layer neural networks for \emph{any} polynomial-time algorithm by leveraging cryptographic assumptions, such as the existence of local pseudorandom generators (PRGs) with polynomial stretch~\cite{daniely2021local}. Despite such significant advances, important open questions remain, such as whether the simpler case of learning one hidden-layer neural network over standard Gaussian input remains hard for algorithms not captured by the SQ framework. To make this question more precise, are non-SQ polynomial-time algorithms, which may inspect individual samples -- such as stochastic gradient descent (SGD)~\cite{abbe2020poly} -- able to learn one-hidden layer neural networks over Gaussian input? Understanding the answer to this question is a partial motivation of the present work.

A key technique for constructing hard-to-learn functions is
leveraging ``high-frequency'' oscillations in high-dimensions. The simplest instance of such functions is given by pure cosines of the form $f(x) = \cos(2\pi \gamma \inner{w,x})$, where we refer to $w \in S^{d-1}$ as its \textit{hidden direction}, and $\gamma$ as its \textit{frequency}. Such functions have already been investigated by previous works~\cite{song2017complexity,shamir2018distribution,shalev2017failures} in the context of lower bounds for learning neural networks. 
For these hard constructions, the frequency 
$\gamma$ is taken to scale polynomially with the dimension $d$. Note that as the univariate function $\cos(2\pi \gamma t)$ is $O(\gamma)$-Lipschitz, the function $f$ is well-approximated by one-hidden-layer ReLU network of $\poly(\gamma)$-width on any compact set (see e.g., Appendix \ref{app:relu-approx}). Hence, understanding the hardness of learning such functions is an unavoidable step towards understanding the
hardness of learning one-hidden-layer ReLU networks.

In this work, we pursue this line of inquiry, focusing on weakly learning the cosine neuron class over the standard Gaussian input distribution in the presence of noise. Our main result is a proof, via a reduction from a fundamental problem in lattice-based cryptography called the Shortest Vector Problem (SVP), that such learning task is hard for \emph{any} polynomial-time algorithm, based on the widely-believed cryptographic assumption that (approximate) SVP is computationally intractable against quantum algorithms (See e.g.,~\cite{regev2005lwe,micciancio2009lattice,ducas2017dilithium,alagic_status_2020} and references therein).
Our result therefore extends the hardness of learning such functions from a restricted family of algorithms, such as gradient-based algorithms or SQ, to \emph{all} polynomial-time algorithms by leveraging cryptographic assumptions. Note, however, that SQ lower bounds are unconditional because they are of an information-theoretic nature. Therefore, our result, which is conditional on a computational hardness assumption, albeit a well-founded one in the cryptographic community, and SQ lower bounds are not directly comparable.

The problem of learning cosine neurons with noise can be studied in the broader context of inferring hidden structures in noisy high-dimensional data, as a particular instance of the family of Generalized Linear Models (GLM) \cite{GLM1, GLMbook}. Multiple inference settings, including, for example, the well-known planted clique model \cite{jerrum1992large, AlonClique}, but also GLMs such as sparse regression \cite{gamarnik2017high} exhibit so-called \emph{computational-to-statistical gaps}. These gaps refer to intervals of signal-to-noise ratio (SNR) values where inference of the hidden structure is possible by exponential-time estimators but appears out of reach for any polynomial-time estimator. Following this line of work, we define the SNR of our cosine neuron learning problem to be the inverse of the noise level, and analyze its hardness landscape.
As it turns out, weakly learning the cosine neuron class provides a rich landscape, yielding a computational-to-statistical gap based on \textit{a worst-case hardness} guarantee. We note that this is in contrast with the ``usual'' study of such gaps where such worst-case hardness guarantees are usually elusive and they are mostly based on the refutation of restricted computational classes, such as Sum-of-Squares~\cite{BarakClique}, low-degree polynomial estimators~\cite{kunisky2019notes}, Belief Propagation~\cite{bandeira2018notes}, or local search methods~\cite{gamarnik2019landscape}.

Finally, we establish an upper bound for the computational threshold, thanks to a polynomial-time algorithm based on the Lenstra-Lenstra-Lovász(LLL) lattice basis reduction algorithm (see details in Section~\ref{LLL_main}). Our proposed algorithm is shown to be highly versatile, in the sense that it can be directly used to solve two seemingly very different GLMs: the CLWE detection problem from cryptography and the phase retrieval problem from high-dimensional statistics. Remarkably, this method bypasses the SQ and gradient-based hardness established by previous works \cite{shamir2018distribution,song2017complexity}. Our use of the LLL algorithm to bypass a previously considered ``computationally-hard'' region adjoins similar efforts to solve linear regression with discrete coefficients \cite{NEURIPS2018_ccc0aa1b, gamarnik2019inference}, \cite[Sec. 4.2.1]{kunisky2019notes}, as well as the \emph{correspondence retrieval} problem \cite{andoni2017correspondence}, which includes phase retrieval as a special case. We show in Section \ref{LLL_main} how our algorithms obtain optimal sample complexity for recovery in both these problems in the noiseless setting. An interesting observation is that in the latter case, the resulting algorithm, and also the very similar LLL-based algorithm by \cite{andoni2017correspondence}, improves upon AMP-based algorithms \cite{Barbier5451} in terms of sample complexity, often thought to be optimal among all polynomial-time algorithms \cite{maillard2020phase}. While our LLL algorithm can be seen as an appropriate modification of \cite{andoni2017correspondence}, our analysis employs different tools, leading to improved guarantees. More precisely, our analysis easily extends to distributions that are both log-concave and sub-Gaussian, as opposed to solely Gaussian in \cite{andoni2017correspondence}. In addition, our algorithm incorporates an explicit rounding step for LLL, which allows us to determine its precise noise-tolerance (see details in Section \ref{sec:phase-retrieval}).

\subsection{Related work}

\paragraph{Hardness of learning from cryptographic assumptions.} Among several previous works~\cite{kearns1994cryptographic,kharitonov1993cryptographic} which leverage cryptographic assumptions to establish hardness of improper learning, most relevant to our results is the seminal work of Klivans and Sherstov~\cite{klivans2009cryptographic} whose hardness results are also based on SVP. To elaborate, they show that learning intersections of halfspaces, which can be seen as neural networks with the threshold activation, is hard based on the worst-case hardness of approximate SVP. Our work differs, though, in several important aspects from theirs. First, and perhaps most importantly, our result holds over the well-behaved Gaussian input distribution over $\RR^d$, whereas their hardness utilizes a non-uniform distribution over the Boolean hypercube $\{0,1\}^d$. Second, at a technical level and in agreement with our continuous input domain and their discrete input domain, we take a different reduction route from SVP. Their link to SVP is the Learning with Errors (LWE) Problem~\cite{regev2005lwe}, whereas our link in the reduction is the recently developed Continuous Learning with Errors (CLWE) Problem~\cite{bruna2020continuous}.
On another front, very recently, \cite{daniely2021local} presented an abundance of novel hardness results in the context of improper learning by assuming the mere \emph{existence} of Local Pseudorandom Generators (LPRGs) with polynomial stretch. While the LPRG and SVP assumptions are not directly comparable, we emphasize that we rely on the worst-case hardness of SVP, whereas LPRG assumes average-case hardness. A worst-case hardness assumption is arguably weaker as it requires only \emph{one} instance to be hard, whereas an average-case hardness assumption requires instances to be hard on average.

\paragraph{Lower bounds against restricted class of algorithms and upper bounds.}
As mentioned previously, a widely adapted method for proving hardness of learning is through SQ lower bounds~\cite{kearnsSQ1998,blum1994sq,szorenyi2009characterizing,feldman2017planted-clique}. Among previous work, most closely related to our work is~\cite{song2017complexity} and~\cite{shamir2018distribution}, who consider learning linear-periodic function classes which contain cosine neurons. By constructing a different class of hard one-hidden-layer networks, stronger SQ lower bounds over the Gaussian distribution, in terms of both query complexity and noise rate, have been established~\cite{goel2020superpolynomial,diakonikolas2020algorithms}. Yet, for technical reasons, the SQ model cannot rule out algorithms such as stochastic gradient descent (SGD), since these algorithms can in principle inspect each sample individually. In fact,~\cite{abbe2020poly} carry this advantage of SGD to the extreme and show that SGD is \emph{poly-time universal}.
\cite{arous2021online} establishes sharp bounds using SGD for weakly learning a single planted neuron, and reveals a fundamental dependency between the regularity of their \textit{dimension-independent} activation function, which they name the ``information exponent'', and the sample complexity. The regularity of the activation function has been leveraged in several works to yield positive learning results~\cite{kakade2011efficient,zhong2017recovery,ge2017learning,soltanolkotabi2017learning,allen2018learning,goel2019timeaccuracy,frei2020agnostic, diakonikolas2020approximation}. Finally, statistical-to-computational gaps using the family of Approximate Message Passing (AMP) algorithms~\cite{donoho2009message,rangan2011generalized} for the algorithmic frontier have been established in various high-dimensional inference settings, including proper learning of certain single-hidden layer neural networks  \cite{aubin2019committee}, spiked matrix-tensor recovery \cite{mannelli2020marvels} and also GLMs \cite{Barbier5451}.

\paragraph{The LLL algorithm and statistical inference problems.}
For our algorithmic results, we employ the LLL algorithm. Specifically, our techniques are originally based on the breakthrough use of the LLL algorithm to solve a class of average-case subset sum problems in polynomial-time, as established first by Lagarias and Odlyzko \cite{Lagarias85} and later via a greatly simplified argument by Frieze \cite{FriezeSubset}. While the power of LLL algorithm is very well established in the theoretical computer science~\cite{shamir1982polynomial,lagarias1984knapsack}, integer programming~\cite{kannan1983improved}, and computational number theory communities (see~\cite{simon2010selected} for a survey), to the best of our knowledge, it has found only a handful of applications in the theory of statistical inference. Nevertheless, a few years ago, a strengthening of the original LLL-based arguments by Lagarias, Odlyzko and Frieze has been used to prove that linear regression with rational-valued hidden vector and continuous features can be solved in polynomial-time given access only to one sample \cite{NEURIPS2018_ccc0aa1b}. This problem has been previously considered ``computationally-hard'' \cite{gamarnik2017high} and is proven to be impossible for the LASSO estimator \cite{Wain2009, gamarnik2019sparse}, greedy local-search methods \cite{gamarnik2017high} and the AMP algorithm \cite{Reeves19}.  In a subsequent work to \cite{NEURIPS2018_ccc0aa1b}, the suggested techniques have been generalized to the linear regression and phase retrieval settings under the more relaxed assumptions of discrete (and therefore potentially irrational)-valued hidden vector \cite{gamarnik2019inference}. Our work is based on insights from \cite{NEURIPS2018_ccc0aa1b, gamarnik2019inference}, but is importantly generalizing the use of the LLL algorithm  (a) for the recovery of an arbitrary unit \textit{continuous-valued} hidden vector and (b) for multiple GLMs such as the cosine neuron, the phase retrieval problem, and the CLWE problem. However, for noiseless phase retrieval, we note that the optimal sample complexity of $d+1$ has previously been achieved by~\cite{andoni2017correspondence} using an LLL-based algorithm very similar to ours.

\subsection{Main Contributions: the Hardness Landscape of Learning Cosine Neurons}
\label{sec:hardness-landscape}

In this work, we thoroughly study the hardness of improperly learning single cosine neurons over isotropic $d$-dimensional Gaussian data. We study them under the existence of a small amount of adversarial noise per sample, call it $\beta \geq 0,$ which we prove is necessary for the hardness to take place. Specifically we study improperly (weakly) learning in the squared loss sense, the function $f(x)=\cos(2\pi \gamma \inner{w,x}),$ for some hidden direction $w \in S^{d-1},$ from $m$ samples of the form $z_i=f(x_i)+\xi_i, i=1,\ldots,m$ where $x_i \iiddistr N(0,I_d)$ and arbitrary $|\xi_i| \leq \beta.$

\paragraph{Information-theoretic bounds under constant noise.} We first address the statistical, or also known as information-theoretic, question of understanding for which noise level $\beta$ one can hope to learn $f(x)$ from polynomially in $d$ many samples, by using computationally unconstrained estimators. Since the range of the functions $f=f_{w}$ is the interval $[-1,1]$ it is a trivial observation that for any $\beta \geq 1$ learning is impossible. This follows because the (adversarial) noise could then produce always the uninformative case where $z_i=0$ for all $i=1,\ldots,m.$ 
    
Our first result (see Section \ref{sec:IT} for details), is a design and analysis of an algorithm which runs in $O(\exp( d \log(\gamma/\beta)))$ time and satisfies the following property. For any $\beta$ smaller than a \textit{sufficiently small constant}, the output hypothesis of the algorithm learns the function $f$ with access to $O( d \log(\gamma/\beta))$ samples, with high probability. To the best of our knowledge, such an information-theoretic result has not appeared before in the literature of learning a single cosine neuron. We consider this result essential and reassuring as it implies that the learning task is statistically achievable if $\beta$ is less than a small constant. Therefore, any hardness claim in terms of polynomial-time algorithms aiming to learn this function class is meaningful and implies an essential computational barrier.
     
\paragraph{Cryptographic hardness under moderately small noise.} 
Our second and main result, presented in Section \ref{crypto_main}, is a reduction establishing that (weakly) learning this function class under any $\beta$ which scales at least inverse polynomially with $d$, i.e. $\beta \geq d^{-C}$ for some constant $C>0,$ is as hard as a worst-case lattice problem on which the security of lattice-based cryptography is based on.
    
\begin{theorem}[Informal]
  Consider the function class $\sF_{\gamma} = \{f_{\gamma,w}(x) = \cos(2\pi \gamma \langle x, w \rangle) \mid w \in \sS^{d-1} \}$. Weakly learning $\sF_\gamma$ over Gaussian inputs $x \sim N(0,I_d)$ under any inverse-polynomial adversarial noise when $\gamma \geq 2\sqrt{d}$ and $\beta=1/\mathsf{poly}(d)$, is hard, assuming worst-case lattice problems are secure against quantum attacks.
\end{theorem}
The exact sense of cryptographic hardness used is that weakly learning the single cosine neuron under the described assumptions, reduces to solving a \emph{worst-case} lattice problem, known as the approximate Shortest Vector Problem (SVP). The approximation factor of SVP obtained in our reduction, is not known to be NP-hard~\cite{aharonov2005conp}, but it is widely believed to be computationally hard against any polynomial-time algorithm, including quantum algorithms~\cite{micciancio2009lattice}.
The reduction makes use of a recently developed average-case detection problem, called Continuous Learning with Errors (CLWE)~\cite{bruna2020continuous} which has been established to be hard under the same hardness assumption on SVP.
Our reduction shows that weakly learning the single cosine neuron in polynomial time, implies a polynomial-time algorithm for solving the CLWE problem (see Section \ref{sec:definitions-and-notations} for the definition). The link here between the two settings comes from the periodicity of the cosine function, and the fact that the CLWE has an appropriate $\mathrm{mod} \ 1$ structure, as well. 

Interestingly, our reduction works for any class of function $g(x)=\phi(\gamma \inner{w,x})$ where $\phi$ is a  1-periodic and $O(1)$-Lipschitz function and under $\gamma \geq 2 \sqrt{d},$ generalizing the hardness claim much beyond the single cosine neuron. Moreover, our reduction shows that the computational hardness in fact applies to a certain position-dependent \emph{random} noise model, instead of bounded \emph{adversarial} noise (See Remark~\ref{rem:random-noise}). Lastly, as mentioned above, we highlight that this is a (conditional) lower bound against \emph{any} polynomial-time estimator, not just SQ or gradient-based methods.

\paragraph{Polynomial-time algorithm under exponentially small noise.}
We finally address the question of whether there is some polynomial-time algorithm that can weakly learn the single cosine neuron, in the presence of potentially exponentially small noise. Notably, the current lower bound with respect to SQ \cite{song2017complexity} or gradient based methods \cite{shamir2018distribution} apply without any noise assumption \emph{per-sample}, raising the suspicion that no ``standard'' learning method works even in the case $\beta=0.$

We design and analyze an 
algorithm for the  single cosine neuron which provably succeeds in learning the function $f=f_{w}$ when  $\beta \leq \exp(-\tilde{O}(d^3))$, and with access to only $d+1$ samples. Note that this sample complexity is perhaps surprising: one needs \textit{only one more sample} than just receiving the samples in the ``pure'' linear system form $\inner{w,x_i}$ instead of $\cos(2\pi\gamma \inner{w,x_i})+\xi_i$. The algorithm comes from reducing the problem to an integer relation detection question and then make a careful use of the powerful Lenstra-Lenstra-Lov\'asz (LLL) lattice basis reduction algorithm \cite{lenstra1982factoring} to solve it in polynomial time. The integer relation detection allows us to recover the (unknown) integer periods naturally occuring because of the periodicity of the cosine, which then allows us to provably ``invert'' the cosine, and then learn the hidden direction $w$ simply by solving a linear system and then the function. 

The LLL algorithm is a celebrated algorithm in theoretical computer science and mathematics, which has rarely been used in the learning literature (with the notable recent exceptions \cite{ andoni2017correspondence, NEURIPS2018_ccc0aa1b, gamarnik2019inference}). We consider our connection between learning the single cosine neuron, integer relation detection and the LLL algorithm, a potentially interesting algorithmic novelty of the present work. 
We note that~\cite{bruna2020continuous} likewise use the LLL algorithm to solve CLWE in the noiseless setting. When applied to CLWE, our algorithm, via a significantly more involved application of the LLL algorithm and careful analysis, improves upon their algorithm in terms of both sample complexity and noise-tolerance.

\paragraph{Application to noiseless phase retrieval: $d+1$ samples suffice.}

Notice that the cosine activation function loses information in two distinct steps: first it ``loses'' the sign, since it is an even function, and then it ``loses'' localisation beyond its period (fixed at $[-1/2,1/2)$). As a result, any algorithm learning the cosine neuron (such as our proposed LLL-based algorithm) can be immediately extended to solve the two separate cases, where one only loses the sign (which is known as the phase retrieval problem in high dimensional statistics) or only the localisation (which is known as the CLWE problem in cryptography). In particular, the noiseless cosine learning problem `contains' the phase retrieval problem, where one is asked to recover an unknown vector $w$ from measurements $|\langle x_i, w \rangle |$, since $\cos( 2\pi \gamma \langle w, x_i \rangle) = \cos( 2\pi \gamma | \langle x_i, w \rangle |)$. 
Therefore, as an immediate consequence of our algorithmic results, we achieve the optimal sample complexity of noiseless\footnote{Or exponentially small noise; see Corollary \ref{cor:phase-recovery} for the precise statement} phase retrieval. As mentioned previously, this algorithmic result, while interesting and a consequence of our analysis for cosine learning, has already been established in the prior work  \cite{andoni2017correspondence} using a very similar LLL-based algorithm.

We note that achieving in polynomial-time the optimal sample complexity is perhaps of independent interest from a pure algorithm design point of view. While Gaussian elimination can trivially solve for $w$ given $d$ samples of the form $\inner{x_i,w}$ where $x_i \iiddistr N(0,I_d)$, our algorithm shows that by ``losing'' the sign of $\inner{x_i,w}$ one needs only one sample more to recover again $w$ in polynomial-time. However, we remark that the LLL algorithm has a running time of $O(d^6\log^3 M)$~\cite{nguyen2009lll}\footnote{The $L^2$ algorithm by~\cite{nguyen2009lll} speeds up LLL using floating-point arithmetic, but the running time still grows faster than $O(d^5)$.}, where $d$ is the dimension and $M$ is the maximum $\ell_2$-norm of the given lattice basis vectors, which can relatively quickly become computationally challenging with increasing dimension despite its polynomial time complexity. 
We refer the reader to Section~\ref{sec:phase-retrieval} for a formal statement of the phase retrieval problem and our results.

\begin{figure}
    \centering
    \includegraphics[width=0.8\textwidth]{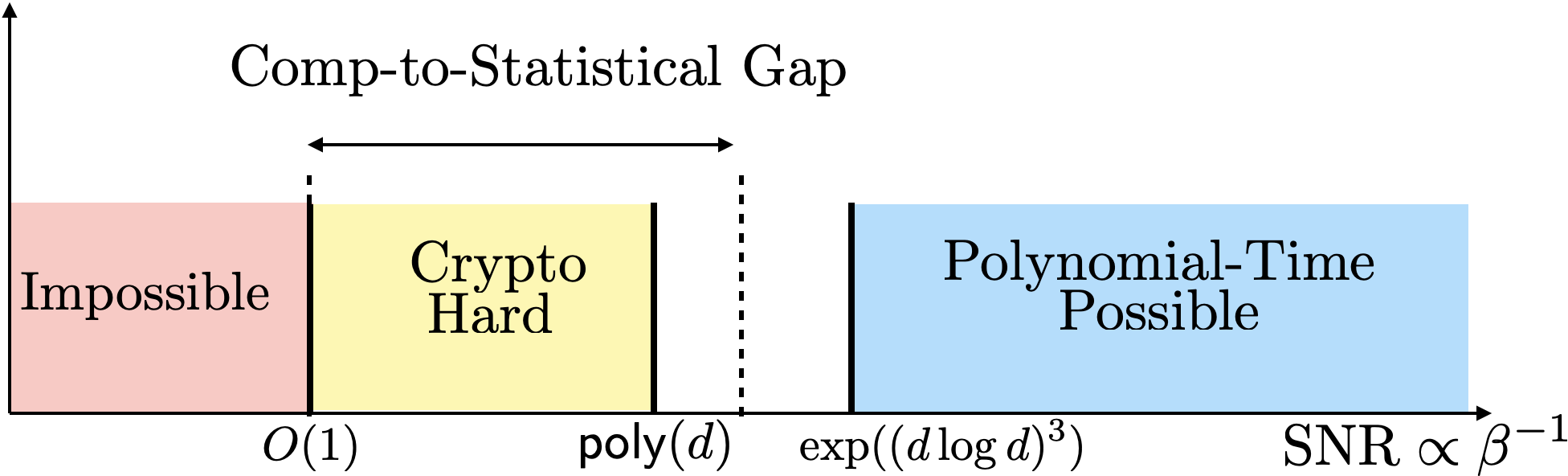}
    \caption{Our results at a glance for weakly learning the class $\sF_{\gamma}$. Section \ref{sec:IT} describes information-theoretical limits, Section \ref{crypto_main} presents the reduction from CLWE, while Section \ref{LLL_main} introduces an efficient algorithm based on LLL.}
    \label{fig:hardness-landscape}
\end{figure}

\subsection{Future Directions}

Our results heavily rely on the specific nature of the periodic activation function, so a natural question is to which extent our results can be extended beyond the single periodic neuron class. 
\begin{itemize}
    \item For lower bounds, a challenging but very interesting generalization would be to establish the cryptographic-hardness of learning certain family of GLMs whose activation function does not need to be periodic. A potentially easier route forward on this direction, would be to consider the Hermite decomposition of the activation function, similar to~\cite{arous2021online}, and establish lower bounds on the performance of low-degree methods ~\cite{kunisky2019notes}, of SGD~\cite{arous2021online}, or of local search methods methods~\cite{gamarnik2019landscape}, for activation functions whose low-degree Hermite coefficients are exponentially small. 

   \item For upper bounds, we believe that our proposed LLL-based algorithm may be extended beyond learning even periodic activation functions, such as the cosine activation, by appropriately post-processing the measurements, but leave this for future work. Furthermore, it would be interesting to better understand (empirically or analytically) the noise tolerance of our LLL-based algorithm for ``low-frequency'' activation functions, such as the absolute value underlying the phase retrieval problem which has ``zero'' frequency.
\end{itemize}

\section{Definitions and Notations}
\label{sec:definitions-and-notations}
\paragraph{Distribution-specific PAC-learning.} We consider the problem of learning a sequence of real-valued function classes $\{\sF_d\}_{d \in \NN}$, each over the standard Gaussian input distribution on $\RR^d$, an instance of what is called distribution-specific PAC learning~\cite{kharitonov1993cryptographic,shamir2018distribution}. The input is a multiset of i.i.d.~labeled examples $(x,y) \in \RR^d \times \RR$, where $x \sim N(0,I_d)$, $y = f(x) + \xi$, $f \in \sF_d$, and $\xi \in \RR$ is some type of observation noise. We denote by $D=D_f$ the resulting data distribution. The goal of the learner is to output an hypothesis $h: \mathbb{R}^d \rightarrow \mathbb{R}$ that is close to the target function $f$ in the squared loss sense over the Gaussian input distribution. We say a learning algorithm is \emph{proper} if it outputs an hypothesis $h \in \sF_d$. On the other hand, we say a learning algorithm is \emph{improper} if $h$ is not necessarily in $\sF_d$~\cite{shalevshwartz2014understanding}. We omit the index $d$, when the input dimension is clear from the context.

We denote by $\ell: \RR \times \RR \rightarrow \RR_{\ge 0}$ the squared loss function defined by $\ell(y,z) = (y-z)^2$. For a given hypothesis $h$ and a data distribution $D$ on pairs $(x,z) \in \mathbb{R}^d \times \mathbb{R}$, we define its \emph{population loss} $L_D(h)$ over a data distribution $D$ by
\begin{align}\label{pop_loss_gen}
    L_D(h) = \EE_{(x,y) \sim D} [\ell(h(x),y)] \;. 
\end{align}

\begin{definition}[Weak learning]
\label{def:weak-learning}
Let $\eps=\eps(d) > 0$ be a sequence of numbers, $\delta \in (0,1)$ a fixed constant, and let $\{\sF_d\}_{d \in \NN}$ be a sequence of function classes defined on input space $\RR^d$. We say that a (randomized) learning algorithm $\sA$ $\eps$-weakly learns $\{\sF_d\}_{d \in \NN}$ over the standard Gaussian distribution if for every $f \in \sF_d$ the algorithm outputs a hypothesis $h_d$  such that for large values of d with probability at least $1-\delta$
\begin{align*}
    L_{D_f}(h_d) \le L_{D_f}(\EE[f(x)])-\eps\;.
\end{align*}
Note that $\EE[f(x)]$ is the best constant predictor for the data distribution $D=D_f$. Hence, we refer to $L_D(\EE[f(x)])=\mathrm{Var}_{Z \sim N(0,I_d)}(f(Z)),$ as the trivial loss, and $\eps$ as the edge of the learning algorithm.
\end{definition} From simplicity, we refer to an hypothesis as \textit{weakly learning} a function class if it can achieve edge $\epsilon$ which is depending inverse polynomially in $d$.

\paragraph{Periodic Neurons.} Let $\gamma = \gamma(d) > 1$ be a sequence of numbers indexed by the input dimension $d \in \NN$, and let $\phi: \RR \rightarrow [-1,1]$ be an 1-periodic function. We denote by $\sF_\gamma^\phi$ the function class
\begin{align}
\label{periodic-neurons}
    \sF_\gamma^\phi = \{f: \RR^d \rightarrow [-1,1] \mid f(x) = \phi(\gamma \langle w, x \rangle), w \in S^{d-1}\}
\end{align}
Note that each member of the function class $\sF_\gamma^\phi$ is fully characterized by a unit vector $w \in S^{d-1}$. We refer such function classes as \emph{periodic neurons}.

\paragraph{Cosine Learning.} We define the \textit{cosine distribution} on dimension $d$ with frequency $\gamma = \gamma(d)$, adversarial noise rate $\beta=\beta(d)$, and hidden direction $w \in S^{d-1}$ to be the distribution of samples of the form $(x, z) \in \RR^d \times \RR$, where $x \iiddistr N(0,I_d)$, some bounded adversarial noise $|\xi| \le \beta$, and
\begin{align}
\label{cosine-dist}
     z = \cos(2\pi \gamma \langle w, x \rangle) + \xi.
\end{align}The cosine learning problem consists of weakly learning the cosine distribution, per Definition~\ref{def:weak-learning}. This learning problem is the central subject of our analysis. Hence, we slightly abuse notation and denote the corresponding cosine function class by \begin{align} \label{eq: Fcosine}
\sF_\gamma = \{\cos(2\pi \gamma \langle w, x \rangle) \mid w \in S^{d-1}\}. 
\end{align}

\paragraph{Continuous Learning with Errors (CLWE)~\cite{bruna2020continuous}.}
We define the CLWE distribution $A_{w,\beta,\gamma}$ on dimension $d$ with frequency $\gamma=\gamma(d) \ge 0$, and noise rate $\beta = \beta(d) \ge 0$ to be the distribution of i.i.d. samples of the form $(x,z)\in \mathbb{R}^d \times [-1/2,1/2)$ where for independent $x \sim N(0,I_d) , \xi \sim N(0,\beta)$ and 
\begin{align}
z = \gamma \inner{x,w}+\xi \mod 1~. \label{CLWE}
\end{align}
Note that for the $\mathrm{mod}$ 1 operation, we take the representatives in $[-1/2,1/2)$.
The CLWE problem consists of detecting between i.i.d. samples from the CLWE distribution or an appropriate null distribution. In the context of CLWE, we refer to the distribution $N(0,I_d) \times U([-1/2,1/2))$ as the null distribution and denote it by $A_0$.
Given $\gamma=\gamma(d)$ and $\beta=\beta(d)$, we consider a sequence of decision problems $\{\clwe_{\beta,\gamma}\}_{d\in\NN}$, indexed by the input dimension $d$, in which the learner is given samples from an unknown distribution $D$ such that either $D \in \{A_{w,\beta,\gamma} \mid w \in S^{d-1}\}$, and $D=A_0$. The algorithm is asked to decide whether $D \in \{A_{w,\beta,\gamma}\mid w \in S^{d-1}\}$ or $D=A_0$ in polynomial-time. Under this setup, we define the \emph{advantage} of an algorithm as the difference between the probability that it correctly detects samples from $D \in \{A_{w,\beta,\gamma}\mid w \in S^{d-1}\}$, and the probability that errs (decides ``$D \neq A_0$'') given samples from $D = A_0$. We call the advantage \emph{negligible} if it decays superpolynomially. For a more detailed setup of this problem, we refer the reader to Appendix~\ref{app:formal-setup}.

Bruna et al.~\cite{bruna2020continuous} showed worst-case evidence that the CLWE problem is computationally hard even with inverse polynomial noise rate $\beta$ if $\gamma \ge 2\sqrt{d}$, despite its seemingly mild requirement of non-negligible advantage. In fact, their evidence of computational intractability is based on \emph{worst-case} lattice problems called the Shortest Vector Problem (SVP)~\cite{micciancio2002complexity}. In particular, they showed that distinguishing a \emph{typical} CLWE distribution, where the randomness is over the uniform choice of hidden direction $w \in S^{d-1}$, from the null distribution is as hard as solving the \emph{worst} instance of approximate SVP. For a formal definition of the approximate SVP, we refer the reader to Appendix~\ref{app:formal-setup}, but note that the (quantum) worst-case hardness of this lattice problems is widely-believed by the cryptography community~\cite{micciancio2009lattice} (See Conjecture~\ref{conj:svp-poly-factor}).

\begin{theorem}[{\cite[Corollary 3.2]{bruna2020continuous}}]
\label{thm:clwe-hardness}
Let $\beta = \beta(d) \in (0,1)$ and $\gamma = \gamma(d) \geq 2\sqrt{d}$ such that $\gamma/\beta$ is polynomially bounded. Then, there is a polynomial-time quantum reduction from $O(d/\beta)$-approximate $\mathrm{SVP}$ to $\clwe_{\beta,\gamma}$.
\end{theorem}

\begin{conjecture}[{\cite[Conjecture 1.2]{micciancio2009lattice}}]
\label{conj:svp-poly-factor}
There is no polynomial-time quantum algorithm that approximates $\mathrm{SVP}$ to within polynomial factors.
\end{conjecture}

\paragraph{Weak learning and parameter recovery.} 

Recall that every element of the function class $\mathcal{F}_{\gamma}$ is fully characterized by the hidden unit vector $w \in S^{d-1}.$ Hence, one possible strategy towards achieving weak learning of the cosine distribution, could be to recover the vector $w$ from samples of the form \eqref{cosine-dist}. The following lemma (proven in Appendix~\ref{app:population}) shows that given any $w'$ sufficiently close to $w$ one can construct an hypothesis that weakly learns the function $f(x)=\cos(2\pi \gamma \inner{w,x}).$
\begin{proposition}\label{prop:population_main2}
Suppose $\gamma=\omega(1)$.  For any $w' \in S^{d-1}$ with $\min\{\|w-w'\|^2_2,\|w+w'\|^2_2\} \leq 1/(16\pi^2 \gamma^2)$, the functions $h_{A}(x)=\cos(2 \pi\gamma \inner{A,x}), A \in \{w',w\}$ satisfy for large values of $d$ that
\begin{align*}
    \mathbb{E}_{x \sim N(0,I_d)}[\ell((h_w(x),h_{w'}(x))] \leq \mathrm{Var}_{x \sim N(0,I_d)}[(h_w(x))^2]-1/12.
\end{align*}
\end{proposition}

\paragraph{The LLL algorithm and integer relation detection.}
In our algorithmic result, we make use of an appropriate integer relation detection application of the celebrated lattice basis reduction LLL algorithm \cite{lenstra1982factoring}. We say that for some $b \in \mathbb{R}^n$ the vector $m \in \mathbb{Z}^n \setminus \{0\}$ is an \textit{integer relation} for $b$ if $\inner{m,b}=0.$ We make use of the following theorem, and we refer the interested reader to the Appendix \ref{LLL_app} for a complete proof and intuition behind the result. 
\begin{theorem}\label{thm:LLL}
Let $n,N \in \NN.$ Suppose $b \in (2^{-N}\mathbb{Z})^n$ with $b_1=1.$ Let also $m \in \mathbb{Z}^n$ be an integer relation of $b$. Then an appropriate application of the LLL algorithm with input $b$ outputs an integer relation $m' \in \mathbb{Z}^n$ of $b$ with $\|m'\|_2 =O(2^{n/2}\|m\|_2 \|b\|_{2})$ in time polynomial in $n,N$ and $\log (\|m\|_{\infty}\|b\|_{\infty}).$
\end{theorem}

\paragraph{Notation.}
Let $\mathbb{Z}$ denote the set of integers and $\mathbb{R}$ denote the set of real numbers. For $a \in \mathbb{R},$ We use $\mathbb{Z}_{\geq a}$ and $\mathbb{R}_{\geq a}$ for the set of integers at least equal to $a$, and for the set of real numbers at least equal to $a$, respectively. We denote by $\NN=\mathbb{Z}_{\geq 1}$ the set of natural numbers. For $k \in \NN$ we set $[k]:=\{1,2,\ldots,k\}$.  For $d \in \mathbb{N}$, $1\leq p < \infty$ and any $x\in\mathbb{R}^d$, $\|x\|_p$ denotes the $p-$norm $(\sum_{i=1}^d |x_i|^p)^{1/p}$ of $x$, and $\|x\|_{\infty}$ denotes $\max_{1\leq i\leq d}|x_i|$. Given two vectors $x,y \in \mathbb{R}^d$ the Euclidean inner product is  $\inner{x,y}:=\sum_{i=1}^d x_iy_i$. By  $\log: \mathbb{R}^+ \rightarrow \mathbb{R}$ we refer the natural logarithm with base $e$. For $x \in \mathbb{Z}$ and $N \in \NN$ we denote by $(x)_N:=\sgn(x)\lfloor{2^N x\rfloor}/2^N$. Throughout the paper we use the standard asymptotic notation, $o, \omega, O,\Theta,\Omega$ for comparing the growth of two positive sequences $(a_d)_{d \in \NN }$ and $(b_d)_{d \in \NN }$: we say $a_d=\Theta(b_d)$
if there is an absolute constant $c>0$ such that $1/c\le a_d/ b_d \le c$; $a_d =\Omega(b_d)$ or $b_d = O(a_d)$ if there exists  an absolute constant $c>0$ such that $a_d/b_d \ge c$; and $a_d =\omega(b_d)$ or $b_d = o(a_d)$ if $\lim _d a_d/b_d =0$. We say $x=\poly(d)$ if for some $0\leq q<r$ it holds $\Omega(d^q)=x=O(d^r).$ 
\section{Main Results}

In this section we present our main results towards understanding the fundamental hardness of (weakly) learning the single cosine  neuron class given by \eqref{eq: Fcosine}. We present our results in terms of signal to noise ratio (SNR) equal to $1/\beta,$ where recall that $\beta>0$ is an upper bound on the level of adversarial noise $\xi$ one may introduce at the samples given by \eqref{cosine-dist}. All proofs of the statements are deferred to the appendices of each subsection.

\paragraph{A key correspondence.} At the heart of our main results are the following simple elementary equalities that hold for all $v \in \RR$, and may help the intuition of the reader.
\begin{align}
    \cos(2\pi (v~\mod 1)) &= \cos(2\pi v) \label{eqn:mod1} \\
    \arccos(\cos(2\pi v)) &= 2\pi |v \mod 1| \label{eqn:arccos-mod1}\;,
\end{align}
where in Eq~\eqref{eqn:arccos-mod1}, we recall that our $\mathrm{mod} \; 1$ operation takes representatives in $[-1/2,1/2)$. 

An immediate outcome of these equalities, is a key correspondence between the labels of cosine samples and ``phaseless'' CLWE samples, where we reminder the reader that the notion of a CLWE sample is defined in \eqref{CLWE}. By~\eqref{eqn:mod1}, applying the cosine function to CLWE labels results in the cosine distribution with the same frequency, and hidden direction. Conversely, by~\eqref{eqn:arccos-mod1}, applying $\arccos$ to cosine labels results in an arguably harder variant of CLWE, in which the ($\mathrm{mod} \; 1$)-signs of the labels are dropped, with again the same frequency and hidden direction. We say this ``phaseless'' variant of CLWE is harder than CLWE as we can trivially take the absolute value of CLWE labels to obtain these phaseless CLWE samples, and so an algorithm for solving phaseless CLWE automatically implies an algorithm for CLWE.

We have ignored the issue of additive noise for the sake of simplicity in the above discussion. Indeed, the amount of noise in the samples is a key quantity for characterizing the difficulty of these learning problems and the main technical challenge in carrying the reduction between learning single cosine neurons and CLWE. In subsequent sections, we carefully analyze the interplay between the noise level and the computational difficulty of these learning problems.

\subsection{The Information-Theoretically Possible Regime: Small Constant Noise}
\label{sec:IT}
Before discussing the topic of computational hardness, we address the important first question of identifying the noise levels $\beta$ under which \textit{some estimator}, running in potentially exponential time, can weakly learn the class of interest from polynomially many samples.
Note that any constant level of noise above 1, that is $\beta \geq 1,$ would make learning impossible for trivial reasons. Indeed, as the cosine takes values in $[-1,1]$ if $\beta \geq 1$ all the labels $z_i$ can be transformed to the uninformative $0$ value because of the adversarial noise. One can naturally wonder whether any estimator can succeed at the presence of some constant noise level $\beta \in (0,1)$. In this section, we establish that for sufficiently small but constant $\beta>0$ weak learning is indeed possible with polynomially many samples by running an appropriate exponential-time estimator.

 Towards establishing this result, we leverage Proposition \ref{prop:population_main2}, according to which to achieve weak lernability it suffices to construct an estimator that achieves $\ell_2$ recovery of $w$ or $-w$ with an $\ell_2$ error $O(1/\gamma)$. For this reason, we build an exponential-time algorithm that can provably obtain this $\ell_2$ guarantee.  
\begin{theorem}[Information-theoretic upper bound]
\label{thm:it-ub-cosine-main}
For some constants $c_0,C_0>0$ 
(e.g. $c_0=1-\cos(\frac{\pi}{200}), C_0=40000$) 
the following holds. Let $d \in \NN$ and let $\gamma = \gamma(d) > 1$, $\beta(d) \le c_0$, and $\tau = \arccos(1-\beta)/(2\pi)$. Moreover, let $P$ be data distribution given by \eqref{cosine-dist} with frequency $\gamma$, hidden direction $w$, and noise level $\beta$. Then, there exists an $\exp(O(d\log (\gamma/\tau)))$-time algorithm using $O(d\log(\gamma/\tau))$ i.i.d. samples from $P$ that outputs a direction $\hat{w} \in S^{d-1}$ satisfying $\min \{\|\hat{w}-w\|^2_2,\|\hat{w}+w\|^2_2\}  \le C_0 \tau^2 /\gamma^2$ with probability $1-\exp(-\Omega(d))$.
\end{theorem}

The following corollary follows immediately from Theorem \ref{thm:it-ub-cosine-main}, Proposition \ref{prop:population_main2} and the elementary identity that $\arccos(1-\beta) =\Theta( \sqrt{\beta})$ for small $\beta$.

\begin{corollary}
\label{cor:it-ub-cosine-main}
Under the assumptions of Theorem \ref{thm:it-ub-cosine-main} there exists some sufficiently small $c_1>0,$ such that if $\beta \leq c_1$ there exist a $\exp(O(d\log (\gamma/\beta)))$-time algorithm using $O(d\log(\gamma/\beta))$ i.i.d. samples from $P$ that weakly learns the function class $\mathcal{F}_{\gamma}$.
\end{corollary}

The proof of both Theorem \ref{thm:it-ub-cosine-main} and Corollary \ref{cor:it-ub-cosine-main} can be found in Appendix \ref{app:it-ub}.
\begin{algorithm}[t]
\caption{Information-theoretic recovery algorithm for learning cosine neurons}
\label{alg:it-recovery}
\KwIn{Real numbers $\gamma=\gamma(d) > 1$, $\beta=\beta(d) $, and a sampling oracle for the cosine distribution \eqref{cosine-dist} with frequency $\gamma$, $\beta$-bounded noise, and hidden direction $w$.}
\KwOut{Unit vector $\hat{w} \in S^{d-1}$ s.t. $\min \{\|\hat{w}-w\|_2,\|\hat{w}+w\|_2\}=O(\arccos(1-\beta)/\gamma)$.}
\vspace{1mm} \hrule \vspace{1mm}
Let $\tau = \arccos(1-\beta)/(2\pi)$, $\eps = 2\tau/\gamma$, $m = 64d\log(1/\eps)$, and let $\sC$ be an $\eps$-cover of the unit sphere $S^{d-1}$. Draw $m$ samples $\{(x_i,y_i)\}_{i=1}^m$ from the cosine distribution \eqref{cosine-dist}. \\
\For{$i=1$ \KwTo $m$}{
    $z_i = \arccos(y_i)/(2\pi)$
    }
\For{$v \in \sC$}{
    Compute $T_v = \frac{1}{m}\sum_{i=1}^{m} \one \left[|\gamma \langle v, x_i \rangle - z_i \mod 1| \le 3\tau \right] +  \one \left[|\gamma \langle v, x_i \rangle + z_i \mod 1| \le 3\tau \right]$
        }
\Return $\hat{w} = \arg\max_{v \in \sC} T_v$.
\end{algorithm}

The exponential-time algorithm achieving the guarantees of Theorem \ref{thm:it-ub-cosine-main}  is described in Algorithm \ref{alg:it-recovery} and proceeds as following. First, it needs to construct an $\epsilon$-cover $\sC$ of the sphere $S^{d-1}$ where $\eps = \tau/\gamma$.  Note that this step already requires at least exponential-time as any such cover needs to be of exponential size. 
Furthermore, note that such a construction is indeed possible in exponential time by just drawing $O(N \log N)$ uniform random points on the sphere where $N \approx \epsilon^{-d}$ is the $\epsilon$-covering number of the sphere. 
Following that it assigns a score to each point in the cover, call it $v$, which simply counts the number of samples that could have been possibly produced under the assumption that the vector $v$ was the true hidden vector. The algorithm then outputs the element of maximum score.  The analysis then proceeds by a careful probabilistic reasoning to claim that the maximizer needs to land $O(1/\gamma)$-close to the true hidden vector (or its antipode), something true for all $\beta$ less than a sufficiently small constant. Finally, notice that to properly choose the appropriate quantification of the score assignment the algorithm uses the ``key correspondence'' \eqref{eqn:arccos-mod1} to transform the samples into ``phaseless CLWE samples'', which allowed for a cleaner presentation of the algorithm and an easier analysis. We refer the reader to Appendix~\ref{app:it-ub} for full details.

\subsection{The Cryptographically Hard Regime: Polynomially Small Noise}
\label{crypto_main}

Given the results in the previous subsection, we discuss now whether a polynomial-time algorithm can achieve weak learnability of the class $\mathcal{F}_{\gamma}$ for some noise level $\beta$ smaller than an inverse polynomial quantity in $d$, which we call an inverse-polynomial edge, per Definition \ref{def:weak-learning}. We answer this in the negative by showing a reduction from CLWE to the problem of weakly learning $\mathcal{F}_\gamma$ to any inverse-polynomial edge. This implies that a polynomial-time algorithm for weakly learning $\sF_\gamma$ would yield polynomial-time quantum attacks against worst-case lattice problems, which are widely believed to be hard against quantum computers. As mentioned in the introduction, our reduction applies with any $1$-periodic and $O(1)$-Lipschitz activation $\phi$. We provide a proof sketch below, and defer the full details to Appendix \ref{app:cosine-clwe-hardness}.

\begin{theorem}
\label{thm:cosine-clwe-hardness}
    Let $d \in \NN$, $\gamma = \omega(\sqrt{\log d}), \beta = \beta(d) \in (0,1)$. Moreover, let $L > 0$, let $\phi : \RR \rightarrow [-1,1]$ be an $L$-Lipschitz 1-periodic univariate function, and $\tau = \tau(d)$ be such that $\beta/(L\tau) = \omega(\sqrt{\log d})$. Then, a polynomial-time (improper) algorithm that weakly learns the function class $\sF_{\gamma}^\phi = \{f_{\gamma,w}(x) = \phi( \gamma \langle w, x \rangle) \mid w \in \sS^{d-1} \}$ over Gaussian inputs $x \iiddistr N(0,I_d)$ under $\beta$-bounded adversarial noise implies a polynomial-time algorithm for $\clwe_{\tau,\gamma}$.
\end{theorem}

By the hardness of CLWE (Theorem~\ref{thm:clwe-hardness}) and our Theorem~\ref{thm:cosine-clwe-hardness}, we can immediately deduce the cryptographic hardness of learning the single cosine neuron under inverse polynomial noise.

\begin{corollary}
\label{cor:cosine-clwe-hardness}
Let $d \in \NN$, $\gamma = \gamma(d) \ge 2\sqrt{d}$ and $\tau = \tau(d) \in (0,1)$ be such that $\gamma/\tau = \poly(d)$, and $\beta = \beta(d)$ be such that $\beta/\tau = \omega(\sqrt{\log d})$. Then, a polynomial-time algorithm that weakly learns the cosine neuron class $\sF_\gamma$ under $\beta$-bounded adversarial noise implies a polynomial-time quantum algorithm for $O(d/\tau)$-approximate $\mathrm{SVP}$.
\end{corollary}

\begin{proof}[Proof sketch of Theorem \ref{thm:cosine-clwe-hardness}]
Recall that the goal is to reduce $\clwe_{\tau,\gamma}$ to the problem of weakly learning the function class $\sF_{\gamma}^\phi$. Now, $\clwe_{\tau,\gamma}$ is the problem of distinguishing the distribution $A_{w,\tau,\gamma}$ which outputs samples of the form $(x,z)$ where $z = \gamma \inner{w,x} + \xi, x \sim N(0,I_d), \xi \sim N(0,\tau)$, for some hidden direction $w \in S^{d-1}$, from the null distribution $A_0$ which outputs $(x,z)$ where $ x \sim N(0,I_d)$ but $z \sim U[-1/2,1/2]$ independent from $x.$ Notice that (similar to Eq~\eqref{eqn:mod1}) the 1-periodicity and the Lipschitzness of $\phi$ implies that for any $\gamma \ge 0$, $w \in S^{d-1}$, $x \in \RR^d$, and $\xi \in \RR$,
\begin{align}
\label{eqn:sketch-lip-condition}
    \phi(z_i)=\phi(\gamma \inner{w,x} + \xi \mod 1) = \phi(\gamma \inner{w,x} + \xi)= \phi(\gamma \langle w,  x \rangle ) + \tilde{\xi'} \;,
\end{align}
for some $\tilde{\xi} \in [-L |\xi|, L |\xi|]$.  Using Eq.~\eqref{eqn:sketch-lip-condition} one can then directly use $m$ CLWE samples with Gaussian random noise, say, $(x_i,z_i),$ and transform them into $m$ samples from $D_w^\phi$ with bounded adversarial noise by $L\tau \le \beta$, by simply considering the pairs $(x_i,\phi(z_i)), i=1,2,\ldots,m$.

Let us suppose now we have a learning algorithm that weakly learns the function class $\sF_\gamma^\phi$ with $\beta$-bounded adversarial noise. Then we can draw $m$ samples from $A_{w,\tau,\gamma},$ transform them as above into samples from $D_w^\phi$, run the (robust) learning algorithm on $D_w^\phi$, and finally obtain an hypothesis $h=h(x_i,\phi(z_i))$ that weakly learns the function class $\sF_\gamma^\phi$. On the other hand, samples from $A_0$ have labels $z_i$ independent with $x_i$ and therefore are completely uninformative for the learning problem of interest. In particular, one can never hope to achieve weak learning of the function class $\sF_\gamma^\phi$, using the hypothesis function $h=h(x_i,\phi(z_i))$ on the samples $(x_i, z_i)$ now generated from $A_0$. This difference is quantified by the loss of the hypothesis $\mathcal{L}_D(h)$ which in case $D=A_{w,\tau,\gamma},$ is smaller by an inverse polynomial additive factor from the trivial error, while in the case in case $D=A_0$ it is lower bounded by the trivial error. This property is what allows indeed to detect between $A_{w,\tau,\gamma}$  and $A_0$, and therefore solve $\clwe_{\beta,\gamma}$ and complete the reduction.
\end{proof}

\begin{remark}[Robust learning under position-dependent random noise is hard]
\label{rem:random-noise}
    Robustness against advesarial noise in Theorem~\ref{thm:cosine-clwe-hardness} is not necessary for computational hardness. In fact, the reduction only requires robustness against a certain position-dependent random noise. More precisely, for a fixed hidden direction $w \in S^{d-1}$, the random noise $\tilde{\xi}$ is given by $\tilde{\xi} = \phi(\gamma\inner{w,x}+\xi) -\phi(\gamma\inner{w,x})$, where $x \sim N(0,I_n)$ and $\xi \sim N(0,\beta)$.
\end{remark}

\subsection{The Polynomial-Time Possible Regime: Exponentially Small Noise}
\label{LLL_main}

In this section, in sharp contrast with the previous section, we design and analyze a novel polynomial-time algorithm which provably weakly learns the single cosine neuron with only $d+1$ samples, when the noise is exponentially small. The algorithm is based on the celebrated lattice basis reduction LLL algorithm and its specific application obtaining the integer relation detection guarantee described in Theorem \ref{thm:LLL}. Let us also recall from notation that for a real number $x$ and $N \in \mathbb{Z}_{\geq 1},$ we denote by $(x)_N:=\sgn(x)\lfloor{2^N x\rfloor}/2^N.$ We establish the following result, proved in Appendix \ref{LLL_app}.

\begin{algorithm}[t]
\caption{LLL-based algorithm for learning the single cosine neuron.}
\label{alg:lll}
\KwIn{i.i.d.~noisy $\gamma$-single cosine neuron samples $\{(x_i,z_i)\}_{i=1}^{d+1}$.}
\KwOut{Unit vector $\hat{w} \in S^{d-1}$ such that $\min(\|\hat{w} - w\|,\|\hat{w}+w\|) = \exp(-\Omega((d \log d)^3))$.}
\vspace{1mm} \hrule \vspace{1mm}

\For{$i=1$ \KwTo $d+1$}{
    $z_i \leftarrow \sgn(z_i)\cdot \min(|z_i|, 1)$ \\
    $\tilde{z}_i = \arccos(z_i)/(2\pi) \mod 1$
    }

Construct a $d \times d$ matrix $X$ with columns $x_2,\ldots,x_{d+1}$, and let $N=d^3(\log d)^2$.\\
\If{$\det(X) = 0$}{\Return $\hat{w} = 0$ and output FAIL}
Compute $\lambda_1 = 1$ and $\lambda_i = \lambda_i(x_1,\ldots,x_{d+1})$ given by $(\lambda_2,\ldots,\lambda_{d+1})^\top = X^{-1} x_1$. \\
Set $M=2^{3d}$ and $\tilde{v}=\left((\lambda_2)_N,\ldots,(\lambda_{d+1})_N, (\lambda_{1}z_1)_N,\ldots, (\lambda_{d+1}z_{d+1})_N, 2^{-N}\right) \in \mathbb{R}^{2d+2}$\\
Output $(t_1,t_2,t)\in  \ZZ^{d+1} \times \ZZ^{d+1} \times \ZZ$ from running the LLL basis reduction algorithm on the lattice generated by the columns of the following $(2d+3)\times (2d+3)$ integer-valued matrix,\begin{equation*}
\left(\begin{array}{@{}c|c@{}}
  \begin{matrix}
  M 2^N(\lambda_1)_N 
  \end{matrix}
  & M 2^N \tilde{v} \\
\hline
  0_{(2d+2)\times 1} &

  I_{(2d+2)\times (2d+2)}
\end{array}\right)
  \end{equation*}\\ 
Compute $g = \mathrm{gcd}(t_2)$, by running Euclid's algorithm. \\
\If{$g=0 \vee (t_2/g) \notin \{-1,1\}^{d+1}$}{\Return $\hat{w} = 0$ and output FAIL}
$\hat{w} \leftarrow \mathrm{SolveLinearEquation}(w', X^\top w' = (t_2/g)z + (t_1/g))$ \\
\Return $\hat{w}/\| \hat{w}\|$ and output SUCCESS.
\end{algorithm}

\begin{theorem}\label{thm:algo_main}

Suppose that  $1 \leq \gamma \leq d^Q$ for some fixed $Q>0,$ and $\beta  \leq \exp(-  (d \log d)^3).$ Then Algorithm \ref{alg:lll} with input $(x_i,z_i)_{i=1,\ldots,d+1}$ i.i.d. samples from \eqref{cosine-dist} with frequency $\gamma$, hidden direction $w$ and noise level $\beta,$ outputs $w' \in S^{d-1}$ with
\begin{align*}
\min\{ \|w'-w\|_2, \|w'+w\|_2\} = O\left(\frac{\beta}{\gamma} \right)=\frac{1}{\gamma}\exp(-\Omega( (d \log d)^3))~,
\end{align*}and terminates in $\poly(d)$ steps, with probability $1-\exp(-\Omega(d))$.
Moreover, if the algorithm skips the last normalization step, the output $w' \in \mathbb{R}^d$ satisfies $\min\{ \|w'-\gamma w\|_2, \|w'+\gamma w\|_2\} = O\left(\beta \right)=\exp(-\Omega( (d \log d)^3))$. 
\end{theorem}

In particular, by combining our result with Proposition \ref{prop:population_main2}, one concludes the following result.

\begin{corollary}
Suppose that  $\omega(1)= \gamma =\mathrm{poly}(d)$ and $\beta  \leq \exp(-  (d \log d)^3).$ Then there exists a polynomial-in-$d$ time algorithm using $d+1$ samples from a single cosine neuron distribution \eqref{cosine-dist}, with frequency $\gamma$ and noise level $\beta,$ that weakly learns the function class $\mathcal{F}_{\gamma}.$
\end{corollary}
\begin{proof}[Proof sketch of Theorem \ref{thm:algo_main}]
For the purposes of the sketch let us focus on the noiseless case, explaining at the end how an exponentially small tolerance is possible. In this setting, we receive $m$ samples of the form $z_i=\cos(2\pi \inner{w,x_i}), i \in [m]$. The algorithm then uses the arcosine and obtains the ``phaseless'' CLWE values $\tilde{z}_i$ which according to \eqref{eqn:arccos-mod1} satisfy for some \textit{unknown} $\epsilon_i \in \{-1,1\},K_i \in \mathbb{Z}$ $\inner{w,x_i}=\epsilon_i \tilde{z_i} +K_i.$ Notice that if we knew the integer values of $\epsilon_i, K_i,$ since we know $\tilde{z}_i,$ the problems becomes simply solving a linear system for $w$. The algorithm then leverages the application of the powerful LLL algorithm to perform integer relation detection and identify the values of $\epsilon_i, K_i$, as stated in Theorem \ref{thm:LLL}. The way it does it is as follows. It first finds coefficients $\lambda_i, i=1,2,\ldots,d+1$ such that $\sum_{i=1}^{d+1} \lambda_i x_i=0$ which can be easily computed because we have $d+1$ vectors in $\mathbb{R}^d.$ Then using the definition of $\tilde{z}_i,$ the relation between the coefficient implies the identity 

\begin{align}
\label{IR}
\sum_{i=1}^{d+1} \epsilon_i \lambda_i  \tilde{z}_i+\sum_{i=1}^{d+1} K_i \lambda_i=\sum_{i=1}^{d+1} \lambda_i \inner{x_i,w}=0.
\end{align}
In particular, the $\epsilon_i,K_i$ are coefficients in an integer relation connecting the \textit{known} numbers $\lambda_i z_i, \lambda_i, i=1,2,\ldots,d+1.$ Now, an issue is that as one cannot enter the real numbers as input for the lattice-based LLL, the algorithm truncates the numbers to the first $N$ bits and then hope that post-truncation \textit{all the near-minimal integer relations} between these truncated numbers remain a (small multiple of) $\epsilon_i,K_i$, a sufficient condition so that LLL can identify them based on Theorem \ref{thm:LLL}. We establish that indeed this the case and this is the most challenging part of the argument. The argument is based on some careful application of the anticoncentration properties of low-degree polynomials (notice that the $\lambda_i$ are rational functions of $x_i$ by Cramer's rule), to deduce that the numbers $\lambda_i, \lambda_i z_i$ are in ``sufficient general position'', in terms of rational independence, for the argument to work. We remark that this is a potentially important technical advancement over the prior applications of the LLL algorithm towards performing such inference tasks, such as for average-case subset sum problems \cite{Lagarias85, FriezeSubset} or regression with discrete coefficients \cite{NEURIPS2018_ccc0aa1b, gamarnik2019inference} where the corresponding $\lambda_i, \lambda_i z_i$ coefficients are (truncated) i.i.d. continuous random variables in which case anticoncentration is immediate (see e.g. \cite[Theorem 2.1]{NEURIPS2018_ccc0aa1b}). The final step is to prove that the algorithm is able to tolerate some noise level. We establish that indeed if $N=\tilde{\Theta}(d^3)$ then indeed the argument can still work and tolerate $\exp(-\tilde{\Theta}(d^3))$-noise by showing that the near-minimal integer relations remain unchanged under this level of exponentially small noise.
\end{proof}
\begin{remark}
\label{rem:hiddenvect}
    While the main recovery guarantee in Theorem \ref{thm:algo_main} is stated in terms of the hidden direction $w \in S^{d-1}$, Algorithm \ref{alg:lll} in fact also recovers the vector $\gamma w$ (up to global sign), if one skips the last line of the algorithm, which normalises the output to the unit sphere. Such recovery is shown as a crucial step towards establishing the main result. This stronger recovery will be used in the next section. 
\end{remark}

\begin{remark}[CLWE with exponentially small noise]
    Notice that the detection problem in CLWE (\ref{CLWE}) reduces to the cosine learning problem (\ref{cosine-dist}). Indeed, if $\check{z} = \gamma \langle x, w \rangle + \check{\xi}~\text{mod}~ 1 ~\in [-1/2,1/2)$ is a CLWE sample, then $z = \cos(\check{z})$ satisfies 
    \begin{align*}
        z = \cos ( 2\pi \gamma \langle x, w\rangle) + \xi~, 
    \end{align*}
    with $|\xi| \leq 2\pi \gamma |\check{\xi}|$. Algorithm \ref{alg:lll} and the associated analysis Theorem \ref{thm:algo_main} thus improve upon the exact CLWE recovery of \cite[Section 6]{bruna2020continuous} in two aspects: (i) it requires $d+1$ samples as opposed to $d^2$; and (ii) it tolerates exponentially small noise. 
\end{remark}

\begin{remark}[CLWE with subexponentially small noise]
    The intermediate regime of subexponentially small noise, which corresponds to the uncharted region between ``Crypto-Hard'' and ``Polynomial-Time Possible'' in Figure~\ref{fig:hardness-landscape} where $\beta = \exp(-\Theta(d^c))$ for some $c \in (0,1)$, has not been explored in our work. However, we conjecture that this regime is still hard for polynomial-time algorithms. While~\cite{bruna2020continuous} did not consider this noise regime for the CLWE problem, given the problem's analogy to the LWE problem~\cite{regev2005lwe}, it is plausible that the quantum reduction from CLWE to approximate SVP also applies for subexponentially small noise, since the quantum reduction for LWE extends to subexponentially small noise. That is, it is possible that the requirement $\gamma/\beta = \poly(d)$ in Theorem~\ref{thm:clwe-hardness} can be relaxed, given the high degree of similarity between CLWE and LWE. If this is true, then a polynomial-time algorithm for CLWE with $\gamma \ge 2\sqrt{d}$ and $\beta \in (0,1)$ implies a polynomial-time quantum algorithm for $O(d/\beta)$-approximate SVP. Hence, by Theorem~\ref{thm:cosine-clwe-hardness}, a polynomial-time algorithm for our setting with subexponentially small noise would yield a ``breakthrough'' quantum algorithm for approximate SVP, since no polynomial-time algorithms are known to achieve subexponential approximation factors of the form $2^{O(d^c)}$ for any constant $c < 1$. In more detail, the best known algorithms for approximate SVP are lattice block reductions, such as the Block Korkin-Zolotarev (BKZ) algorithm and its variants~\cite{schnorr1987hierarchy,schnorr1994lattice,micciancio2016practical}, or slide reductions~\cite{gama2008finding,aggarwal2016slide}. These block reduction algorithms, which can be seen as generalizations of the LLL algorithm, trade-off running time for better SVP approximation factors. However, none is known to achieve SVP approximation factor $2^{O(d^c)}$ for any constant $c < 1$ in polynomial time.
\end{remark}

\subsection{Exact Recovery for Phase Retrieval with Optimal Sample Complexity} 
\label{sec:phase-retrieval}
Phase retrieval is a classic inverse problem \cite{fienup1982phase} with important applications in computational physics and signal processing, and which has been thoroughly studied in the high-dimensional statistics and non-convex optimization literature \cite{BALAN2006345,jaganathan2015phase,goldstein2018phasemax,mondelli2018fundamental,Barbier5451,chen2019gradient,mannelli2020complex,mannelli2020optimization,mignacco2021stochasticity}. In the noiseless setting, the phase retrieval problem asks one to exactly recover a hidden signal $w \in \mathbb{R}^{d}$, up to global symmetry $\pm w$, given sign-less measurements of the form
\begin{align*}
    y = |\inner{x,w}|\;.
\end{align*}

As mentioned in Section~\ref{sec:hardness-landscape}, our cosine learning problem can be seen as ``containing'' the phase retrieval problem since the even-ness of the cosine function immediately ``erases'' the sign of the inner product $\inner{x,w}$. More precisely, the phase retrieval problem can be \emph{reduced} to the cosine learning problem by simply applying the cosine function to the measurements and noticing that
\begin{align*}
    \cos(2\pi|\inner{x,w}|) = \cos(2\pi\inner{x,w})\;.
\end{align*}

Hence, Algorithm \ref{alg:lll}, without the last normalization step (see Remark \ref{rem:hiddenvect}), can be immediately used to exactly solve phase retrieval under exponentially small noise. Formally, Theorem~\ref{thm:algo_main} (for $\gamma=\|w\|_2$) certifies near exact recovery for (Gaussian) phase retrieval using only $d+1$ samples:
\begin{corollary}[Recovery of Phase Retrieval under exponentially small noise]
\label{cor:phase-recovery}
    Let us consider noise level $\beta \leq (2\pi)^{-1}\exp(-(d \log d)^3)$, and arbitrary $w \in \mathbb{R}^d$ such that $1 \leq \|w\|_2 =\poly(d)$. Suppose $\{(x_i, y_i)\}_{i=1,\ldots d+1}$ are i.i.d. samples 
    of the form $x_i \sim N(0, I_d)$ and $y_i = | \langle x_i, w \rangle| + \check{\xi}_i$, with arbitrary $|\check{\xi}_i| \leq \beta$.  
    Then Algorithm \ref{alg:lll} with input $\{(x_i, z_i=\cos(2\pi y_i))\}_{i=1,\ldots d+1}$ returns an un-normalized output $w'$ satisfying $\min\{ \|w'-w\|_2, \|w'+w\|_2\}= O(\beta)$ and terminates in $\mathsf{poly}(d)$ steps, with probability $1 - \exp(-\Omega(d))$.
\end{corollary}

Remarkably, our lattice-based algorithm improves upon the AMP-based algorithm analysed in \cite{Barbier5451}, which requires $m \approx 1.128 d$ in the high-dimensional regime for exact recovery, and therefore shows that AMP is not optimal amongst polynomial-time algorithms in the regime of exponentially small adversarial noise. Hence, this adds phase retrieval to a list of problems, including for example linear regression with discrete coefficients, where in the exponentially-small noise regime no computational-statistical gap is present \cite{NEURIPS2018_ccc0aa1b} \cite[Section 4.2.1]{kunisky2019notes}. 
We note that the possibility that LLL might be efficient for exponentially-small noise phase retrieval was already suggested in \cite{NEURIPS2018_ccc0aa1b} and later established for discrete-valued $w$ in \cite{gamarnik2019inference}. In fact, previous results by~\cite{andoni2017correspondence} have already shown that exact (i.e., noiseless) phase retrieval is possible with optimal sample complexity using an LLL-based algorithm very similar to ours.
We also remark that our result is stated under the Gaussian distribution, as opposed to generic \emph{i.i.d.} entries as in \cite{Barbier5451}. The reason is that we rely crucially on anti-concentration properties of random low-degree polynomials, which are satisfied in the Gaussian case \cite{Carbery2001DistributionalAL, v012a011}. However, these anti-concentration properties can be extended to log-concave random variables~\cite[Theorem 8]{Carbery2001DistributionalAL}, and as a result our analysis easily extends to $x_i$ following a product distribution of a density which is both log-concave and sub-Gaussian. In this respect, we strengthen previous results by~\cite{andoni2017correspondence}, whose analysis is tailored to the Gaussian case.

An interesting question is whether the sample size $d+1$ is \emph{information-theoretically optimal} to recover $w$ up to error $\beta$ from the studied phase retrieval setting. In other words, whether the recovery is possible with $d$ samples by any estimator, and irrespective of any computational constraints. For simplicity, we focus on the noiseless case $\beta=0$, in which case the goal is \textit{exact recovery}. 
We note that the answer depends on the prior knowledge on $w$, or, assuming throughout a rotationally invariant prior for $w$, on the prior distribution of $\|w\|$.
Indeed, in the extreme setting where the hidden vector $w\in \mathbb{R}^d$ is unconstrained, we immediately observe that there are $2^d$ possible vectors $w'$ satisfying $| \inner{x_i, w'} | = | \inner{x_i, w}|$. 
As a consequence, by taking into consideration the global sign flip symmetry, exact recovery is possible only with probability at most $2^{-d+1}$. On the other extreme, if one knew that $\|w \|=1$, then generically only two ($w$ and $-w$) of these $2^d$ possibilities will satisfy the exact norm constraint, making exact recovery (up to global sign flip) possible with only $d$ samples in that case. The following theorem addresses the general case between these two extremes, and establishes that exact recovery using only $d$ samples cannot be generally certified with high probability, in stark contrast with Corollary \ref{cor:phase-recovery}.

\begin{theorem}
\label{thm:IToptim}
Assume a uniform prior on the direction $w/\|w\|_2 \in S^{d-1}$, and assume that $\gamma=\|w\|_2>0$ is distributed independently of $w$ according to a probability density $q_\gamma$ which satisfies the following assumption:
For some $B>\sqrt{2}$ and $C>0$, the function $q_{\gamma}: \mathbb{R} \rightarrow [0,+\infty)$ satisfies 
\begin{equation}
\label{assump:basic}
 q_\gamma(t) t^{-d+1} ~\text{ is non-increasing in } t \in [1, B]~\text{, and } \int_{\sqrt{2}}^B q_\gamma(t) dt \geq C. 
\end{equation}
Consider $d\geq 2$ i.i.d. samples $\{x_i, y_i=|\inner{x_i, w}| \}_{i=1\ldots d}$, where $x_i$ are i.i.d. $N(0,I_d)$ and $w$ is drawn from two independent variables: $w/\|w\|$ uniformly distributed in $S^{d-1}$ and $\|w\|$ is distributed with density $q_\gamma$ satisfying (\ref{assump:basic}).
Let $\mathcal{A}$ be any estimation procedure (deterministic or randomized) that takes as input $\{(x_i, y_i)\}_{i=1,\ldots,d}$ and outputs $w' \in \mathbb{R}^d$. Then with probability $\omega(d^{-2})$ it holds $w' \not \in \{-w,w\}.$
\end{theorem}
This theorem is proved in Appendix \ref{optim_sample_app}. The main idea of the proof is to show that, with non-neglibile probability ($\omega(d^{-2})$), some of the `spurious' solutions $w'$ satisfying $|\inner{x_i, w'} | = | \inner{x_i, w}|$ are such that $\|w'\| \leq \|w\|$. Combined with our assumption on the prior $q_\gamma$ and the optimality of MAP estimators in terms of error probability, the result follows.
We also note that Assumption (\ref{assump:basic}) is very mild, and is satisfied e.g. when $\gamma$ is uniformly distributed in $[1,B]$, or when $w$ is either uniformly distributed in a circular ring, or follows a Gaussian distribution. 
Therefore, our proposed algorithm, as well as the algorithm used in \cite{andoni2017correspondence}, obtains a \emph{sharp} optimal sample complexity in this phase-retrieval setup, in the sense that even one less sample than the sample complexity of our algorithm is not sufficient for exact recovery with high probability.

Finally, we would like to highlight that our result and the described lower bound should be also understood in contrast with the recently established \emph{weak recovery} threshold that $d/2(1+o(1))$ measurements actually suffice for achieving some non-trivial (constant) error with $w$ \cite{mondelli2018fundamental}.

\section*{Acknowledgements}
We thank Oded Regev, Ohad Shamir, Lenka Zdeborov\'{a}, Antoine Maillard, and Afonso Bandeira for providing helpful comments. We also thank Daniel Hsu for pointing out the relevant prior work \cite{andoni2017correspondence} after an initial version of our manuscript was posted online. MS and JB are partially supported by the Alfred P. Sloan Foundation, NSF RI-1816753, NSF CAREER CIF-1845360, and NSF CCF-1814524. IZ is supported by the CDS Moore-Sloan Postdoctoral Fellowship.

\printbibliography

@ARTICLE{Wain2009,
  author={Wainwright, Martin J.},
  journal={IEEE Transactions on Information Theory}, 
  title={Sharp Thresholds for High-Dimensional and Noisy Sparsity Recovery Using $\ell _{1}$ -Constrained Quadratic Programming (Lasso)}, 
  year={2009},
  volume={55},
  number={5},
  pages={2183-2202},
  doi={10.1109/TIT.2009.2016018}}

@INPROCEEDINGS{Reeves19,
  author={Reeves, Galen and Xu, Jiaming and Zadik, Ilias},
  booktitle={2019 IEEE 8th International Workshop on Computational Advances in Multi-Sensor Adaptive Processing (CAMSAP)}, 
  title={All-or-Nothing Phenomena: From Single-Letter to High Dimensions}, 
  year={2019},
  volume={},
  number={},
  pages={654-658},
  doi={10.1109/CAMSAP45676.2019.9022473}}

@article{AlonClique,
author = {Alon, Noga and Krivelevich, Michael and Sudakov, Benny},
title = {Finding a large hidden clique in a random graph},
journal = {Random Structures \& Algorithms},
volume = {13},
number = {3‐4},
pages = {457-466},
doi = {10.1002/(SICI)1098-2418(199810/12)13:3/4<457::AID-RSA14>3.0.CO;2-W},
url = {https://onlinelibrary.wiley.com/doi/abs/10.1002/%28SICI%291098-2418%28199810/12%2913%3A3/4%3C457%3A%3AAID-RSA14%3E3.0.CO%3B2-W},
eprint = {https://onlinelibrary.wiley.com/doi/pdf/10.1002/%28SICI%291098-2418%28199810/12%2913%3A3/4%3C457%3A%3AAID-RSA14%3E3.0.CO%3B2-W},
year = {1998}
}

@article{jerrum1992large,
  title={Large cliques elude the Metropolis process},
  author={Jerrum, Mark},
  journal={Random Structures \& Algorithms},
  volume={3},
  number={4},
  pages={347--359},
  year={1992},
  publisher={Wiley Online Library}
}

@misc{bandeira2018notes,
      title={Notes on computational-to-statistical gaps: predictions using statistical physics}, 
      author={Afonso S. Bandeira and Amelia Perry and Alexander S. Wein},
      year={2018},
      eprint={1803.11132},
      archivePrefix={arXiv},
      primaryClass={stat.ML}
}

@misc{kunisky2019notes,
      title={Notes on Computational Hardness of Hypothesis Testing: Predictions using the Low-Degree Likelihood Ratio}, 
      author={Dmitriy Kunisky and Alexander S. Wein and Afonso S. Bandeira},
      year={2019},
      eprint={1907.11636},
      archivePrefix={arXiv},
      primaryClass={math.ST}
}

@INPROCEEDINGS{BarakClique,
  author={B. {Barak} and S. B. {Hopkins} and J. {Kelner} and P. {Kothari} and A. {Moitra} and A. {Potechin}},
  booktitle={2016 IEEE 57th Annual Symposium on Foundations of Computer Science (FOCS)}, 
  title={A Nearly Tight Sum-of-Squares Lower Bound for the Planted Clique Problem}, 
  year={2016},
  volume={},
  number={},
  pages={428-437},
  doi={10.1109/FOCS.2016.53}}

@misc{gamarnik2019sparse,
      title={Sparse High-Dimensional Linear Regression. Algorithmic Barriers and a Local Search Algorithm}, 
      author={David Gamarnik and Ilias Zadik},
      year={2019},
      eprint={1711.04952},
      archivePrefix={arXiv},
      primaryClass={math.ST}
}

@inproceedings{gamarnik2017high,
  title={High dimensional linear regression with binary coefficients: Mean squared error and a phase transition},
  author={Gamarnik, David and Zadik, Ilias},
  booktitle={Conference on Learning Theory (COLT)},
  year={2017}
}

@misc{gamarnik2019landscape,
  title={The landscape of the planted clique problem: Dense subgraphs and the overlap gap property},
  author={Gamarnik, David and Zadik, Ilias},
  eprint={1904.07174},
  archivePrefix={arxiv},
  year={2019}
}

@article{GLM1,
 ISSN = {00359238},
 URL = {http://www.jstor.org/stable/2344614},
 author = {J. A. Nelder and R. W. M. Wedderburn},
 journal = {Journal of the Royal Statistical Society. Series A (General)},
 number = {3},
 pages = {370--384},
 publisher = {[Royal Statistical Society, Wiley]},
 title = {Generalized Linear Models},
 volume = {135},
 year = {1972}
}

@InProceedings{andoni2017correspondence,
  title = 	 {Correspondence retrieval},
  author = 	 {Andoni, Alexandr and Hsu, Daniel and Shi, Kevin and Sun, Xiaorui},
  booktitle = 	 {COLT},
  pages = 	 {105--126},
  year = 	 {2017},
  volume = 	 {65},
  series = 	 {Proceedings of Machine Learning Research},
  publisher =    {PMLR}
}

@inproceedings{kakade2011efficient,
 author = {Kakade, Sham M and Kanade, Varun and Shamir, Ohad and Kalai, Adam},
 booktitle = {Advances in Neural Information Processing Systems},
 editor = {J. Shawe-Taylor and R. Zemel and P. Bartlett and F. Pereira and K. Q. Weinberger},
 pages = {},
 publisher = {Curran Associates, Inc.},
 title = {Efficient Learning of Generalized Linear and Single Index Models with Isotonic Regression},
 volume = {24},
 year = {2011}
}

@article{maly2003co,
  title={The coarea formula for Sobolev mappings},
  author={Mal{\'y}, Jan and Swanson, David and Ziemer, William},
  journal={Transactions of the American Mathematical Society},
  volume={355},
  number={2},
  pages={477-492},
  year={2003}
}

@book{ledoux2001concentration,
  title={The concentration of measure phenomenon},
  author={Ledoux, Michel},
  number={89},
  year={2001},
  publisher={American Mathematical Soc.}
}

@inproceedings{kannan1983improved,
author = {Kannan, Ravi},
title = {Improved Algorithms for Integer Programming and Related Lattice Problems},
year = {1983},
isbn = {0897910990},
publisher = {Association for Computing Machinery},
address = {New York, NY, USA},
url = {https://doi.org/10.1145/800061.808749},
doi = {10.1145/800061.808749},
booktitle = {Proceedings of the Fifteenth Annual ACM Symposium on Theory of Computing},
pages = {193–206},
numpages = {14},
series = {STOC '83}
}

@inproceedings{shamir1982polynomial,
  title={A polynomial time algorithm for breaking the basic Merkle-Hellman cryptosystem},
  author={Shamir, Adi},
  booktitle={23rd Annual Symposium on Foundations of Computer Science (sfcs 1982)},
  pages={145--152},
  year={1982},
  organization={IEEE}
}

@inproceedings{lagarias1984knapsack,
  title={Knapsack public key cryptosystems and diophantine approximation},
  author={Lagarias, Jeffrey C},
  booktitle={Advances in cryptology},
  pages={3--23},
  year={1984},
  organization={Springer}
}

@article{Lagarias85,
author = {Lagarias, J. C. and Odlyzko, A. M.},
title = {Solving Low-Density Subset Sum Problems},
year = {1985},
issue_date = {Jan. 1985},
publisher = {Association for Computing Machinery},
address = {New York, NY, USA},
volume = {32},
number = {1},
issn = {0004-5411},
url = {https://doi.org/10.1145/2455.2461},
doi = {10.1145/2455.2461},
journal = {J. ACM},
month = jan,
pages = {229–246},
numpages = {18}
}

@article{Carbery2001DistributionalAL,
  title={Distributional and L-q norm inequalities for polynomials over convex bodies in R-n},
  author={A. Carbery and James Wright},
  journal={Mathematical Research Letters},
  year={2001},
  volume={8},
  pages={233-248}
}

@inbook{GLMbook,
author="M{\"u}ller, Marlene",
title="Generalized Linear Models",
bookTitle="XploRe --- Learning Guide",
year="2000",
publisher="Springer Berlin Heidelberg",
address="Berlin, Heidelberg",
pages="205--228",
isbn="978-3-642-60232-0",
doi="10.1007/978-3-642-60232-0_7",
url="https://doi.org/10.1007/978-3-642-60232-0_7"
}

@inproceedings{gama2008finding,
author = {Gama, Nicolas and Nguyen, Phong Q.},
title = {Finding Short Lattice Vectors within Mordell's Inequality},
year = {2008},
isbn = {9781605580470},
publisher = {Association for Computing Machinery},
address = {New York, NY, USA},
url = {https://doi.org/10.1145/1374376.1374408},
doi = {10.1145/1374376.1374408},
booktitle = {Proceedings of the Fortieth Annual ACM Symposium on Theory of Computing},
pages = {207–216},
numpages = {10},
location = {Victoria, British Columbia, Canada},
series = {STOC '08}
}

@inproceedings{micciancio2016practical,
author = {Micciancio, Daniele and Walter, Michael},
title = {Practical, Predictable Lattice Basis Reduction},
year = {2016},
isbn = {9783662498897},
publisher = {Springer-Verlag},
address = {Berlin, Heidelberg},
url = {https://doi.org/10.1007/978-3-662-49890-3_31},
doi = {10.1007/978-3-662-49890-3_31},
booktitle = {Proceedings, Part I, of the 35th Annual International Conference on Advances in Cryptology --- EUROCRYPT 2016 - Volume 9665},
pages = {820–849},
numpages = {30}
}

@InProceedings{aggarwal2016slide,
author="Aggarwal, Divesh
and Li, Jianwei
and Nguyen, Phong Q.
and Stephens-Davidowitz, Noah",
editor="Micciancio, Daniele
and Ristenpart, Thomas",
title="Slide Reduction, Revisited---Filling the Gaps in SVP Approximation",
booktitle="Advances in Cryptology -- CRYPTO 2020",
year="2020",
publisher="Springer International Publishing",
address="Cham",
pages="274--295",
isbn="978-3-030-56880-1"
}

@article{schnorr1987hierarchy,
  title={A hierarchy of polynomial time lattice basis reduction algorithms},
  author={Schnorr, Claus-Peter},
  journal={Theoretical computer science},
  volume={53},
  number={2-3},
  pages={201--224},
  year={1987},
  publisher={Elsevier}
}

@article{schnorr1994lattice,
    author = {Schnorr, C. P. and Euchner, M.},
    title = {Lattice Basis Reduction: Improved Practical Algorithms and Solving Subset Sum Problems},
    year = {1994},
    issue_date = {Sept. 7, 1994},
    publisher = {Springer-Verlag},
    address = {Berlin, Heidelberg},
    volume = {66},
    number = {2},
    issn = {0025-5610},
    url = {https://doi.org/10.1007/BF01581144},
    doi = {10.1007/BF01581144},
    journal = {Math. Program.},
    month = sep,
    pages = {181–199},
    numpages = {19}
}

@article{nguyen2009lll,
    author = {Nguyen, Phong Q. and Stehlé, Damien},
    title = {An LLL Algorithm with Quadratic Complexity},
    journal = {SIAM Journal on Computing},
    volume = {39},
    number = {3},
    pages = {874-903},
    year = {2009},
    doi = {10.1137/070705702}
}

@inproceedings{bruna2020continuous,
  title={Continuous LWE},
  author={Bruna, Joan and Regev, Oded and Song, Min Jae and Tang, Yi},
  booktitle={Proceedings of the 53rd Annual ACM SIGACT Symposium on Theory of Computing},
  year={2021}
}

@misc{arous2021online,
    title={Online stochastic gradient descent on non-convex losses from high-dimensional inference}, 
    author={Gerard Ben Arous and Reza Gheissari and Aukosh Jagannath},
    year={2021},
    eprint={2003.10409},
    archivePrefix={arXiv},
    primaryClass={stat.ML}
}

@book{ebeling1994lattices,
    author = {Ebeling, Wolfgang and Hirzebruch, Friedrich},
    title = {Lattices and Codes: A Course Partially Based on Lectures by F. Hirzebruch},
    year = {1994},
    isbn = {3528064978},
    publisher = {Informatica International, Incorporated}
}

@inbook{simon2010selected,
author="Simon, Denis",
editor="Nguyen, Phong Q.
and Vall{\'e}e, Brigitte",
title="Selected Applications of LLL in Number Theory",
bookTitle="The LLL Algorithm: Survey and Applications",
year="2010",
publisher="Springer Berlin Heidelberg",
address="Berlin, Heidelberg",
pages="265--282",
isbn="978-3-642-02295-1",
doi="10.1007/978-3-642-02295-1_7",
url="https://doi.org/10.1007/978-3-642-02295-1_7"
}

@misc{jaganathan2015phase,
    title={Phase Retrieval: An Overview of Recent Developments},
    author={Jaganathan, Kishore and Eldar, Yonina C and Hassibi, Babak},
    eprint={1510.07713},
    archivePrefix={arxiv},
    year={2015}
}

@article{donoho2009message,
	author = {Donoho, David L. and Maleki, Arian and Montanari, Andrea},
	title = {Message-passing algorithms for compressed sensing},
	volume = {106},
	number = {45},
	pages = {18914--18919},
	year = {2009},
	doi = {10.1073/pnas.0909892106},
	publisher = {National Academy of Sciences},
	issn = {0027-8424},
	URL = {https://www.pnas.org/content/106/45/18914},
	eprint = {https://www.pnas.org/content/106/45/18914.full.pdf},
	journal = {Proceedings of the National Academy of Sciences}
}

@inproceedings{rangan2011generalized,
  title={Generalized approximate message passing for estimation with random linear mixing},
  author={Rangan, Sundeep},
  booktitle={2011 IEEE International Symposium on Information Theory Proceedings},
  pages={2168--2172},
  year={2011},
  organization={IEEE}
}

@misc{maillard2020phase,
  title={Phase retrieval in high dimensions: Statistical and computational phase transitions},
  author={Maillard, Antoine and Loureiro, Bruno and Krzakala, Florent and Zdeborov{\'a}, Lenka},
  eprint={2006.05228},
  archivePrefix={arxiv},
  year={2020}
}

@misc{mannelli2020complex,
  title={Complex dynamics in simple neural networks: Understanding gradient flow in phase retrieval},
  author={Mannelli, Stefano Sarao and Biroli, Giulio and Cammarota, Chiara and Krzakala, Florent and Urbani, Pierfrancesco and Zdeborov{\'a}, Lenka},
  eprint={2006.06997},
  archivePrefix={arxiv},
  year={2020}
}

@article{chen2019gradient,
  title={Gradient descent with random initialization: Fast global convergence for nonconvex phase retrieval},
  author={Chen, Yuxin and Chi, Yuejie and Fan, Jianqing and Ma, Cong},
  journal={Mathematical Programming},
  volume={176},
  number={1},
  pages={5--37},
  year={2019},
  publisher={Springer}
}

@article{mannelli2020marvels,
  title = {Marvels and Pitfalls of the Langevin Algorithm in Noisy High-Dimensional Inference},
  author = {Sarao Mannelli, Stefano and Biroli, Giulio and Cammarota, Chiara and Krzakala, Florent and Urbani, Pierfrancesco and Zdeborov\'a, Lenka},
  journal = {Phys. Rev. X},
  volume = {10},
  issue = {1},
  pages = {011057},
  numpages = {41},
  year = {2020},
  publisher = {American Physical Society},
  doi = {10.1103/PhysRevX.10.011057},
  url = {https://link.aps.org/doi/10.1103/PhysRevX.10.011057}
}

@article{goldstein2018phasemax,
  title={Phasemax: Convex phase retrieval via basis pursuit},
  author={Goldstein, Tom and Studer, Christoph},
  journal={IEEE Transactions on Information Theory},
  volume={64},
  number={4},
  pages={2675--2689},
  year={2018},
  publisher={IEEE}
}

@article{aubin2019committee,
  title={The committee machine: Computational to statistical gaps in learning a two-layers neural network},
  author={Aubin, Benjamin and Maillard, Antoine and Barbier, Jean and Krzakala, Florent and Macris, Nicolas and Zdeborov{\'a}, Lenka},
  journal={Journal of Statistical Mechanics: Theory and Experiment},
  volume={2019},
  number={12},
  pages={124023},
  year={2019},
  publisher={IOP Publishing}
}

@article{BALAN2006345,
title = {On signal reconstruction without phase},
journal = {Applied and Computational Harmonic Analysis},
volume = {20},
number = {3},
pages = {345-356},
year = {2006},
issn = {1063-5203},
doi = {https://doi.org/10.1016/j.acha.2005.07.001},
url = {https://www.sciencedirect.com/science/article/pii/S1063520305000667},
author = {Radu Balan and Pete Casazza and Dan Edidin}
}

@article{szarek1991condition,
title = {Condition numbers of random matrices},
journal = {Journal of Complexity},
volume = {7},
number = {2},
pages = {131-149},
year = {1991},
issn = {0885-064X},
doi = {https://doi.org/10.1016/0885-064X(91)90002-F},
url = {https://www.sciencedirect.com/science/article/pii/0885064X9190002F},
author = {Stanislaw J Szarek}
}

@article{mignacco2021stochasticity,
	author={Francesca Mignacco and Pierfrancesco Urbani and Lenka Zdeborova},
	title={Stochasticity helps to navigate rough landscapes: comparing gradient-descent-based algorithms in the phase retrieval problem},
	journal={Machine Learning: Science and Technology},
	year={2021}
}

@misc{mannelli2020optimization,
  title={Optimization and generalization of shallow neural networks with quadratic activation functions},
  author={Mannelli, Stefano Sarao and Vanden-Eijnden, Eric and Zdeborov{\'a}, Lenka},
  eprint={2006.15459},
  archivePrefix={arxiv},
  year={2020}
}

@inproceedings{mondelli2018fundamental,
  title={Fundamental limits of weak recovery with applications to phase retrieval},
  author={Mondelli, Marco and Montanari, Andrea},
  booktitle={Conference On Learning Theory},
  pages={1445--1450},
  year={2018},
  organization={PMLR}
}

@article{fienup1982phase,
  title={Phase retrieval algorithms: a comparison},
  author={Fienup, James R},
  journal={Applied optics},
  volume={21},
  number={15},
  pages={2758--2769},
  year={1982},
  publisher={Optical Society of America}
}

@article {Barbier5451,
	author = {Barbier, Jean and Krzakala, Florent and Macris, Nicolas and Miolane, L{\'e}o and Zdeborov{\'a}, Lenka},
	title = {Optimal errors and phase transitions in high-dimensional generalized linear models},
	volume = {116},
	number = {12},
	pages = {5451--5460},
	year = {2019},
	doi = {10.1073/pnas.1802705116},
	publisher = {National Academy of Sciences},
	issn = {0027-8424},
	URL = {https://www.pnas.org/content/116/12/5451},
	eprint = {https://www.pnas.org/content/116/12/5451.full.pdf},
	journal = {Proceedings of the National Academy of Sciences}
}

@article{lenstra1982factoring,
	Author = {Lenstra, Arjen Klaas and Lenstra, Hendrik Willem and Lov{\'a}sz, L{\'a}szl{\'o}},
	Journal = {Mathematische Annalen},
	Number = {4},
	Pages = {515--534},
	Publisher = {Springer},
	Title = {Factoring polynomials with rational coefficients},
	Volume = {261},
	Year = {1982}}

@article{v012a011,
 author = {Meka, Raghu and Nguyen, Oanh and Vu, Van},
 title = {Anti-concentration for Polynomials of Independent Random Variables},
 year = {2016},
 pages = {1--17},
 doi = {10.4086/toc.2016.v012a011},
 publisher = {Theory of Computing},
 journal = {Theory of Computing},
 volume = {12},
 number = {11},
 URL = {http://www.theoryofcomputing.org/articles/v012a011},
}

@book{goldreich2001foundations, 
    place={Cambridge}, 
    title={Foundations of Cryptography}, 
    volume={1}, 
    DOI={10.1017/CBO9780511546891}, 
    publisher={Cambridge University Press}, 
    author={Goldreich, Oded}, 
    year={2001}}

@article{hadamard_ineq,
author="Hadamard, Jacques",
title="Resolution d'une question relative aux determinants",
journal="Bull. des Sciences Math.",
ISSN="",
publisher="",
year="1893",
month="",
volume="2",
number="",
pages="240-246",
URL="https://ci.nii.ac.jp/naid/20000814080/en/",
DOI="",
}

@inbook{rudelson_ICM,
author = { Mark   Rudelson  and  Roman   Vershynin },
title = {Non-asymptotic Theory of Random Matrices: Extreme Singular Values},
booktitle = {Proceedings of the International Congress of Mathematicians 2010 (ICM 2010)},
chapter = {},
pages = {1576-1602},
doi = {10.1142/9789814324359_0111},
URL = {https://www.worldscientific.com/doi/abs/10.1142/9789814324359_0111},
eprint = {https://www.worldscientific.com/doi/pdf/10.1142/9789814324359_0111}
}

@article{caron2005zero,
  title={The zero set of a polynomial},
  author={Caron, Richard and Traynor, Tim},
  journal={WSMR Report},
  pages={05--02},
  year={2005}
}

@book{wainwright_2019, 
place={Cambridge}, series={Cambridge Series in Statistical and Probabilistic Mathematics}, title={High-Dimensional Statistics: A Non-Asymptotic Viewpoint}, DOI={10.1017/9781108627771}, publisher={Cambridge University Press}, author={Wainwright, Martin J.}, year={2019}, collection={Cambridge Series in Statistical and Probabilistic Mathematics}}

@article{FriezeSubset,
Author = {Alan M. Frieze},
Journal = {SIAM J. Comput.},
Pages = {536-539},
Title = {On the Lagarias-Odlyzko Algorithm for the Subset Sum Problem},
Volume = {15},
Year = {1986}}

@misc{soltanolkotabi2017learning,
  title={Learning relus via gradient descent},
  author={Soltanolkotabi, Mahdi},
  eprint={1705.04591},
  archivePrefix={arxiv},
  year={2017}
}

@InProceedings{szorenyi2009characterizing,
author="Sz{\"o}r{\'e}nyi, Bal{\'a}zs",
editor="Gavald{\`a}, Ricard
and Lugosi, G{\'a}bor
and Zeugmann, Thomas
and Zilles, Sandra",
title="Characterizing Statistical Query Learning: Simplified Notions and Proofs",
booktitle="Algorithmic Learning Theory",
year="2009",
publisher="Springer Berlin Heidelberg",
address="Berlin, Heidelberg",
pages="186--200",
isbn="978-3-642-04414-4"
}

@misc{frei2020agnostic,
  title={Agnostic learning of a single neuron with gradient descent},
  author={Frei, Spencer and Cao, Yuan and Gu, Quanquan},
  eprint={2005.14426},
  archivePrefix={arxiv},
  year={2020}
}

@misc{zhong2017recovery,
	author = {Zhong, Kai and Song, Zhao and Jain, Prateek and Bartlett, Peter L and Dhillon, Inderjit S},
	eprint = {1706.03175},
	archivePrefix={arxiv},
	title = {Recovery guarantees for one-hidden-layer neural networks},
	year = {2017}}

@misc{ge2017learning,
  title={Learning one-hidden-layer neural networks with landscape design},
  author={Ge, Rong and Lee, Jason D and Ma, Tengyu},
  eprint={1711.00501},
  archivePrefix={arxiv},
  year={2017}
}

@article{brutzkus2017sgd,
  title={SGD learns over-parameterized networks that provably generalize on linearly separable data},
  author={Brutzkus, Alon and Globerson, Amir and Malach, Eran and Shalev-Shwartz, Shai},
  eprint={1710.10174},
  archivePrefix={arxiv},
  year={2017}
}

@inproceedings{NEURIPS2018_ccc0aa1b,
 author = {Zadik, Ilias and Gamarnik, David},
 booktitle = {Advances in Neural Information Processing Systems},
 editor = {S. Bengio and H. Wallach and H. Larochelle and K. Grauman and N. Cesa-Bianchi and R. Garnett},
 pages = {},
 publisher = {Curran Associates, Inc.},
 title = {High Dimensional Linear Regression using Lattice Basis Reduction},
 url = {https://proceedings.neurips.cc/paper/2018/file/ccc0aa1b81bf81e16c676ddb977c5881-Paper.pdf},
 volume = {31},
 year = {2018}
}

@misc{gamarnik2019inference,
    title={Inference in High-Dimensional Linear Regression via Lattice Basis Reduction and Integer Relation Detection}, 
    author={David Gamarnik and Eren C. Kızıldağ and Ilias Zadik},
    year={2019},
    eprint={1910.10890},
    archivePrefix={arXiv},
    primaryClass={math.ST}
}

@inproceedings{goel2020superpolynomial,
  title={Superpolynomial lower bounds for learning one-layer neural networks using gradient descent},
  author={Goel, Surbhi and Gollakota, Aravind and Jin, Zhihan and Karmalkar, Sushrut and Klivans, Adam},
  booktitle={International Conference on Machine Learning},
  pages={3587--3596},
  year={2020},
  organization={PMLR}
}

@inproceedings{diakonikolas2020approximation,
  title={Approximation schemes for relu regression},
  author={Diakonikolas, Ilias and Goel, Surbhi and Karmalkar, Sushrut and Klivans, Adam R and Soltanolkotabi, Mahdi},
  booktitle={Conference on Learning Theory},
  pages={1452--1485},
  year={2020},
  organization={PMLR}
}

@inproceedings{goel2019timeaccuracy,
 author = {Goel, Surbhi and Karmalkar, Sushrut and Klivans, Adam},
 booktitle = {Advances in Neural Information Processing Systems},
 editor = {H. Wallach and H. Larochelle and A. Beygelzimer and F. d\textquotesingle Alch\'{e}-Buc and E. Fox and R. Garnett},
 publisher = {Curran Associates, Inc.},
 title = {Time/Accuracy Tradeoffs for Learning a ReLU with respect to Gaussian Marginals},
 volume = {32},
 year = {2019}
}

@inproceedings{diakonikolas2020algorithms,
  title={Algorithms and sq lower bounds for pac learning one-hidden-layer relu networks},
  author={Diakonikolas, Ilias and Kane, Daniel M and Kontonis, Vasilis and Zarifis, Nikos},
  booktitle={Conference on Learning Theory},
  pages={1514--1539},
  year={2020},
  organization={PMLR}
}

@misc{allen2018learning,
  title={Learning and generalization in overparameterized neural networks, going beyond two layers},
  author={Allen-Zhu, Zeyuan and Li, Yuanzhi and Liang, Yingyu},
  eprint={1811.04918},
  archivePrefix={arxiv},
  year={2018}
}

@misc{daniely2021local,
  title={From Local Pseudorandom Generators to Hardness of Learning},
  author={Daniely, Amit and Vardi, Gal},
  eprint={2101.08303},
  archivePrefix={arxiv},
  year={2021}
}

@inproceedings{eldan2016power, 
    title = {The Power of Depth for Feedforward Neural Networks}, 
    author = {Ronen Eldan and Ohad Shamir}, 
    booktitle = {29th Annual Conference on Learning Theory}, 
    ages = {907--940}, 
    year = {2016}, 
    editor = {Vitaly Feldman and Alexander Rakhlin and Ohad Shamir}, 
    volume = {49}, 
    series = {Proceedings of Machine Learning Research}, 
    publisher = {PMLR} 
}

@inproceedings{shalev2017failures,
  title={Failures of gradient-based deep learning},
  author={Shalev-Shwartz, Shai and Shamir, Ohad and Shammah, Shaked},
  booktitle={International Conference on Machine Learning},
  pages={3067--3075},
  year={2017},
  organization={PMLR}
}

@article{klivans2009cryptographic,
  title={Cryptographic hardness for learning intersections of halfspaces},
  author={Klivans, Adam R and Sherstov, Alexander A},
  journal={Journal of Computer and System Sciences},
  volume={75},
  number={1},
  pages={2--12},
  year={2009},
  publisher={Elsevier}
}

@article{kearns1994cryptographic,
author = {Kearns, Michael and Valiant, Leslie},
title = {Cryptographic Limitations on Learning Boolean Formulae and Finite Automata},
year = {1994},
issue_date = {Jan. 1994},
publisher = {Association for Computing Machinery},
address = {New York, NY, USA},
volume = {41},
number = {1},
issn = {0004-5411},
url = {https://doi.org/10.1145/174644.174647},
doi = {10.1145/174644.174647},
journal = {J. ACM},
month = jan,
pages = {67–95},
numpages = {29}
}

@article{khot2005hardness,
author = {Khot, Subhash},
title = {Hardness of Approximating the Shortest Vector Problem in Lattices},
year = {2005},
issue_date = {September 2005},
publisher = {Association for Computing Machinery},
address = {New York, NY, USA},
volume = {52},
number = {5},
issn = {0004-5411},
url = {https://doi.org/10.1145/1089023.1089027},
doi = {10.1145/1089023.1089027},
journal = {J. ACM},
month = sep,
pages = {789–808},
numpages = {20}
}

@inproceedings{haviv2007tensor,
author = {Haviv, Ishay and Regev, Oded},
title = {Tensor-Based Hardness of the Shortest Vector Problem to within Almost Polynomial Factors},
year = {2007},
isbn = {9781595936318},
publisher = {Association for Computing Machinery},
address = {New York, NY, USA},
url = {https://doi.org/10.1145/1250790.1250859},
doi = {10.1145/1250790.1250859},
booktitle = {Proceedings of the Thirty-Ninth Annual ACM Symposium on Theory of Computing},
pages = {469–477},
numpages = {9},
location = {San Diego, California, USA},
series = {STOC '07}
}

@misc{ducas2017dilithium,
    author = {Leo Ducas and Tancrede Lepoint and Vadim Lyubashevsky and Peter Schwabe and Gregor Seiler and Damien Stehle},
    title = {CRYSTALS -- Dilithium: Digital Signatures from Module Lattices},
    howpublished = {Cryptology ePrint Archive, Report 2017/633},
    year = {2017},
    note = {\url{https://eprint.iacr.org/2017/633}},
}

@inbook{micciancio2009lattice,
author="Micciancio, Daniele
and Regev, Oded",
title="Lattice-based Cryptography",
bookTitle="Post-Quantum Cryptography",
year="2009",
publisher="Springer Berlin Heidelberg",
address="Berlin, Heidelberg",
pages="147--191",
isbn="978-3-540-88702-7",
doi="10.1007/978-3-540-88702-7_5",
url="https://doi.org/10.1007/978-3-540-88702-7_5"
}

@phdthesis{NSDthesis,
  type = {PhD Thesis},
  title = {On the Gaussian measure over lattices},
  school = {New York University},
  author = {{Stephens-Davidowitz}, Noah},
  year = {2017}
}

@inproceedings{kharitonov1993cryptographic,
author = {Kharitonov, Michael},
title = {Cryptographic Hardness of Distribution-Specific Learning},
year = {1993},
isbn = {0897915917},
publisher = {Association for Computing Machinery},
address = {New York, NY, USA},
url = {https://doi.org/10.1145/167088.167197},
doi = {10.1145/167088.167197},
booktitle = {Proceedings of the Twenty-Fifth Annual ACM Symposium on Theory of Computing},
pages = {372–381},
numpages = {10},
location = {San Diego, California, USA},
series = {STOC '93}
}

@book{shalevshwartz2014understanding,
author = {Shalev-Shwartz, Shai and Ben-David, Shai},
title = {Understanding Machine Learning: From Theory to Algorithms},
year = {2014},
isbn = {1107057132},
publisher = {Cambridge University Press},
address = {USA}
}

@book{micciancio2002complexity,
	series = {The {Springer} {International} {Series} in {Engineering} and {Computer} {Science}},
	title = {Complexity of {Lattice} {Problems}: {A} {Cryptographic} {Perspective}},
	isbn = {978-0-7923-7688-0},
	shorttitle = {Complexity of {Lattice} {Problems}},
	url = {https://www.springer.com/gp/book/9780792376880},
	publisher = {Springer US},
	author = {Micciancio, Daniele and Goldwasser, Shafi},
	year = {2002},
	doi = {10.1007/978-1-4615-0897-7},
}

@techreport{alagic_status_2020,
	title = {Status {Report} on the {Second} {Round} of the {NIST} {Post}-{Quantum} {Cryptography} {Standardization} {Process}},
	language = {en},
	number = {NIST Internal or Interagency Report (NISTIR) 8309},
	institution = {National Institute of Standards and Technology},
	author = {Alagic, Gorjan and Alperin-Sheriff, Jacob and Apon, Daniel and Cooper, David and Dang, Quynh and Kelsey, John and Liu, Yi-Kai and Miller, Carl and Moody, Dustin and Peralta, Rene and Perlner, Ray and Robinson, Angela and Smith-Tone, Daniel},
	month = jul,
	year = {2020},
}

@article{shamir2018distribution,
author = {Shamir, Ohad},
title = {Distribution-Specific Hardness of Learning Neural Networks},
year = {2018},
issue_date = {January 2018},
volume = {19},
number = {1},
journal = {J. Mach. Learn. Res.},
pages = {1135–1163},
numpages = {29}
}

@inproceedings{song2017complexity,
 author = {Song, Le and Vempala, Santosh and Wilmes, John and Xie, Bo},
 booktitle = {Advances in Neural Information Processing Systems},
 editor = {I. Guyon and U. V. Luxburg and S. Bengio and H. Wallach and R. Fergus and S. Vishwanathan and R. Garnett},
 pages = {},
 publisher = {Curran Associates, Inc.},
 title = {On the Complexity of Learning Neural Networks},
 url = {https://proceedings.neurips.cc/paper/2017/file/a78482ce76496fcf49085f2190e675b4-Paper.pdf},
 volume = {30},
 year = {2017}
}

@misc{abbe2020poly,
  title={Poly-time universality and limitations of deep learning},
  author={Emmanuel Abbe and Colin Sandon},
  eprint={2001.02992},
  archivePrefix={arxiv},
  year={2020}
}

@book{vershynin2018high,
  title={High-dimensional probability: an introduction with applications in data science},
  author={Vershynin, Roman},
  isbn={9781108415194},
  lccn={2018016910},
  series={Cambridge Series in Statistical and Probabilistic Mathematics},
  year={2018},
  publisher={Cambridge University Press}
}

@article{feldman2017planted-clique,
author = {Feldman, Vitaly and Grigorescu, Elena and Reyzin, Lev and Vempala, Santosh S. and Xiao, Ying},
title = {Statistical algorithms and a lower bound for detecting planted cliques},
year = {2017},
issue_date = {June 2017},
publisher = {Association for Computing Machinery},
address = {New York, NY, USA},
volume = {64},
number = {2},
journal = {J. ACM},
articleno = {Article 8},
numpages = {37}
}

@inproceedings{regev2005lwe,
    author = {Regev, Oded},
    title = {On lattices, learning with errors, random linear codes, and cryptography},
    booktitle = {STOC},
    series = {STOC '05},
    year = {2005},
    isbn = {1-58113-960-8},
    pages = {84--93},
    numpages = {10},
    doi = {10.1145/1060590.1060603},
}

@article{aharonov2005conp,
author = {Aharonov, Dorit and Regev, Oded},
title = {Lattice problems in NP $\cap$ CoNP},
year = {2005},
issue_date = {September 2005},
publisher = {Association for Computing Machinery},
address = {New York, NY, USA},
volume = {52},
number = {5},
issn = {0004-5411},
url = {https://doi.org/10.1145/1089023.1089025},
doi = {10.1145/1089023.1089025},
journal = {J. ACM},
month = sep,
pages = {749–765},
numpages = {17}
}

@article{kearnsSQ1998,
 author = {Kearns, Michael},
 title = {Efficient noise-tolerant learning from statistical queries},
 journal = {J. ACM},
 issue_date = {Nov. 1998},
 volume = {45},
 number = {6},
 month = Nov,
 year = {1998},
 issn = {0004-5411},
 pages = {983--1006},
 numpages = {24},
}

@inproceedings{blum1994sq,
 author = {Blum, Avrim and Furst, Merrick and Jackson, Jeffrey and Kearns, Michael and Mansour, Yishay and Rudich, Steven},
 title = {Weakly learning DNF and characterizing statistical query learning using Fourier analysis},
 booktitle = {STOC},
 series = {STOC '94},
 year = {1994},
 isbn = {0-89791-663-8},
 pages = {253--262},
 numpages = {10},
 doi = {10.1145/195058.195147},
}
\newpage

\appendix

\section{Formal Setup}
\label{app:formal-setup}
In this section, we present the formal definitions of all problems required to state our hardness result (Theorem \ref{thm:clwe-hardness}). We begin with a description of average-case decision problems, of which the CLWE decision problem is a special instance~\cite{bruna2020continuous}.

\subsection{Average-Case Decision Problems}
\label{app:average-case-setup}
We introduce the notion of average-case decision problems (or simply binary hypothesis testing problems), based on~\cite{goldreich2001foundations}, where we refer the interested reader for more details. In such average-case decision problems the statistician receives $m$ samples from either a distribution $D$ or another distribution $D'$, and needs to decide based on the produced samples whether the generating distribution is $D$ or $D'.$ We assume that the statistician may use any, potentially randomized, algorithm $\sA$ which is a measurable function of the $m$ samples and outputs the Boolean decision $\{\mathrm{YES}, \mathrm{NO}\}$ corresponding to their prediction of whether $D$ or $D'$ respectively generated the observed samples. Now, for any Boolean-valued algorithm $\sA $ examining the samples, we define the \emph{advantage} of $\sA$ solving the decision problem, as the sequence of positive numbers
\begin{align*}
    \Bigl| \Pr_{x \sim D^{\otimes m}}[\sA(x) = \mathrm{YES}] - \Pr_{x \sim {D'}^{\otimes m}}[\sA(x) = \mathrm{YES}] \Bigr|
    \; .
\end{align*}
As mentioned above, we assume that the algorithm $\sA$ outputs two values ``YES'' or ``NO''. Furthermore, the output ``YES'' means that algorithm $\sA$ has decided that the given samples $x $ comes from the distribution $D$, and ``NO'' means that $\sA$ decided that $x$ comes from the alternate distribution $D'$. Therefore, naturally the advantage quantifies by how much the algorithm is performing better than just deciding with probability $1/2$ between the two possibilities.

Our setup requires two standard adjustments to the setting described above. First, in our setup we consider a sequence of distinguishing problems, indexed by a growing (dimension) $d \in \NN$, and for every $d$ we receive $m=m(d)$ samples and seek to distinguish between two distributions $D_d$ and $D'_d$. Now, for any sequence of  Boolean-valued algorithms $\sA=\sA_d $ examining the samples, we naturally define the \emph{advantage} of $\sA$ solving the sequence of decision problems, as the sequence of positive numbers
\begin{align*}
    \Bigl| \Pr_{x \sim D_d^{\otimes m}}[\sA(x) = \mathrm{YES}] - \Pr_{x \sim {D'}_d^{\otimes m}}[\sA(x) = \mathrm{YES}] \Bigr|
    \; .
\end{align*}
As a remark, notice that any such distinguishing algorithm $\mathcal{A}$ required to terminate in at most time $T=T(d),$ is naturally implying that the algorithm has access to at most $m \leq T$ samples.

Now, as mentioned above, we require another adjustment. We assume that the distributions $D_d,D'_d$ are each generating $m$ samples in two stages: first by drawing a common structure for all samples, unknown to the statistician (also usually called in the statistics literature as a latent variable), which we call $s$, and second by drawing some additional and independent-per-sample randomness. In CLWE, $s$ corresponds to the hidden vector $w$ chosen uniformly at random from the unit sphere and the additional randomness per sample comes from the Gaussian random variables $x_i$. Now, to appropriately take into account this adjustment, we define the \emph{advantage} of a sequence of algorithms $\sA=\{\sA_{d}\}_{d \in \NN}$ solving the \emph{average-case} decision problem of distinguishing two distributions $D_{d, s}$ and $D'_{d, s}$ parametrized by $d$ and some latent variable $s$ chosen from some distribution $\sS_d$, as
\begin{align*}
    \Bigl| \Pr_{s \sim \sS_d, x \sim D_{d,s}^{\otimes m}}[\sA(x) = \mathrm{YES}] - \Pr_{s \sim \sS_d, x \sim {D'}_{d,s}^{\otimes m}}[\sA(x) = \mathrm{YES}] \Bigr|
    \; .
\end{align*}

Finally, we say that algorithm $\sA=\{\sA_d\}_{d\in\NN}$ has \emph{non-negligible advantage} if its advantage is at least an inverse polynomial function of $d$, i.e., a function behaving as $\Omega(d^{-c})$ for some constant $c > 0$.

\subsection{Decision and Phaseless CLWE}
\label{app:decision-and-phaseless-clwe}
We now give a formal definition of the decision CLWE problem, continuing the discussion from Section \ref{sec:definitions-and-notations}. We also introduce the phaseless-CLWE distribution, which can be seen as the CLWE distribution $A_{\bw, \beta, \gamma}$ defined in \eqref{CLWE}, with the absolute value function applied to the labels (recall that we take representatives in $[-1/2,1/2)$ for the $\mathrm{mod} \; 1$ operation). The Phaseless-CLWE distribution is, at an intuitive level, useful for stating and proving guarantees of our LLL algorithm in the exponentially small noise regime for learning the cosine neuron (See Section~\ref{LLL_main} and Appendix~\ref{LLL_app}).

\begin{definition}[Decision-CLWE] For parameters $\beta, \gamma > 0$, the average-case decision problem $\clwe_{\beta, \gamma}$ is to distinguish from i.i.d. samples the following two distributions over $\mathbb{R}^d \times [-1/2,1/2)$ with non-negligible advantage: (1) the CLWE distribution $A_{\bw, \beta, \gamma},$ per \eqref{CLWE}, for some uniformly random unit vector $w \in S^{d-1}$ (which is fixed for all samples), and (2) $N(0,I_d) \times U ([-1/2,1/2])$.
\end{definition}

\paragraph{Phaseless-CLWE.} We define the Phaseless-CLWE distribution on dimension $d$ with frequency $\gamma$, $\beta$-bounded adversarial noise, hidden direction $w$ to be the distribution of the pair $(x,z)\in \mathbb{R}^d \times [0,1/2]$ where $x \iiddistr N(0,I_d) $ and 
\begin{align}
z=\eps ( \gamma \inner{x,w}+\xi) \mod 1 \label{CLWE_pl}
\end{align}
for some $\eps \in \{-1,1\}$ such that $z \ge 0$, and bounded noise $|\xi| \leq \beta.$ 

\subsection{Worst-Case Lattice Problems}
\label{app:lattice-problems}
We begin with a definition of a lattice. A \emph{lattice} is a discrete additive subgroup of $\mathbb{R}^d$. In this work, we assume all lattices are full rank, i.e., their linear span is $\mathbb{R}^d$.
For a $d$-dimensional lattice $\Lambda$, a set of linearly independent vectors $\{b_1, \dots, b_d\}$ is called a \emph{basis} of $\Lambda$ if $\Lambda$ is generated by the set, i.e., $\Lambda = B \mathbb{Z}^d$ where $B = [b_1, \dots, b_d]$. Formally,
\begin{definition}
Given linearly independent $b_1,\ldots, b_d \in \mathbb{R}^d$, let 
\begin{align} \Lambda=\Lambda(b_1,\ldots,b_d)=\left\{\sum_{i=1}^{d} \lambda_i b_i : \lambda_i \in \mathbb{Z}, i=1,\ldots,d \right\}~, 
\end{align} 
which we refer to as the lattice generated by $b_1,\ldots,b_d$.
\end{definition}

We now present a worst-case \emph{decision} problem on lattices called GapSVP. In GapSVP, we are given an instance of the form $(\Lambda,t)$, where $\Lambda$ is a $d$-dimensional lattice and $t \in \RR$, the goal is to distinguish between the case where $\lambda_1(\Lambda)$, the $\ell_2$-norm of the shortest non-zero vector in $\Lambda$, satisfies $\lambda_1(\Lambda) < t$ from the case where $\lambda_1(\Lambda) \ge \alpha(d) \cdot t$ for some ``gap'' $\alpha(d) \ge 1$. Given a decision problem, it is straightforward to conceive of its search variant. That is, given a $d$-dimensional lattice $\Lambda$, approximate $\lambda_1(\Lambda)$ up to factor $\alpha(d)$. Note that the search version, which we call $\alpha$-approximate SVP in the main text, is \emph{harder} than its decision variant, since an algorithm for the search variant immediately yields an algorithm for the decision problem. Hence, the worst-case hardness of decision problems implies the hardness of their search counterparts. We note that GapSVP is known to be NP-hard for ``almost'' polynomial approximation factors, that is, $2^{(\log d)^{1-\eps}}$ for any constant $\eps > 0$, assuming problems in $\mathrm{NP}$ cannot be solved in quasi-polynomial time~\cite{khot2005hardness,haviv2007tensor}. As mentioned in the introduction of the paper, the problem is strongly believed to be computationally hard (even with quantum computation), for \emph{any} polynomial approximation factor $\alpha(d)$ \cite{micciancio2009lattice}.

Below we present formal definitions of two of the most fundamental lattice problems, GapSVP and the Shortest Independent Vectors Problem (SIVP). The SIVP problem, similar to GapSVP, is also believed to be computationally hard (even with quantum computation) for \emph{any} polynomial approximation factor $\alpha(d)$. Interestingly, the hardness of CLWE can also be based on the worst-case hardness of SIVP~\cite{bruna2020continuous}.

\begin{definition}[GapSVP] 
For an approximation factor $\alpha = \alpha(d)$, an instance of $\mathrm{GapSVP}_\alpha$ is given by an $d$-dimensional lattice $\Lambda$ and a number $t > 0$. In \textnormal{YES} instances, $\lambda_1(\Lambda) \leq t$, whereas in \textnormal{NO} instances, $\lambda_1(\Lambda) > \alpha \cdot t$.
\end{definition}

\begin{definition}[SIVP]
For an approximation factor $\alpha = \alpha(d)$, an instance of $\mathrm{SIVP}_\alpha$ is given by an $d$-dimensional lattice $\Lambda$. The goal is to output a set of $d$ linearly independent lattice vectors of length at most $\alpha \cdot \lambda_d(\Lambda)$.
\end{definition}

\section{Exponential-Time Algorithm: Constant Noise}
\label{app:it-ub}

\begin{algorithm}[t]
\caption{Information-theoretic recovery algorithm for learning cosine neurons (Restated)}
\label{alg:it-recovery-app}
\KwIn{Real numbers $\gamma=\gamma(d) > 1$, $\beta=\beta(d) $, and a sampling oracle for the cosine distribution \eqref{cosine-dist} with frequency $\gamma$, $\beta$-bounded noise, and hidden direction $w$.}
\KwOut{Unit vector $\hat{w} \in S^{d-1}$ s.t. $\min \{\|\hat{w}-w\|_2,\|\hat{w}+w\|_2\}=O(\arccos(1-\beta)/\gamma)$.}
\vspace{1mm} \hrule \vspace{1mm}
Let $\tau = \arccos(1-\beta)/(2\pi)$, $\eps = 2\tau/\gamma$, $m = 64d\log(1/\eps)$, and let $\sC$ be an $\eps$-cover of the unit sphere $S^{d-1}$. Draw $m$ samples $\{(x_i,y_i)\}_{i=1}^m$ from the cosine distribution \eqref{cosine-dist}. \\
\For{$i=1$ \KwTo $m$}{
    $z_i = \arccos(y_i)/(2\pi)$
    }
\For{$v \in \sC$}{
    Compute $T_v = \frac{1}{m}\sum_{i=1}^{m} \one \left[|\gamma \langle v, x_i \rangle - z_i \mod 1| \le 3\tau \right] +  \one \left[|\gamma \langle v, x_i \rangle + z_i \mod 1| \le 3\tau \right]$
        }
\Return $\hat{w} = \arg\max_{v \in \sC} T_v$.
\end{algorithm}

We provide full details of the proof of Theorem~\ref{thm:it-ub-cosine-main}, restated as Corollary~\ref{cor:it-ub-cosine-app} at the end of this section. Algorithm~\ref{alg:it-recovery}, the recovery algorithm in the main text, is restated as Algorithm~\ref{alg:it-recovery-app} here. The goal of Algorithm~\ref{alg:it-recovery-app} is to use $m=\poly(d)$ samples to recover in polynomial-time the hidden direction $w \in S^{d-1}$, in the $\ell_2$ sense. More concretely, the goal is to compute an estimator $\hat{w}=\hat{w}((x_i,z_i)_{i=1,\ldots,m})$ for which it holds $\min\{\|\hat{w}-w\|^2_2,\|\hat{w}+w\|^2_2\} = o(1/\gamma^2),$ with probability $1-\exp(-\Omega(d))$. 

We first start with Lemma~\ref{lem:cosine-to-phaseless-reduction}, which reduces the recovery problem under the cosine distribution (See Eq.~\eqref{cosine-dist}) to the recovery problem under the phaseless CLWE distribution (See Appendix~\ref{app:decision-and-phaseless-clwe}). Then, we prove Lemma~\ref{lem:it-ub-phaseless-clwe}, which states that there is an exponential-time algorithm for recovering the hidden direction $w \in S^{d-1}$ in Phaseless-CLWE under sufficiently small adversarial noise. Theorem~\ref{thm:it-ub-cosine-main} follows from Lemmas~\ref{lem:cosine-to-phaseless-reduction} and~\ref{lem:it-ub-phaseless-clwe}.

\begin{lemma}
\label{lem:cosine-to-phaseless-reduction}
Assume $\beta \in [0,1]$. Suppose that one receives a sample $(x,\tilde{z})$ from the cosine distribution on dimension $d$ with frequency $\gamma$ under $\beta$-bounded adversarial noise. Let $\bar{z}:=\sgn(\tilde{z})\min(1,|\tilde{z}|)$. Then, the pair $(x,\arccos(\bar{z})/(2\pi) \mod 1)$ is a sample from the Phaseless-CLWE distribution on dimension $d$ with frequency $\gamma$ under $\frac{1}{2\pi}\arccos(1-\beta)$-bounded adversarial noise.
\end{lemma}

\begin{proof}
Recall $\tilde{z}=\cos(2\pi (\gamma \inner{w,x}))+\xi, $ for $x \sim N(0,I_d)$ and $|\xi| \leq \beta.$ It suffices to show that 
\begin{align}
\label{eq:goal-it-1}
\frac{1}{2\pi}\arccos(\bar{z})=\epsilon\gamma \inner{w,x}+\xi' \mod 1
\end{align} for some $\epsilon \in \{-1,1\}$ and $\xi' \in \mathbb{R}$ with $|\xi'| \leq \frac{1}{2\pi}\arccos(1-\beta).$

First, notice that we may assume that without loss of generality $\bar{z}=\tilde{z}$. Indeed, assume for now $\tilde{z}>1$. The case $\tilde{z} < -1$ can be shown with almost identical reasoning. From the definition of $\tilde{z}$, it must hold that $\xi>0$ and $\tilde{z} \leq 1+\xi$. Hence \begin{align*}\bar{z}=1=\cos(2\pi (\gamma \inner{w,x}))+\tilde{\xi}.\end{align*}for $\tilde{\xi}:=\xi+1-\tilde{z} \in (0,\xi) \subseteq (0,\beta)$. Hence, $(x,\bar{z})$ is a sample from the cosine distribution in dimension $d$ with frequency $\gamma$ under $\beta$-bounded adversarial noise.

Now, given the above observation, to establish \eqref{eq:goal-it-1}, it suffices to show that for some $\epsilon \in \{-1,1\}$, and $K \in \mathbb{Z}$,
\begin{align*}
\left|\frac{1}{2\pi}\arccos(\tilde{z})-\epsilon\gamma \inner{w,x}-K\right| \leq \frac{1}{2\pi}\arccos(1-\beta)\;,
\end{align*}
or equivalently using that the cosine function is $2\pi$ periodic and even, it suffices to show that
\begin{align*}
|\arccos(\tilde{z})-\arccos(\cos(2\pi \gamma \inner{w,x}))| \leq \arccos(1-\beta)\;.
\end{align*}The result then follows from the definition of $\tilde{z}$ and the simple calculus Lemma \ref{lem:arccos}.
\end{proof}

We will use the following covering number bound for the running time analysis of Algorithm~\ref{alg:it-recovery-app}, and the proof of Lemma~\ref{lem:it-ub-phaseless-clwe}.
\begin{lemma}[{\cite[Corollary 4.2.13]{vershynin2018high}}]
\label{lem:sphere-covering}
The covering number $\sN$ of the unit sphere $S^{d-1}$ satisfies the following upper and lower bound for any $\eps > 0$
\begin{align}
  \left(\frac{1}{\epsilon}\right)^d \leq  \sN(S^{d-1},\eps) \le \left(\frac{2}{\eps}+1\right)^d\;.
\end{align}
\end{lemma}

\begin{remark}
\label{rem:covering-algo}
An $\eps$-cover for the unit sphere $S^{d-1}$ can be constructed in time $O(\exp(d\log (1/\epsilon)))$ by sampling $O(N\log N)$ unit vectors uniformly at random from $S^{d-1}$, where we denote by $N=\sN(S^{d-1},\eps)$. The termination time gurantee follows from  Lemma~\ref{lem:sphere-covering} and the property holds with probability $1-\exp(-\Omega(d))$. We direct the reader for a complete proof of this fact in Appendix~\ref{app:covering-algo-random}.
\end{remark}

Now we prove our main lemma, which states that recovery of the hidden direction in Phaseless-CLWE under adversarial noise is possible in exponential time, when the noise level $\beta$ is smaller than a small constant.

\begin{algorithm}[t]
\caption{Information-theoretic recovery algorithm for learning the Phaseless-CLWE}
\label{alg:it-recovery-app_CLWE}
\KwIn{Real numbers $\gamma=\gamma(d) > 1$, $\beta=\beta(d) $, and a sampling oracle for the phaseless-CLWE distribution \eqref{CLWE_pl} with frequency $\gamma$, $\beta$-bounded noise, and hidden direction $w$.}
\KwOut{Unit vector $\hat{w} \in S^{d-1}$ s.t. $\min \{\|\hat{w}-w\|_2,\|\hat{w}+w\|_2\}=O(\beta/\gamma)$.}
\vspace{1mm} \hrule \vspace{1mm}
Let $\eps = 2\tau/\beta$, $m = 64d\log(1/\eps)$, and let $\sC$ be an $\eps$-cover of the unit sphere $S^{d-1}$. Draw $m$ samples $\{(x_i,z_i)\}_{i=1}^m$ from the phaseless CLWE distribution \eqref{CLWE_pl}. \\

\For{$v \in \sC$}{
    Compute $T_v = \frac{1}{m}\sum_{i=1}^{m} \one \left[|\gamma \langle v, x_i \rangle - z_i \mod 1| \le 3\beta \right] +  \one \left[|\gamma \langle v, x_i \rangle + z_i \mod 1| \le 3\beta \right]$
        }
\Return $\hat{w} = \arg\max_{v \in \sC} T_v$.
\end{algorithm}

\begin{lemma}[Information-theoretic upper bound for recovery of Phaseless-CLWE]
\label{lem:it-ub-phaseless-clwe}
Let $d \in \NN$ and let $\gamma = \gamma(d) > 1$, and $\beta = \beta(d) \in (0,1/400)$. Moreover, let $P$ be the Phaseless-CLWE distribution with frequency $\gamma$, $\beta$-bounded adversarial noise, and hidden direction $w$. Then, there exists an $\exp(O(d\log (\gamma/\beta)))$-time algorithm, described in Algorithm \ref{alg:it-recovery-app_CLWE}, using $O(d\log(\gamma/\beta))$ samples from $P$ that outputs a direction $\hat{w} \in S^{d-1}$ satisfying \begin{align}\label{ew:guar_CLWE_p}\min(\|\hat{w}-w\|^2_2,\|\hat{w}+w\|^2_2) \le 40000\beta^2/\gamma^2\end{align} with probability $1-\exp(-\Omega(d))$.
\end{lemma}
\begin{proof}
Let $P$ be the Phaseless-CLWE distribution and $w$ be the hidden direction of $P$. We describe first the steps of the Algorithm  \ref{alg:it-recovery-app_CLWE} we use and then prove its correctness.

Let $\eps = \beta/\gamma$, and $\sC$ be an $\eps$-cover of the unit sphere. By Remark~\ref{rem:covering-algo}, we can construct such an $\epsilon$-cover $\sC$ in $O(\exp(d \log (\gamma/\beta)))$ time such that $|\sC| \le \exp(O(d\log(\gamma/\beta)))$. We now draw $m=36d\log(\gamma/\beta)$ samples $\{(x_i,z_i)\}_{i=1}^{m}$ from $P$. Now, given these samples and the threshold value $t=3\beta$, we compute for each of the $|\sC| \le \exp(O(d\log(\gamma/\beta)))$ directions $v \in \sC$ the following counting statistic,
\begin{align*}
    T_v := \frac{1}{m}\sum_{i=1}^{m} \left(\one \left[|\gamma \langle v, x_i \rangle - z_i \mod 1| \le 3 \beta \right] + \one \left[|\gamma \langle v, x_i \rangle + z_i \mod 1|  \le 3 \beta \right] \right)\;.
\end{align*}
$T_v$ is simply measuring the fraction of the $z_i$'s falling in a $\mathrm{mod} \; 1$-width $3\beta$ interval around $\gamma \inner{v,x_i}$ or $-\gamma \inner{v,x_i},$ accounting for the uncertainty over the sign $\epsilon \in \{-1,1\}$ in the definition of Phaseless-CLWE.  We then suggest our estimator to be $\hat{w} = \arg\max_{v \in \sC} T_v.$ The algorithm can be clearly implemented in $|\sC| \le \exp(O(d\log(\gamma/\beta)))$ time.

We prove the correctness of our algorithm by establishing \eqref{ew:guar_CLWE_p} with probability $1-\exp(-\Omega(d))$. We first show that some direction $v \in \sC$ which is sufficiently close to $w$ satisfies $T_v \geq \frac{2}{3}$ with probability $1-\exp(-\Omega(d))$. Indeed, let us consider $v \in \sC$ be a direction such that $\|w-v\|_2 \le \eps=\beta/\gamma$. The existence of such a $v$ follows from our definition of $\sC.$ We denote for every $i=1,\ldots,m$ by $\epsilon_i \in \{-1,1\}$ the sign chosen by the $i$-th sample, and \begin{align}\label{eq:xi_IT} \xi_i =  z_i - \epsilon_i \gamma \langle w, x_i \rangle\end{align} the adversarial noise added to the sample per \eqref{CLWE_pl}. Now notice that the following trivially holds almost surely for $v,$ \begin{align*}
    T_v \geq \frac{1}{m}\sum_{i=1}^{m} \one \left[|\gamma \langle v, x_i \rangle - \epsilon_i z_i \mod 1| \le 3 \beta \right]  \;.
\end{align*}By elementary algebra and using \eqref{eq:xi_IT} we have $\epsilon_i z_i -\gamma \langle v, x_i \rangle \mod 1=\gamma \langle  w-v, x_i \rangle + \xi_i \mod 1.$ Combining the above it suffices to show that  
\begin{align}
\label{eq:goal-it}
  \frac{1}{m}\sum_{i=1}^{m} \one \left[|\gamma \langle  w-v, x_i \rangle + \xi_i  \mod 1| \le 3 \beta \right] \geq \frac{2}{3}  \;.
\end{align}
with probability $1-\exp(-\Omega(d))$.

Now we have \begin{align*} \Pr [|\gamma \langle  w-v, x_i \rangle + \xi_i \mod 1|  \le 3\beta] & \ge \Pr[|\gamma \langle w-v, x_i \rangle  \mod 1| \le 2\beta]\\
&\geq \Pr[|\gamma \langle w-v, x_i \rangle | \le 2\beta]
\end{align*}using for the first inequality that $\beta$-bounded adversarial noise cannot move points within distance $2\beta$ to the origin to locations with distance larger than $3\beta$ from the origin and for the second the trivial inequality $|a| \geq |a \mod 1|.$ Now, notice that $\gamma \langle  w-v, x_i \rangle $ is  distributed as a sample from a  Gaussian (see Definition~\ref{def:periodic-Gaussian}) with mean 0 and standard deviation at most $\gamma \|v-w\|_2 \leq \gamma \epsilon=\beta$. Hence, we can immediately conclude $\Pr[|\gamma \langle w-v, x_i \rangle | \le 2\beta] \ge 3/4$ since the probability of a Gaussian vector falling within 2 standard deviations of the mean is at least 0.95. By a standard application of Hoeffding's inequality, we can then conclude that \eqref{eq:goal-it} holds with probability $ 1-\exp(-\Omega(m)) = 1-\exp(-\Omega(d))$.

We now show that with probability $1-\exp(-\Omega(d))$ for any $v \in \sC$ which satisfies $\min(\|v - w \|_2,\|v+w\|_2) \ge 200\beta/\gamma$, it holds $T_v \le 1/2$. Notice that given the established existence of a $v$ which is $\beta/\gamma$-close to $w$ and satisfies $T_v \ge 2/3$, with probability $1-\exp(-\Omega(d))$, the result follows.  Let $v \in \sC$ be a direction satisfying $\|v - w\|_2 \ge 200\beta/\gamma$. Without loss of generality, assume that $\|v-w\|_2 \le \|v+w\|_2$. Then, using \eqref{eq:xi_IT} we have $\gamma \langle v, x_i \rangle - z_i=\gamma \langle v-\epsilon_i w, x_i \rangle - \epsilon_i \xi_i \mod 1$ and $\gamma \langle v, x_i \rangle + z_i=\gamma \langle v+\epsilon_i w, x_i \rangle + \epsilon_i \xi_i \mod 1.$ Hence, since $\eps \in \{-1,1\}, |\xi_i| \leq \beta$ for all $i=1,\ldots,m$ we have by a triangle inequality
\begin{align*}
    T_v \leq \frac{1}{m}\sum_{i=1}^{m} \left(\one \left[|\gamma \langle v-w, x_i \rangle  \mod 1| \le 4 \beta \right] + \one \left[|\gamma \langle v+w, x_i \rangle  \mod 1| \le 4 \beta \right]  \right)\;.
\end{align*}

Now by our assumption on $v$ both $\gamma \langle v-w, x_i \rangle$ and $\gamma \langle v+w, x_i \rangle$ are distributed as mean-zero Gaussians with standard deviation at least $\gamma\|w-v\|_2 \geq 200 \beta.$ Hence, both $\gamma \langle v-w, x_i \rangle \mod 1$ and $\gamma \langle v+w, x_i \rangle \mod 1$ are distributed as periodic Gaussians with width at least $200 \beta$ (see Definition~\ref{def:periodic-Gaussian}). By Claim~\ref{claim:gaussian-mod1} and the fact that $\beta < 1/400$, \begin{align*}
\Pr[|\gamma \langle v-w, x_i \rangle \mod 1| \le 4\beta] & \le 16\beta/(400\beta\sqrt{2\pi})\cdot (1+2(1+ (400\beta)^2)e^{-1/(160000\beta^2)} \\
& \le 4/(25 \sqrt{2\pi})<\frac{1}{12}.
\end{align*}By symmetry the same upper bound holds for $\Pr[|\gamma \langle v+w, x_i \rangle \mod 1| \le 4\beta].$ Hence, 
\begin{align*}
    \Pr_{(x_i, z_i) \sim P}\left[\{|\gamma \langle v-w, x_i \rangle \mod 1| \le 3\beta\}\cup \{|\gamma \langle v+w, x_i \rangle \mod 1 \mod 1| \le 3\beta\} \right] < 1/6\;.
\end{align*}
By a standard application of Hoeffding's inequality, we have
\begin{align*}
\Pr[T_v > 1/2] \le \exp(-m/18) \leq  \exp(-2d\log(1/\eps)), \end{align*} and by the union bound over all $v \in \sC$ satisfying $\|v-w\| \ge 200\beta/\gamma$, 
\begin{align}
    \Pr\left[\bigcup_{\|v-w\| \ge 200\beta/\gamma}\{T_v > 1/2\}\right]
    < |\sC|\cdot \exp(-2d\log(1/\eps)) = \exp(-\Omega(d)) \;. \nonumber
\end{align}
This completes the proof.
\end{proof}

Finally, we discuss the recovery in terms of samples from the cosine distribution.

\begin{corollary}[Restated Theorem~\ref{thm:it-ub-cosine-main}]
\label{cor:it-ub-cosine-app}
For some constants $c_0,C_0>0$ (e.g., $c_0=1-\cos(\pi/200), C_0=40000$) the following holds. Let $d \in \NN$ and let $\gamma = \gamma(d) > 1$, $\beta=\beta(d) \le c_0$, and $\tau = \frac{1}{2\pi}\arccos(1-\beta)$. Moreover, let $P$ be the cosine distribution with frequency $\gamma$, hidden direction $w$, and noise level $\beta$. Then, there exists an $\exp(O(d\log (\gamma/\tau)))$-time algorithm, described in Algorithm \ref{alg:it-recovery-app}, using $O(d\log(\gamma/\tau))$ i.i.d. samples from $P$ that outputs a direction $\hat{w} \in S^{d-1}$ satisfying $\min \{\|\hat{w}-w\|^2_2,\|\hat{w}+w\|^2_2\}  \le C_0 \tau^2 /\gamma^2$ with probability $1-\exp(-\Omega(d))$.
\end{corollary}
\begin{proof}
    We first define $m=O(d\log(\gamma/\beta))$ reflecting the sample size needed for the algorithm analyzed in Lemma \ref{lem:it-ub-phaseless-clwe} to work. We then draw $m$ samples $\{(x_i,\tilde{z}_i)\}_{i=1}^{m}$ from the cosine distribution. 
    From this point Algorithm \ref{alg:it-recovery-app} simply combines the reduction step of Lemma \ref{lem:cosine-to-phaseless-reduction} and then the algorithm described in the proof of Lemma \ref{lem:it-ub-phaseless-clwe}.
    
    Specifically, using Lemma \ref{lem:cosine-to-phaseless-reduction}, we can transform our i.i.d. samples to i.i.d. samples from the Phaseless CLWE distribution on dimension $d$ with frequency $\gamma$ under $\frac{1}{2\pi}\arccos(1-\beta)$-bounded adversarial noise. The transformation simply happens by applying the arccosine function to every projected $\tilde{z_i},$ so it takes $O(1)$ time per sample, a total of $O(m)$ steps. We then use the last step of Algorithm~\ref{alg:it-recovery-app} and employ Lemma \ref{lem:it-ub-phaseless-clwe} which analyzes Algorithm \ref{alg:it-recovery-app} to conclude that the output $\hat{w} \in S^{d-1}$ satisfies $\min(\|\hat{w}-w\|^2,\|\hat{w}+w\|^2) \le 40000\tau^2 /\gamma^2$ with probability $1-\exp(-\Omega(d))$.
\end{proof}

\section{Cryptographically-Hard Regime: Polynomially-Small Noise}
\label{app:cosine-clwe-hardness}
We give a full proof of Theorem~\ref{thm:cosine-clwe-hardness}, restated as Theorem~\ref{thm:cosine-clwe-hardness-app} here. Given Theorem~\ref{thm:cosine-clwe-hardness}, Corollary~\ref{cor:cosine-clwe-hardness}, also restated below as Corollary~\ref{cor:cosine-clwe-hardness-app}, follows from the hardness of CLWE~\cite{bruna2020continuous}.

\begin{theorem}[Restated Theorem~\ref{thm:cosine-clwe-hardness}]
\label{thm:cosine-clwe-hardness-app}
    Let $d \in \NN$, $\gamma = \omega(\sqrt{\log d}), \beta = \beta(d) \in (0,1)$. Moreover, let $L > 0$, let $\phi : \RR \rightarrow [-1,1]$ be an $L$-Lipschitz 1-periodic univariate function, and $\tau = \tau(d)$ be such that $\beta/(L\tau) = \omega(\sqrt{\log d})$. Then, a polynomial-time (improper) algorithm that weakly learns the function class $\sF_{\gamma}^\phi = \{f_{\gamma,w}(x) = \phi( \gamma \langle w, x \rangle) \mid w \in \sS^{d-1} \}$ over Gaussian inputs $x \iiddistr N(0,I_d)$ under $\beta$-bounded adversarial noise implies a polynomial-time algorithm for $\clwe_{\tau,\gamma}$.
\end{theorem}

\begin{proof}
Recall that a polynomial-time algorithm for $\clwe_{\tau,\gamma}$ refers to  distinguishing between $m$ samples $(x_i, z_i=\gamma \langle w,  x_i \rangle + \xi_i \mod 1)_{i=1,2,\ldots, m}$, where $x_i \sim N(0,I_d), \xi_i \sim N(0,\tau)$ and $w \sim U(S^{d-1}),$ from $m$ random samples $(x_i, z_i)_{i=1,2,\ldots,m}$, where $y_i \sim U([0,1])$ with non-negligible advantage over the trivial random guess (See Appendix~\ref{app:average-case-setup} and~\ref{app:decision-and-phaseless-clwe}).  We refer to the former sampling process as drawing $m$ i.i.d. samples from the CLWE distribution, where from now on we call $P$ for the CLWE distribution, and to the latter sampling process as drawing $m$ i.i.d. samples from the null distribution, which we denote by $Q$. Here, and everywhere in this proof, the number of samples $m$ denotes a quantity which depends polynomially on the dimension $d$.

Let $\eps = \eps(d) \in (0,1)$ be an inverse polynomial, and let $\sA$ be a polynomial-time learning algorithm that takes as input $m$ samples from $P$, and with probability $2/3$ outputs a hypothesis $h : \RR \rightarrow \RR$ such that $L_P(h) \le L_P(\EE[\phi(z)]) - \eps$. Since we are using the squared loss, we can assume without loss of generality that $h : \RR \rightarrow [-1,1]$ because clipping the output of the hypothesis $h$, i.e., $\tilde{h}(x) = \sgn(h)\cdot \max(|h(x)|,1)$ is always an improvement over $h$ pointwise because the labels are always inside the range $[-1,1]$.

Let $D$ be an unknown distribution on $2m$ i.i.d. samples, that is equal to either $P$ or $Q$. Our reduction consists of a statistical test that distinguishes between $D = P$ and $D = Q$. Our test is using the (successful in weakly learning $f_{\gamma,w}$ if $D=P$) predictor $h$ returned by $\sA$ on (some appropriate function of the first) $m$ out of the $2m$ samples drawn from $D$.  Then, we compute the empirical loss of $h$ on the remaining $m$ samples from $D$, and $m$ samples drawn from $Q$, respectively, and test
\begin{align}
\label{eq:clwe-test}
    \hat{L}_D(h) \le \hat{L}_Q(h) - \eps/4 \;.
\end{align}

We conclude $D = P$ if $h$ passes the test and $D = Q$ otherwise. The way we prove that this test succeeds with probability $2/3 - o(1)$, is by using the fact that $\sA$ outputs a hypothesis $h$ with $\eps$-edge with probability $2/3$ when given $m$ samples from $P$ as input. In the following, we now formally prove the correctness of this test.

We first assume $D=P$, and consider the first $m$ samples $(x_i,z_i)_{i=1,\ldots,m}$ drawn from $P$. Now observe the elementary equality that for all $v \in \mathbb{R}$ it holds $\phi(v \mod 1) = \phi(v).$ Hence,
\begin{align*}
    \phi(\gamma \langle w,  x_i \rangle + \xi_i) = \phi(z_i).
\end{align*}
Furthermore, notice that by the fact that the $\phi$ is an $L$-Lipschitz function we have
\begin{align} 
\label{lip_condition}
    \phi(\gamma \langle w,  x_i \rangle) + \tilde{\xi}_i = \phi(z_i)
\end{align}
for some $\tilde{\xi}_i \in [-L |\xi_i|, L |\xi_i|]$.
By Mill's inequality, for all $i=1,2,\ldots,m$ we have $\Pr[|\xi_i| > \beta/L] \le \sqrt{2/\pi}\exp(-\beta^2/(2L^2\tau^2))$. Since $\beta/(L\tau) = \omega(\sqrt{\log d})$, we conclude that 
\begin{align*}
    \Pr[\bigcup_{i=1}^m \left\{|\xi_i| > \beta/L \right\}] \le \sqrt{2/\pi} \cdot m\exp(-\beta^2/(8\pi^2\tau^2))=md^{-\omega(1)}=o(1) \;,
\end{align*}
where the last equality holds because $m$ depends polynomially on $d$. Hence, it holds that 
\begin{align*}
    |\xi_i'| \le L |\xi_i| \le \beta\;,
\end{align*}
for all $i=1,\ldots,m$ with probability $1-o(1)$ over the randomnesss of $\xi_i, i=1,2,\ldots,m.$ Combining the above with \eqref{lip_condition}, we conclude that with probability $1-o(1)$ over $\xi_i$, using our knowledge of $(x_i,z_i)$, we have at our disposal samples from the function $f_{\gamma,w}(x) = \phi(\gamma \langle w, x \rangle)$ corrupted by adversarial noise of magnitude at most $\beta$. Let us write by $\phi(P)$ the data distribution obtained by applying $\phi$ to labels of the samples from $P$, and similarly write $\phi(Q)$ for the null distribution $Q$.

By assumption and the above, given these samples $(x_i,\phi(z_i))_{i=1,2,\ldots,m}$ we have that $\sA$ outputs an hypothesis $h: \mathbb{R}^d \rightarrow [-1,1]$ such that for $m$ large enough, with probability at least $2/3$,
\begin{align*}
    L_{\phi(P)}(h) \leq L_{\phi(P)}\left(\EE_{(x,z) \sim P}[\phi(z)]\right)-\epsilon,
\end{align*}
for some $\epsilon = 1/\poly(d) >0$.

Now, note that by Claim~\ref{claim:gaussian-mod1}, the marginal distribution of $\phi(\gamma \langle w,x\rangle)$ is $2\exp(-2\pi^2\gamma^2)$-close in total variation distance to the distribution of $\phi(y)$, where $y \sim U([0,1])$. Moreover, notice that since the loss $\ell$ is continuous, and $h(x), x\in \mathbb{R}^d$ and of course $\phi(z), y \in \mathbb{R} $ both take values in $[-1,1]$, 
\begin{align}
\label{loss_bound}
\sup_{(x,y) \in \mathbb{R}^d \times \mathbb{R}} \ell(h(x), \phi(y)) \leq \sup_{(a,b) \in [-1,1]^d \times [-1,1]} \ell(a, b) \leq 4;.
\end{align}

Let us denote $c = \EE_{(x,y) \sim Q}[\phi(y)]$ for simplicity. Clearly $|c|,|\phi(y)| \leq 1$. Also, 
\begin{align*}
    |L_{\phi(P)}(c) - L_{\phi(Q)}(c))| &= \left| \EE_{(x,y) \sim P}[(\phi(y)-c)^2] - \EE_{(x,y) \sim Q}[(\phi(y)-c)^2]\right| \\
    &\le \int_{-1}^{1}\phi(y)^2 |P(y)-Q(y)|dy + 2c\int_{-1}^{1}|\phi(y)| |P(y)-Q(y)|dy \\
    &\le (1+2|c|)\int_{-1}^{1} |P(y)-Q(y)|dy \\
    &\le 6 \cdot TV(P_y,Q_y) \\
    &\le 12 \exp(-2\pi^2\gamma^2)\;.
\end{align*}
From the above, since $\EE_{z \sim P}[\phi(z)]$ is the optimal predictor for $P$ under the squared loss, we deduce
\begin{align*}
    L_{\phi(P)}\left(\EE_{(x,z) \sim P}[\phi(z)]\right) \le L_{\phi(P)}\left(\EE_{y \sim Q}[\phi(y)]\right) \le L_{\phi(Q)}\left(\EE_{y \sim Q}[\phi(y)]\right) + 12\exp(-2\pi^2\gamma^2)\;.
\end{align*}
Now since $\EE_{y \sim Q}[\phi(y)]$ is the optimal predictor for $Q$ under the squared loss, $L_{\phi(Q)}(\EE[\phi(y)]) \le L_{\phi(Q)}(h)$ for any predictor $h$. In addition, $\exp(-2\pi^2\gamma^2) = o(\eps)$ since $\gamma=\omega(\sqrt{\log d})$ and $\eps$ is an inverse polynomial in $d$. Hence, for $d$ large enough, with probability at least $2/3$ 
\begin{align}
    L_{\phi(P)}(h) &\leq L_{\phi(P)}(\EE[\phi(\gamma \langle w, x\rangle) ])-\epsilon \nonumber \\
    &\leq L_{\phi(Q)}(h) + 12\exp(-2\pi^2\gamma^2) -\eps \nonumber \\
    &\le L_{\phi(Q)}(h) - \eps/2 \;. \label{eq:population-loss}
\end{align}
Using the remaining $m$ samples from $P$, we now compute the empirical losses $\hat{L}_{\phi(P)}(h) = \frac{1}{m}\sum_{i=1}^m \ell(h(x_i), \phi(z_i))$, and $\hat{L}_{\phi(Q)}(h) = \frac{1}{m} \sum_{i=1}^m \ell(h(x_i), \phi(y_i))$, where $(x_i,z_i)$ are drawn from $P$ and $(x_i,y_i)$ are drawn from $Q$. By a standard use of Hoeffding's inequality, and the fact that the loss is bounded based on \eqref{loss_bound}, it follows that
\begin{align*}
    |\hat{L}_{\phi(P)}(h)-L_{\phi(P)}(h)| \leq \frac{\epsilon}{8}\;,
\end{align*}with probability $1-\exp (-\Omega(m))$ and respectively
\begin{align*}
    |\hat{L}_{\phi(Q)}(h)-L_{\phi(Q)}(h)| \leq \frac{\epsilon}{8}\;,
\end{align*}
with probability $1-\exp (-\Omega(m))$ for sufficiently large, but still polynomial in $d$, $m$. Combining the last two displayed equations with \eqref{eq:population-loss}, we have that, for $m$ large enough,  with probability at least $2/3 - o(1)$,
\begin{align*}
    \hat{L}_{\phi(P)}(h) \leq L_{\phi(P)}(h)+\frac{\epsilon}{8} \leq \hat{L}_{\phi(Q)}(h)-\frac{\epsilon}{4}.
\end{align*}
Hence, for $m$ large enough, with probability at least $2/3-o(1)$, the test correctly concludes $D=P$ or $D=Q$ by using the empirical loss $\hat{L}_{\phi(D)}(h),$ and comparing it with the value $\hat{L}_{\phi(Q)}(h)-\eps/4$. 
\end{proof}

\begin{corollary}[Restated Corollary~\ref{cor:cosine-clwe-hardness}]
\label{cor:cosine-clwe-hardness-app}
Let $d \in \NN$, $\gamma = \gamma(d) \ge 2\sqrt{d}$ and $\tau = \tau(d) \in (0,1)$ be such that $\gamma/\tau = \poly(d)$, and $\beta = \beta(d)$ be such that $\beta/\tau = \omega(\sqrt{\log d})$. Then, a polynomial-time algorithm that weakly learns the cosine neuron class $\sF_\gamma$ under $\beta$-bounded adversarial noise implies a polynomial-time quantum algorithm for $O(d/\tau)$-approximate $\mathrm{SVP}$.
\end{corollary}
\begin{proof}
The cosine function $\phi(z) = \cos(2\pi z)$ is $2\pi$-Lipschitz and 1-periodic. Hence, the result follows from Theorem~\ref{thm:cosine-clwe-hardness-app} with $L=2\pi$.
\end{proof}

\section{LLL-based Algorithm: Exponentially Small Noise}\label{LLL_app}

In this section we offer the required missing proofs from the Section \ref{LLL_main}.
\subsection{The LLL Algorithm: Background and the Proof of Theorem \ref{thm:LLL}}

The most crucial component of the algorithm analyzed in this section is an appropriate use of the LLL lattice basis reduction algorithm. The LLL algorithm receives as input $n$ linearly independent vectors  $v_1,\ldots, v_n \in \mathbb{Z}^n$ and outputs an integer combination of them with ``small" $\ell_2$ norm. Specifically, let us (re)-define the lattice generated by $n$ \emph{integer} vectors as simply the set of integer linear combination of these vectors.

\begin{definition}\label{dfn_lattice}
Given linearly independent $v_1,\ldots, v_n \in \mathbb{Z}^n$, let \begin{align} \Lambda=\Lambda(v_1,\ldots,v_n)=\left\{\sum_{i=1}^{n} \lambda_i v_i : \lambda_i \in \mathbb{Z}, i=1,\ldots,n \right\}~, 
\end{align} 
which we refer to as the lattice generated by integer-valued $v_1,\ldots,v_n.$ We also refer to $(v_1,\ldots,v_n)$ as an (ordered) basis for the lattice $\Lambda.$
\end{definition}

The LLL algorithm is defined to approximately solve the search version of the Shortest Vector Problem (SVP) on a lattice $\Lambda$, given a basis of it. We have already defined decision-SVP in Appendix~\ref{app:lattice-problems}. We define the search version below for completeness.
\begin{definition}
\label{def:gamma-approx-short}
An instance of the algorithmic  $\Delta$-approximate SVP for a lattice $\Lambda \subseteq \mathbb{Z}^n$ is as follows. Given a lattice basis $v_1,\dots,v_n \in \mathbb{Z}^n$ for the lattice, $\Lambda $; find a vector $\widehat{x}\in \Lambda$, such that
\begin{align*}
    \|\widehat{x}\|\leq \Delta \min_{x\in \Lambda, x\neq 0}\|x\|\;.    
\end{align*}
\end{definition}
The following theorem holds for the performance of the LLL algorithm, whose details can be found in \cite{lenstra1982factoring}.
\begin{theorem}[{\cite{lenstra1982factoring}}]
\label{LLL_thm_original}
There is an algorithm (namely the LLL lattice basis reduction algorithm), which receives as input a basis for a lattice $\Lambda$ given by $v_1,\ldots,v_n \in \mathbb{Z}^n$ which
\begin{itemize}
    \item[(1)] solves the $2^{\frac{n}{2}}$-approximate SVP for $\Lambda$ and,
    \item[(2)] terminates in time polynomial in $n$ and $\log \left(\max_{i=1}^n \|v_i\|_{\infty}\right).$
\end{itemize}
\end{theorem}

In this work, we use the LLL algorithm for an integer relation detection application.

\begin{definition}
An instance of the {\it integer relation detection problem} is as follows. Given a vector $b=(b_1,\dots,b_n)\in\mathbb{R}^n$, find an $m \in\mathbb{Z}^n\setminus\{{\bf 0}\}$, such that $\langle b,m \rangle=\sum_{i=1}^n b_im_i=0$. In this case, $m$ is said to be an integer relation for the vector $b$.
\end{definition}

We now establish Theorem \ref{thm:LLL}, by  proving following more general result. In particular, Theorem \ref{thm:LLL} follows from the theorem below by choosing $M = 2^{n+1}\|m'\|_2$ and using notation $m$ (used in Theorem \ref{thm:LLL}) instead of $m'$ (used in Theorem \ref{thm:LLL_app}), and $m'$ (used in Theorem \ref{thm:LLL}) instead of $t$ (used in Theorem \ref{thm:LLL_app}). 

The following theorem, is rigorously showing how the LLL algorithm can be used for integer relation detection. The proof of the theorem, is based upon some key ideas of the breakthrough use of the LLL algorithm to solve the average-case subset sum problem by Lagarias and Odlyzko \cite{Lagarias85}, and Frieze \cite{FriezeSubset}, and its recent extensions in the context of regression \cite{NEURIPS2018_ccc0aa1b, gamarnik2019inference}.

\begin{theorem}\label{thm:LLL_app}
Let $n,N \in \mathbb{Z}_{> 0}.$ Suppose $b \in (2^{-N}\mathbb{Z})^n$ with $b_1=1.$ Let also $m' \in \mathbb{Z}^n$ be an integer relation of $b$, an integer $M \geq 2^{\frac{n+1}{2}}\|m'\|_2  $ and set
$b_{-1}=(b_2,\ldots,b_n) \in (2^{-N}\mathbb{Z})^{n-1}$. Then running the LLL basis reduction algorithm on the lattice generated by the columns of the following $n\times n$ integer-valued matrix,\begin{equation}
B=\left(\begin{array}{@{}c|c@{}}
  \begin{matrix}
  M 2^N b_1
  \end{matrix}
  & M 2^N b_{-1} \\
\hline
  0_{(n-1)\times 1} &

  I_{(n-1)\times (n-1)}

\end{array}\right)
  \end{equation}outputs $t \in \mathbb{Z}^n$ which 
  \begin{itemize} 
  \item[(1)] is an integer relation for $b$ with $\|t\|_2 \leq 2^{\frac{n+1}{2}}\|m'\|_2 \|b\|_{2}$ and,
  \item[(2)] terminates in time polynomial in $n,N,\log M $ and $\log (\|b\|_{\infty}).$
  \end{itemize}

\end{theorem}

\begin{proof} It is immediate that $B$ is integer-valued and that the determinant of $B$ is $M2^N \not =0,$ and therefore the columns of $B$ are linearly independent. Hence, from Theorem \ref{LLL_thm_original}, we have that the LLL algorithm outputs a vector $z=Bt$ with $t \in \mathbb{Z}^n$ such that it holds \begin{align}\label{LLL_output} \|z\|_2 \leq 2^{\frac{n}{2}}\min_{ x \in \mathbb{Z}^n \setminus \{0\}} \|Bx\|_2.
\end{align} Moreover, it terminates in time polynomial in $n$ and $\log(M2^N \|b_{\infty}\|_{\infty})$ and therefore  in time polynomial in $n,N, \log M$ and $\log( \|b\|_{\infty}).$

Since $m'$ is an integer relation for $b$ it holds, $Bm'=(0,m'_2,\ldots,m'_{n})^t$ and therefore \begin{align*}\min_{ x \in \mathbb{Z}^n \setminus \{0\}} \|Bx\|_2 \leq \|Bm'\|_2 \leq \|m'\|_2.
\end{align*}Hence, combining with \eqref{LLL_output} we conclude 
\begin{align}\label{LLL_output_2} \|z\|_2 \leq 2^{\frac{n}{2}}\|m'\|_2.
\end{align}or equivalently
\begin{align}\label{LLL_output_3} \sqrt{ (M\inner{2^Nb,t})^2+\|t_{-1}\|^2_2} \leq 2^{\frac{n}{2}}\|m'\|_2,
\end{align}where $t_{-1}:=(t_2,\ldots,t_n) \in \mathbb{Z}^{n-1}$.  

Now notice that since $ 2^N\inner{b,t} =\inner{2^Nb,t} \in \mathbb{Z}$ either $2^N \inner{ b,t} \not =0$ and the left hand side of \eqref{LLL_output_3} is at least $M$, or $2^N \inner{ b,t}=0.$ Since the former case is impossible given the right hand side of inequality described in \eqref{LLL_output_3} and that $M \geq 2^{\frac{n+1}{2}}\|m'\|_2>2^{\frac{n}{2}}\|m'\|_2$ we conclude that $2^N \inner{ b,t}=0$ or equivalently $\inner{ b,t}=0$. Therefore, $t$ is an integer relation for $b$.

To conclude the proof it suffices to show that $\|t\|_2 \leq 2^{\frac{n}{2}+1}\|m'\|_2 \|b\|_{2}.$ Now again from  \eqref{LLL_output_3} and the fact that $t$ is an integer relation for $b,$ we conclude that 
\begin{align}\label{LLL_output_4} \|t_{-1}\|_2 \leq 2^{\frac{n}{2}}\|m'\|_2.
\end{align}But since $\inner{b,t}=0$ and $b_1=1$ we have by Cauchy-Schwartz and \eqref{LLL_output_3}  \begin{align*}|t_1|=|\inner{t_{-1},b_{-1}}| \leq \|t_{-1}\|_2  \|b_{-1}\|_2 \leq 2^{\frac{n}{2}}\|m'\|_2 \|b\|_2. 
\end{align*}Hence,
\begin{align*}
    \|t\|_2 \leq \sqrt{2} \max\{ 2^{\frac{n}{2}}\|m'\|_2 \|b\|_2,2^{\frac{n}{2}}\|m'\|_2\} \leq 2^{\frac{n+1}{2}}\|m'\|_2 \|b\|_2,
\end{align*}since $\|b\|_2 \geq |b_1|=1.$
\end{proof}

\subsection{Towards proving Theorem \ref{thm:algo_main}: Auxiliary Lemmas}

We first repeat the algorithm we analyze here for convenience, see Algorithm \ref{alg:lll_app}. Next, we present here three crucial lemmas towards proving the Theorem \ref{thm:algo_main}. The proofs of them are deferred to later sections, for the convenience of the reader.

\begin{algorithm}[t]
\caption{LLL-based algorithm for learning the single cosine neuron (Restated)}
\label{alg:lll_app}
\KwIn{i.i.d.~noisy $\gamma$-single cosine neuron samples $\{(x_i,z_i)\}_{i=1}^{d+1}$.}
\KwOut{Unit vector $\hat{w} \in S^{d-1}$ such that $\min(\|\hat{w} - w\|,\|\hat{w}+w\|) = \exp(-\Omega((d \log d)^3))$.}
\vspace{1mm} \hrule \vspace{1mm}

\For{$i=1$ \KwTo $d+1$}{
    $z_i \leftarrow \sgn(z_i)\cdot \min(|z_i|, 1)$ \\
    $\tilde{z}_i = \arccos(z_i)/(2\pi) \mod 1$
    }

Construct a $d \times d$ matrix $X$ with columns $x_2,\ldots,x_{d+1}$, and let $N=d^3(\log d)^2$.\\
\If{$\det(X) = 0$}{\Return $\hat{w} = 0$ and output FAIL}
Compute $\lambda_1 = 1$ and $\lambda_i = \lambda_i(x_1,\ldots,x_{d+1})$ given by $(\lambda_2,\ldots,\lambda_{d+1})^\top = X^{-1} x_1$. \\
Set $M=2^{3d}$ and $\tilde{v}=\left((\lambda_2)_N,\ldots,(\lambda_{d+1})_N, (\lambda_{1}z_1)_N,\ldots, (\lambda_{d+1}z_{d+1})_N, 2^{-N}\right) \in \mathbb{R}^{2d+2}$\\
Output $(t_1,t_2,t)\in  \ZZ^{d+1} \times \ZZ^{d+1} \times \ZZ$ from running the LLL basis reduction algorithm on the lattice generated by the columns of the following $(2d+3)\times (2d+3)$ integer-valued matrix,\begin{equation*}
\left(\begin{array}{@{}c|c@{}}
  \begin{matrix}
  M 2^N(\lambda_1)_N 
  \end{matrix}
  & M 2^N \tilde{v} \\
\hline
  0_{(2d+2)\times 1} &
  I_{(2d+2)\times (2d+2)}
\end{array}\right)
  \end{equation*}\\ 
Compute $g = \mathrm{gcd}(t_2)$, by running Euclid's algorithm. \\
\If{$g=0 \vee (t_2/g) \notin \{-1,1\}^{d+1}$}{\Return $\hat{w} = 0$ and output FAIL}
$\hat{w} \leftarrow \mathrm{SolveLinearEquation}(w', X^\top w' = (t_2/g)z + (t_1/g))$ \\
\Return $\hat{w}/\| \hat{w}\|$ and output SUCCESS.
\end{algorithm}

The first lemma establishes that given a small, in $\ell_2$ norm, ``approximate" integer relation between real numbers, one can appropriately truncate each number to some sufficiently large number of bits, so that the truncated numbers satisfy a small in $\ell_2$-norm integer relation between them. This lemma is important for the appropriate application of the LLL algorithm, which needs to receive integer-valued input. Recall that for real number $x$ we denote by $(x)_N$ its truncation to its first $N$ bits after zero, i.e. $(x)_N:=2^{-N} \lfloor 2^N x \rfloor.$
\begin{lemma}\label{lem:trunc}
Suppose $n \leq C_0d$ for some constant $C_0>0$ and $s \in \mathbb{R}^n$ satisfies for some $m \in \mathbb{Z}^n$ that $|\inner{m,s}| = \exp(-\Omega( (d \log d)^3))$. Then for some sufficiently large constant $C>0$, if $N=\lceil d^3 (\log d)^2 \rceil$ there is an  $m' \in \mathbb{Z}^{n+1}$ which is equal with $m$ in the first $n$ coordinates, which satisfies that $\|m'\|_2 \leq C d^{\frac{1}{2}}\|m\|_2$ and is an integer relation for the numbers $(s_1)_N,\ldots,(s_n)_N, 2^{-N}.$ 
\end{lemma}The proof of Lemma \ref{lem:trunc} is in Section \ref{app_aux_LLL}.

The following lemma establishes multiple structural properties surrounding $d+1$ samples from the cosine neuron, of the form $(x_i,z_i),i=1,\ldots, d+1$ given by \eqref{cosine-dist}.
\begin{lemma}\label{lem:bounds}
 Suppose that  $\gamma \leq d^{Q}$ for some constant $Q>0$. For some hidden direction $w \in S^{d-1}$ we observe $d+1$ samples of the form $(x_i,z_i), i=1,\ldots,d+1$ where for each $i$, $x_i$ is a sample from the distribution $N(0, I_d)$, and \begin{align*} z_i=\cos(2 \pi  (\gamma \langle w,x_i \rangle)) +\xi_i,\end{align*} for some unknown and arbitrary $\xi_i \in \mathbb{R}$ satisfying $|\xi_i| \leq \exp(-  (d \log d)^3).$ Denote by $X \in \mathbb{R}^{d \times d}$ the random matrix with columns given by the $d$ vectors $x_2,\ldots,x_{d+1}$. With probability $1-\exp(-\Omega(d))$ the following properties hold.
\begin{itemize}
\item[(1)]  $\max_{i=1,\ldots,d+1} \|x_i\|_2 \leq 10\sqrt{d}.$
\item[(2)] $\min_{i=1,\ldots,d+1} |\sin(2\pi \gamma \langle x_i,w \rangle)| \geq 2^{-d}.$
\item[(3)] For all $i=1,\ldots,d+1$ it holds $z_i \in [-1,1]$ and  $ z_i=\cos (2 \pi( \gamma \langle x_i, w\rangle+\xi'_i))$,      for some $\xi'_i \in \mathbb{R}$ with $|\xi'_i|=\exp(-\Omega( (d \log d)^3)).$ 
   \item[(4)] The matrix $X$ is invertible. Furthermore,  
        $\|X^{-1}x_1\|_{\infty} =O( 2^{\frac{d}{2}}\sqrt{d}).$
    \item[(5)] 
    $0<|\mathrm{det}(X)|=O(\exp(d\log d)).$
\end{itemize}
\end{lemma}
The proof of Lemma \ref{lem:bounds} is in Section \ref{app_aux_LLL}.

As explained in the description of our main results in Section \ref{LLL_main}, a step of crucial importance is to show that all ``near-minimal" integer relations, such as \eqref{IR}, for the (truncated versions of) $\lambda_i, \lambda_i \tilde{z}_i, i=1,\ldots,d+1$ are "informative". In what follows, we show that the integer relation with appropriately ``small" norm are indeed informative in terms of recovering the unknown $\epsilon_i,K_i$ of \eqref{IR} and therefore the hidden vector $w.$ The following technical lemma is of instrumental importance for the analysis of the algorithm. 

\begin{lemma}\label{lem:min_norm} Suppose that $\gamma \leq d^Q$ for some constant $Q>0$, and $N=\lceil d^3(\log d)^2 \rceil$. Let $\xi' \in \mathbb{R}^{d+1}$ be such that $\|\xi'\|_{\infty} \leq \exp(-(d \log d)^3)$ and $w \in S^{d-1}$. Suppose that for all $(x_i)_{i=1,\ldots,d+1}$ are i.i.d. $N(0,I_d)$ and that for each $i=1,\ldots,d+1$ for some  $\tilde{z}_i \in [-1/2,1/2]$ there exist $\epsilon_i \in \{-1,1\}, K_i \in \mathbb{Z}$ with $|K_i| \leq d^Q$ such that
\begin{align}\label{eq:lemma_LLL}
    \gamma \langle w,x_i \rangle=\epsilon_i \tilde{z}_i +K_i-\xi'_i.
\end{align}Define also $X \in \mathbb{R}^{d \times d}$ the matrix with columns the $x_2,\ldots,x_{d+1}$ and set $\lambda_1=1$ and $(\lambda_2,\ldots,\lambda_{d+1})^t=X^{-1}x_1.$
Then with probability $1-\exp(-\Omega(d))$, any integer relation $t \in \mathbb{Z}^{2d+3}$ between the numbers $(\lambda_1)_N,\ldots,(\lambda_{d+1})_N, (\lambda_1 \tilde{z}_1)_N, \ldots, (\lambda_{d+1}\tilde{z}_{d+1})_N, 2^{-N}$ with $\|t\|_2 \leq 2^{2d}$ satisfies in the first $2d+2$ coordinates it is equal to a non-zero integer multiple of $(K_1,\ldots,K_{d+1}, \epsilon_1,\ldots,\epsilon_{d+1})$.

\end{lemma}The proof of Lemma \ref{lem:min_norm} is in Section \ref{sec:proof_lem_LLL}.

\subsection{Proof of Theorem \ref{thm:algo_main}}
\label{sec:appD3}

We now proceed with the proof of the Theorem \ref{thm:algo_main} using the lemmas from the previous sections.

\begin{proof}
We analyze the algorithm by first analyze it's correctness step by step as it proceeds and then conclude with the polynomial-in-$d$ bound on its termination time.

We start with using part 3 of Lemma \ref{lem:bounds} which gives us that $z_i \in [-1,1]$ with probability $1-\exp(-\Omega(d))$ for all $i=1,2,\ldots,d+1.$ Therefore the $z_i$'s remain invariant under the operation $z_i \leftarrow \sgn(z_i)\min(|z_i|,1),$ with probability $1-\exp(-\Omega(d))$. Furthermore, using again the part 3 of Lemma \ref{lem:bounds} the $\tilde{z}_i$'s computed in the second step satisfy \begin{align*}\cos(2\pi \tilde{z}_i)=\cos(2\pi (\gamma \langle w,x_i \rangle +\xi'_i))\end{align*} for some $\xi'_i \in \mathbb{R}$ with $|\xi'_i| \leq \exp(-\Omega( (d \log d)^3)).$ Using the $2\pi$- periodicity of the cosine as well as that it is an even function we conclude that for all for $i=1,\ldots,d+1$ there exists $\epsilon_i \in \{-1,1\}, K_i \in \mathbb{Z}$  for which it holds for every $i=1,\ldots,d+1$
\begin{align}\label{eq:main2}
    \gamma \langle w,x_i \rangle=\epsilon_i \tilde{z}_i +K_i-\xi'_i.
\end{align}Notice that if we knew the exact values of $\epsilon_i,K_i$, since we already know $x_i,\tilde{z}_i$ the problem would reduce to inverting a (noisy) linear system of $d+1$ equations and $d$ unknowns. The rest of the algorithm uses an appropriate application of the LLL to learn the values of $\epsilon_i,K_i$ and solve the (noisy) linear system.

Now, notice that using the part 5 of Lemma \ref{lem:bounds} with probability $1-\exp(-\Omega(d))$ the matrix $X$ is invertible and the algorithm is  not going to terminate in the second step. 

In the following step, the $\lambda_i, i=1,2,\ldots,d+1$ are given by $\lambda_1=1$ and the unique $\lambda_i=\lambda_i(x_1,\ldots,x_{d+1}) \in \mathbb{R}, i=2,\ldots,d+1$  satisfying \begin{align*}\sum_{i=1}^{d+1} \lambda_i x_i=x_1+X (\lambda_2,\ldots,\lambda_{d+1})^\top=0.\end{align*} Hence, we conclude that for the unknown direction $w$ it holds
\begin{align*}\sum_{i=1}^{d+1} \lambda_i \gamma \langle w, x_i \rangle =\gamma \langle w, \sum_{i=1}^{d+1} \lambda_i  x_i \rangle = 0.\end{align*}  Using now \eqref{eq:main2} and rearranging the noise terms we conclude
\begin{align}\label{eq:dependency2}
\sum_{i=1}^{d+1}\lambda_i \tilde{z}_i \epsilon_i +\sum_{i=1}^{d+1}\lambda_i K_i =\sum_{i=1}^{d+1}\lambda_i \xi'_i.
\end{align} Now using the fourth part of Lemma \ref{lem:bounds} and the upper bound on $\|\xi'\|_{\infty}$ we have with probability $1-\exp(-\Omega(d))$ that \begin{align*}\left|\sum_{i=1}^{d+1}\lambda_i \xi'_i\right|=O(d\|\lambda\|_{\infty}\|\xi'\|_{\infty}) =O(d2^{\frac{d}{2}}\sqrt{d}\exp(-\Omega( (d \log d)^3))) =\exp(-\Omega( (d \log d)^3)).\end{align*}

Hence, using \eqref{eq:dependency2} we conclude that with probability $1-\exp(-\Omega(d))$ it holds
\begin{align}\label{eq:dependency3}
\left|\sum_{i=1}^{d+1}\lambda_i \bar{z}_i  \epsilon_i +\sum_{i=1}^{d+1}\lambda_i K_i \right| =\exp(-\Omega((d \log d)^3)).
\end{align}

Define $s \in \mathbb{R}^{2d+2}$  given by $s_i=\lambda_i  , i=1,\ldots,d+1$ and $s_i=\lambda_{i-d-1} \tilde{z}_{i-d-1}, i=d+2,\ldots,2d+2$. Define also  $m \in \mathbb{Z}^{2d+2}$ given by $m_i=K_i , i=1,\ldots,d+1$ and $m_i=\epsilon_{i-d-1} , i=d+1,\ldots,2d+2$. For these vectors, given the above, it holds  with probability $1-\exp(-\Omega(d))$ that $|\inner{s,m}|=\exp(-\Omega( (d\log d)^3))$ based on \eqref{eq:dependency3}. Now notice that 
\begin{align}
\label{K_i}
\max_{i=1,\ldots, d+1} |K_i| =O(\gamma \sqrt{d})
\end{align} 
with probability $1-\exp(-\Omega(d))$. Indeed, from the definition of $K_i$ we have for large enough values of $d$ that $|K_i| \leq \gamma |\langle w, x_i \rangle|+1 + |\xi_i|\leq  \gamma \|x_i \|_2+2.$ Recall that using part 1  of Lemma \ref{lem:bounds} for all $i=1,\ldots, d+1 $ it holds $\|x_i\|_2=O(\sqrt{d})$ with probability $1-\exp(-\Omega(d)).$ Hence, for all $i$, $|K_i| =O(\gamma \sqrt{d}),$ with probability $1-\exp(-\Omega(d)).$ Therefore, since $|\epsilon_i|=1$ for all $i=1,\ldots,d+1$ it also holds with probability $1-\exp(-\Omega(d))$ that $\|m\|_2 =O(d \|K\|_{\infty})=O(\gamma d^{\frac{3}{2}}).$

We now employ Lemma \ref{lem:trunc} for our choice of $s$ and $m$ to conclude that for the $N$ chosen by the algorithm there exists an integer $m'_{2d+3}$ so that  $m'=(m,m'_{2d+3}) \in \mathbb{Z}^{2d+3}$ is an integer relation for $ (\lambda_1 )_N,\ldots,(\lambda_{d+1})_N, (\lambda_1 z_1)_N,\ldots,(\lambda_{d+1} z_{d+1})_N, 2^{-N}$ with $\|m'\|_2=O(d^{2}\gamma).$ 

Now we set $b \in (2^{-N}\mathbb{Z})^{2d+3}$ given by $b_i=(\lambda_i)_N$ for $i=1,\ldots,d+1$, $b_i=(\lambda_{i-d-1} \tilde{z}_{i-d-1})_N$ for $i=d+2,\ldots,2d+2$, and $b_{2d+3}=2^{-N}.$ Notice that $b_1=(1)_N=1$ and furthermore that the $\tilde{v}$ defined by the algorithm satisfies $\tilde{v}=(b_2,\ldots,b_{2d+3}).$ On top of this, we have that the $m'$ defined in previous paragraph is an integer relation for $b$ with $\|m'\|_2=O(d^{2}\gamma).$ Since $\gamma$ is polynomial in $d$ we have that $2^{\frac{2d+3+1}{2}}\|m'\|_2 \leq 2^{3d}$ for large values of $d.$ Hence, to analyze the LLL step of our algorithm we use Theorem \ref{thm:LLL_app} for $n=2d+3$, to conclude that the output of the LLL basis reduction step is a $t=(t_1,t_2,t') \in \mathbb{Z}^{d+1} \times \mathbb{Z}^{d+1} \times \mathbb{Z}$ which is an integer relation for $b$ and it satisfies that\begin{align*} \|t\|_2 \leq 2^{d+2}\|m'\|_2 \|b\|_2,
\end{align*}with probability $1-\exp(-\Omega(d))$.

Now we use part 4 of Lemma \ref{lem:bounds} to conclude that $\|\lambda\|_{2} \leq d\|\lambda\|_{\infty} =O(2^{\frac{d}{2}} d^{\frac{3}{2}}),$  with probability $1-\exp(-\Omega(d))$. Since for any real number $x$ it holds $|(x)_N| \leq |x|+1$ and $\tilde{z}_i \in [-1/2,1/2]$ for all $i=1,2,\ldots,d+1$ we conclude that $\|b\|_2=O(\|\lambda\|_2)=O(2^{\frac{d}{2}} d^{\frac{3}{2}}),$  with probability $1-\exp(-\Omega(d))$. Furthermore, since $\|m'\|=O(d^2 \gamma)$ we conclude that since $\gamma$ is polynomial in $d$, for large values of $d$ it holds, 
\begin{align}
\|t\|_2 = O( 2^{\frac{3d}{2}}) \leq 2^{2d}\;, \label{t_bound}
\end{align} 
with probability $1-\exp(-\Omega(d))$. 

We now use the above and \eqref{K_i} to crucially apply Lemma \ref{lem:min_norm} and conclude that for some non-zero integer multiple $c$ it necessarily holds $(t_1)_i=cK_i$ and $(t_2)_i=c\epsilon_i,$ with probability $1-\exp(-\Omega(d))$. Note that the assumptions of the Lemma can be checked to be satisfied in straightforward manner. Now, the greatest common divisor between the elements of $t_2$ equals either $c$ or $-c,$ since the elements of $t_2$  are just $c$-multiples of $\epsilon_i$ which themselves are taking values either $-1$ or $1.$ Hence the step of the algorithm using Euclid's algorithm outputs $g$ such that $g=\epsilon c$ for some $\epsilon \in \{-1,1\}$. In particular, $t_2/g=\epsilon (\epsilon_1,\ldots,\epsilon_{d+1}) \not =0$ implying that the algorithm does not enter the if-condition branch on the next step.

Finally, since $c=\epsilon g$ it also holds $t_1/g=\epsilon (K_1,\ldots,K_{d+1})$ and therefore the last step of the algorithm is solving the linear equations for $i=2,\ldots,d+1$ given by
\begin{align*}
    \inner{x_i,\hat{w}}=\epsilon \left( \epsilon_i \tilde{z}_i+\epsilon K_i \right)=\epsilon \gamma \inner{x_i,w}+\epsilon \xi'_i,
\end{align*}
where we have used \eqref{eq:main2}. Hence if $\xi'=(\xi'_2,\ldots,\xi'_{d+1})^t$ we have  
\begin{align*}\hat{w}=\epsilon \gamma w+\epsilon X^{-1}\xi~.
\end{align*}
Hence, \begin{align*}
\|\hat{w}-\epsilon \gamma w\|_2 \leq \|X^{-1} \xi\|_2.
\end{align*}
Now, using standard results on the extreme singular values of $X$, such as \cite[Equation (3.2)]{rudelson_ICM}, we have that $\sigma_{\max}(X^{-1})=1/\sigma_{\min}(X) \leq 2^{d},$ with probability $1-\exp(-\Omega(d)).$  Hence, with probability $1-\exp(-\Omega(d))$ it holds
\begin{align*}
\|\hat{w}-\epsilon \gamma w\|_2 \leq O\left(2^\frac{d}{2}\| \xi\|_2\right).
\end{align*}
Now since almost surely $\|\xi\|_2 \leq d \beta$ and $\beta \leq \exp(-(d \log d)^3)$ we have $2^\frac{d}{2}\| \xi\|_2=O(\beta)=\exp(-\Omega((d \log d)^3))$ and therefore,  with probability $1-\exp(-\Omega(d))$ it holds
\begin{align}
\label{eq:finalstepbound}
\|\hat{w}-\epsilon \gamma w\|_2 \leq O\left(\beta \right)=\exp(-\Omega( (d \log d)^3)).
\end{align}
Finally, since $| \|x \|_2 - \|x'\|_2 | \leq \|x - x'\|_2$ 
we also have $| \| \hat{w} \|_2 - \gamma | \leq O(\beta) = \exp(-\Omega( (d \log d)^3))$ and therefore 
\begin{eqnarray*}
\left \| \frac{\hat{w}}{\| \hat{w}\|} - \epsilon w \right \|_2 &=& \gamma^{-1} \left \| \frac{\gamma}{\| \hat{w}\|_2}\hat{w} - \epsilon w \gamma \right \|_2 \leq \gamma^{-1} \left( \| \hat{w} - \epsilon \gamma w \|_2 + \frac{\| \hat{w} - \gamma\|_2 }{ \gamma - | \gamma - \|\hat{w} \|_2| }  \right) \\ 
&\leq& \gamma^{-1} \left( \| \hat{w} - \epsilon \gamma w \|_2 + O(\beta) \right) \\
&\leq& O\left(\frac{\beta}{\gamma}\right) =\exp(-\Omega( (d \log d)^3))~,
\end{eqnarray*}
since $\gamma=\omega( \beta)$. Since $\epsilon \in \{-1,1\}$ the proof of correctness is complete.

For the termination time, it suffices to establish that the step using the LLL basis reduction algorithm and the step using the Euclid's algorithm can be performed in polynomial-in-$d$ time. For the LLL step we use Theorem \ref{thm:LLL_app} to conclude that it runs in polynomial-time in $d, N, \log M$ and $\log \|\lambda\|_{\infty}.$ Now clearly $N,  \log M$ are polynomial in $d.$ Furthermore, by part 4 of Lemma \ref{lem:bounds} also $\log \|\lambda\|_{\infty}$ is polynomial in $d$ with probability $1-\exp(-\Omega(d))$. The Euclid's algorithm takes time which is polynomial in $d$ and in $\log \|t_2\|_{\infty}.$ But we have established in \eqref{t_bound} that $\|t_2\|_2 \leq \|t\|_2 \leq 2^{2d},$ with probability $1-\exp(-\Omega(d))$ and therefore the Euclid's algorithm step also indeed requires time which is polynomial-in-$d$.
\end{proof}

\subsection{Proof of Lemma \ref{lem:min_norm}} \label{sec:proof_lem_LLL}

We focus this section on proving the crucial Lemma \ref{lem:min_norm}. As mentioned above, the proof of the lemma is quite involved, and, potentially interestingly, it requires the use of anticoncentration properties of the coefficients $\lambda_i$ which are rational function of the coordinates of $x_i.$ In particular, the following result is a crucial component of establishing Lemma \ref{lem:min_norm}.

\begin{lemma}\label{lem:poly}
Suppose $w \in S^{d-1}$ is an arbitrary vector on the unit sphere and $\gamma \geq 1.$ For two sequences of integer numbers $C=(C_i)_{ i=1,2,\ldots,d+1},C'=(C'_i)_{ i=1,2,\ldots,d+1}$ we define the polynomial $P_{C,C'}(x_1,\ldots,x_{d+1})$ in $d(d+1)$ variables which equals
\begin{align}\label{poly}
&\det(x_2,\ldots,x_{d+1})  \left( \langle \gamma w, x_1 \rangle C_1 +(C')_1 \right)\\
&+\sum_{i=2}^{d+1} \det(x_2,\ldots,x_{i-1},-x_1,x_{i+1},\ldots,x_{d+1})  \left( \langle \gamma w, x_i \rangle C_i +(C')_i \right), \nonumber
\end{align}where each $x_1,\ldots,x_{d+1}$ is assumed to have a $d$-dimensional vector form.

We now draw $x_i$'s in an i.i.d. fashion from the standard Gaussian measure on $d$ dimensions. For any two sequences $C,C'$ it holds 
\begin{align*}
\mathrm{Var}(P_{C,C'}(x_1,\ldots,x_{d+1}))=(d-1)! \gamma^2 \sum_{1 \leq i < j \leq d+1} (C_i-C_j)^2+d! \sum_{i=1}^{d+1}(C')_i^2.
\end{align*}

Furthermore, for some universal constant $B>0$ the following holds. If $C_i,C'_i$ are such that either the $C_i$'s are not all equal to each other or the $C'_i$'s are not all equal to zero, then for any $\epsilon>0,$ 
 \begin{align} \label{eq:anti} \mathbb{P}(|P_{C,C'}(x_1,\ldots,x_{d+1})| \leq \epsilon  ) \leq B (d+1) \epsilon^{\frac{1}{d+1}}.
 \end{align}
\end{lemma}
\begin{proof}
The second part follows from the first one combined with the fact that under the assumptions on $C,C'$ in holds that for some $i=1,\ldots, d+1$ either $(C_i-C'_i)^2 \geq 1$ or $(C'_i)^2 \geq 1$. In particular, in both cases since $\gamma \geq 1,$ \begin{align*}
\mathrm{Var}(P_{C,C'}(x_1,\ldots,x_{d+1})) \geq (d-1)! \geq 1.\end{align*} Now we employ \cite[Theorem 1.4]{v012a011} (originally proved in \cite{Carbery2001DistributionalAL}) which implies that for some universal constant $B>0,$ since our polynomial is multilinear and has degree $d+1$ it holds for any $\epsilon>0$
\begin{align*}  
\mathbb{P}\left(|P_{C,C'}(x_1,\ldots,x_{d+1})| \leq \epsilon \sqrt{\mathrm{Var}(P_{C,C'}(x_1,\ldots,x_{d+1}))} \right) \leq B (d+1) \epsilon^{\frac{1}{d+1}}.
 \end{align*}Using our lower bound on the variance we conclude the result.

Now we proceed with the variance calculation. First we denote 
\begin{align*}
    \mu (x_{-1}):=\det(x_2,\ldots, x_{d+1})\;,
\end{align*}
and for each $i>2$ \begin{align*}\mu (x_{-i}):=\det(x_2,\ldots,x_{i-1},-x_1,x_{i+1},\ldots, x_{d+1}).\end{align*}As all coordinates of the $x_i$'s are i.i.d. standard Gaussian, for each $i=1,\ldots,d+1$ the random variable $\mu (x_{-i})$ has mean zero and variance $d!$. Furthermore, let us denote $\ell(x_i):=\langle \gamma w, x_i \rangle$, which is a random variable with mean zero and variance $\gamma^2$. In particular $\mu (x_{-i}) \ell(x_i)$ has also mean zero as $\mu (x_{-i})$ is independent with $x_i$. Now notice that under this notation,
\begin{align*}P_{C,C'}(x_1,\ldots,x_{d+1})=\sum_{i=1}^d C_i \mu(x_{-i}) \ell(x_{i}) +\sum_{i=1}^d C'_i \mu(x_{-i}) .\end{align*}Hence, we conclude \begin{align*}\mathbb{E}[P_{C,C'}(x_1,\ldots,x_{d+1})]=0.\end{align*}

Now we calculate the second moment of the polynomial. We have
\begin{align*}
\mathbb{E}[P^2_{C,C'}(x_1,\ldots,x_{d+1})]=\sum_{i=1}^{d+1} C^2_i d! \gamma^2 +\sum_{1 \leq i \not = j \leq d} C_i C_j \mathbb{E}[ \mu(x_{-i}) \ell(x_{i})\mu(x_{-j}) \ell(x_{j})]+\sum_{i=1}^{d+1} {C'}_i^2 d! \;.
\end{align*}

Now for all $i \not = j$, 
\begin{align*}
    &\mathbb{E}[ \mu(x_{-i}) \ell(x_{i})\mu(x_{-j}) \ell(x_{j})]\\
    &=\mathbb{E}[ \det(\ldots,x_{i-1},-x_1,x_{i+1},\ldots) \det(\ldots,x_{j-1},-x_1,x_{j+1},\ldots) \langle \gamma w, x_i \rangle\langle \gamma w, x_j \rangle]\\
    &=\sum_{p,q=1}^d \gamma^2 w_p w_q \mathbb{E}[  \det(\ldots,x_{i-1},-x_1,x_{i+1},\ldots) \det(\ldots,x_{j-1},-x_1,x_{j+1},\ldots) (x_i)_p (x_j)_q]
\end{align*}
Now observe that the monomials of the product \begin{align*}\det(\ldots,x_{i-1},-x_1,x_{i+1},\ldots) \det(\ldots,x_{j-1},-x_1,x_{j+1},\ldots)(x_i)_p (x_j)_q\end{align*} have the property that each coordinate of the various $x_i's$ appears at most twice; in other words the degree per variable is at most 2. Hence, the monomials that could potentially have not zero mean with respect to the standard Gaussian measure are the ones where all coordinates of every $x_i,i=1,\ldots,d+1$ appear exactly twice or none at all, in which case the monomial has mean equal to the coefficient of the monomial. By expansion of the determinants, we have that the studied product of polynomials equals to the sum over all $\sigma, \tau$ permutations on $d$ variables of the terms
\begin{align*}
     (-1)^{\mathrm{sgn}(\sigma \tau^{-1})}(\ldots x_{i-1,\sigma(i-1)}(-x_1)_{\sigma(i)}x_{i+1,\sigma{i+1}}\ldots) (\ldots x_{j-1,\tau(j-1)}(-x_1)_{\tau(j)}x_{j+1,\tau(j+1)}\ldots)  (x_i)_p (x_j)_q.
\end{align*}Hence, a straightforward inspection allows us to conclude that for every coordinate to appear exactly twice, we need the corresponding permutations $\sigma, \tau$ to satisfy $\tau(i)=p,\sigma(j)=q$ (from the coordinates $(x_i)_p, (x_j)_q$),  $\sigma(i)=\tau(j)$ (from the coordinate of $x_1$) and finally $\sigma(x)=\tau(x)$ for all $x \in [d] \setminus \{i,j\}$ (the rest coordinates). Furthermore, the value of the mean of this monomial would then be given simply by $(-1)^{\mathrm{sgn}(\sigma \tau^{-1})}.$

Now we investigate more which permutations $\sigma, \tau$ can satisfy the above conditions. The last two conditions imply in straightforward manner that $\tau^{-1}\sigma$ is the transposition  $(i,j).$ Hence, $\tau^{-1}\sigma(j)=i$. But we have $\sigma(j)=q$ and therefore $i=\tau^{-1}\sigma(j)=\tau^{-1}(q)$ which gives $\tau(i)=q.$ We have though as our condition that $\tau(i)=p$ which implies that for such a pair of permutations $\sigma, \tau$ to exist it must hold $p=q.$ Furthermore, for any $\sigma$ with $\sigma(j)=p$ there exist a unique $\tau$ satisfying the above given by $\tau=\sigma \circ (i,j)$, where $\circ$ corresponds to the multiplication in the symmetric group $S_d.$ Hence, if $p \not = q$ no such pair of permutations exist and the mean of the product is zero. If $p=q$ there are exactly $(d-1)!$ such pairs  (all permutations $\sigma$ sending $j$ to $p$ and $\tau$ given uniquely given $\sigma$) which correspond to $(d-1)!$ monomials with mean $(-1)^{\mathrm{sgn}(\sigma)+\mathrm{sgn}( \tau)}=(-1)^{\mathrm{sgn}( \sigma^{-1} \tau)}=-1,$ where we used that the sign of a transposition is $-1.$ Combining the above we conclude that \begin{align*}\mathbb{E}[\det(\ldots,x_{i-1},-x_1,x_{i+1},\ldots) \det(\ldots,x_{j-1},-x_1,x_{j+1},\ldots)(x_i)_p (x_j)_q]=-(d-1)! 1(p=q).\end{align*} Hence, since $\|w\|_2=1,$
\begin{align*}
    &\mathbb{E}[ \mu(x_{-i}) \ell(x_{i})\mu(x_{-j}) \ell(x_{j})]=\sum_{p=1}^d -\gamma^2 w_p^2=-\gamma^2.
    \end{align*}Therefore,
    \begin{align*}
        \mathbb{E}[P^2_{C,C'}(x_1,\ldots,x_{d+1})]&=\sum_{i=1}^{d+1} C^2_i d! \gamma^2 - (d-1)! \gamma^2 \sum_{1 \leq i \not = j \leq d+1} C_i C_j +\sum_{i=1}^{d+1} {C'}_i^2 d! \\
        &=(d-1)!\gamma^2 \sum_{1 \leq i < j \leq d+1} (C_i-C_j)^2
        +d!\sum_{i=1}^{d+1}(C')_i^2.
    \end{align*}The proof is complete.
\end{proof}

We now proceed with the proof of Lemma \ref{lem:min_norm}.

\begin{proof}[Proof of Lemma \ref{lem:min_norm}]
Let $t_1,t_2 \in \mathbb{Z}^{d+1}, t'\in \mathbb{Z}$ with $\|(t_1,t_2,t')\|_2 \leq 2^{2d}$ which is an integer relation;
\begin{align*}
\sum_{i=1}^{d+1}(\lambda_i)_N (t_1)_i+\sum_{i=1}^{d+1}(\lambda_i \tilde{z}_i)_N (t_2)_i  +t'2^{-N}=0.\end{align*}First note that it cannot be the case that $t_1=t_2=0$ as from the integer relation it should be also that $t'=0$ and therefore $t=0$ but an integer relation needs to be non-zero. Hence, from now on we restrict ourselves only to the case where $t_1,t_2$ are not both zero. Now, as clearly $|t'| \leq 2^{2d}$ it also holds
\begin{align*} \left|\sum_{i=1}^{d+1}(\lambda_i)_N (t_1)_i+ \sum_{i=1}^{d+1}(\lambda_i \tilde{z}_i)_N (t_2)_i  \right| \leq 2^{2d}2^{-N}.\end{align*}

Consider $\mathcal{T}$ the set of all pairs $t=(t_1,t_2) \in (\mathbb{Z}^{d+1} \times \mathbb{Z}^{d+1})\setminus \{0\}$ for which there does not exist a $c \in \mathbb{Z} \setminus \{0\}$ such that for $i=1,\ldots,d+1$ $(t_1)_i=c K_i$ and $(t_2)_i=c \epsilon_i.$

To prove our result it suffices therefore to prove that
\begin{align*}
\mathbb{P}\left(\bigcup_{t \in \mathcal{T}, \|t\|_2 \leq 2^{2d}}\left\{\left|\sum_{i=1}^{d+1}(\lambda_i)_N (t_1)_i+ \sum_{i=1}^{d+1}(\lambda_i \tilde{z}_i)_N (t_2)_i   \right|\leq 2^{2d}/2^N\right\}\right) \leq \exp(-\Omega(d))
\end{align*}for which, since for any $x$ it holds $|x-(x)_N| \leq 2^{-N}$ and $\|(t_1,t_2)\|_1 \leq \sqrt{2(d+1)} \|(t_1,t_2)\|_2 \leq 2^{3d} $ for large values of $d,$ it suffices to prove that for large enough values of $d,$\begin{align*}
    \mathbb{P}\left(\bigcup_{t \in \mathcal{T}, \|t\|_2 \leq 2^{2d}}\left\{\left|\sum_{i=1}^{d+1}\lambda_i (t_1)_i+ \sum_{i=1}^{d+1}\lambda_i \tilde{z}_i (t_2)_i  \right| \leq 2^{4d}/2^N\right\}\right) \leq \exp(-\Omega(d)).
\end{align*}

Notice that by using the equations \eqref{eq:lemma_LLL} it holds
\begin{align*}
    &\sum_{i=1}^{d+1}\lambda_i (t_1)_i+\sum_{i=1}^{d+1}\lambda_i \tilde{z}_i (t_2)_i \\
    &=\sum_{i=1}^{d+1}\lambda_i (t_1)_i+\sum_{i=1}^{d+1}\lambda_i ( \epsilon_i \gamma \langle w,x_i \rangle-\epsilon_i K_i+\epsilon_i \xi'_i)(t_2)_i \\
    &=\sum_{i=1}^{d+1}\lambda_i \left(\epsilon_i \langle \gamma w, x_i \rangle (t_2)_i -\epsilon_i K_i (t_2)_i+\epsilon_i \xi_i (t_2)_i + (t_1)_i \right)\\
    &=\sum_{i=1}^{d+1}\lambda_i \left( \langle \gamma w, x_i \rangle C_i +C'_i \right)+\sum_{i=1}^d \lambda_i \xi'_i C_i,
\end{align*}for the integers $ C_i=\epsilon_i (t_2)_i$ and $C'_i=-\epsilon_i K_i (t_2)_i + (t_1)_i .$ Since $t \in \mathcal{T}$ some elementary algebra considerations imply that  either not all $(C_i)_{i=1,\ldots,d+1}$ are equal to each other or one of the $(C'_i)_{i=1,2,\ldots,d+1}$ is not equal to zero. Let us call this region of permissible pairs $(C,C')$ as $\mathcal{C}.$ Furthermore, given that all $t$ satisfy $\|t\|_2\leq 2^{2d},$ and that for all $K_i$ satisfy $|K_i| \leq d^Q$ it holds that any $(C,C')$ defined through the above equations with respect to $t_1,t_2,\epsilon_i,K_i$ satisfies the crude bound that \begin{align*}
\|(C,C')\|^2_2 \leq \|t_2\|^2_2 +2( d^{2Q} \|t_2\|^2_2+\|t_1\|^2_2) \leq 2^{6d}.
\end{align*}

Hence, using this refined notation it suffices to show

\begin{align*}
\mathbb{P}\left(\bigcup_{(C,C') \in \mathcal{C}, \|(C,C')\|_2 \leq 2^{3d}}\left\{\left|\sum_{i=1}^{d+1}\lambda_i \left( \langle \gamma w, x_i \rangle C_i +C'_i \right)+\sum_{i=1}^d \lambda_i \xi_i C_i\right| \leq 2^{4d}/2^N\right\}\right) \leq \exp(-\Omega(d)).
\end{align*}

Now notice that from our exponential-in-$d$ norm upper bound assumptions on $C$, the part 4 of Lemma \ref{lem:bounds}, and since $N=o( (d\log d)^3)$, the following holds with probability $1-\exp(-\Omega(d))$
\begin{align*}\sum_{i=1}^d |\lambda_i \xi_i C_i| =O(2^{4d} \|\xi\|_{\infty})=O(\exp(-(d\log d)^3))=O(2^{-N}).
\end{align*}
Hence it suffices to show that for large enough values of $d,$ 
\begin{align*}
\mathbb{P}\left(\bigcup_{(C,C') \in \mathcal{C}, \|(C,C')\|_2 \leq 2^{3d}}\left\{\left|\sum_{i=1}^{d+1}\lambda_i \left( \langle \gamma w, x_i \rangle C_i +C'_i \right)\right| \leq 2^{5d}/2^N\right\}\right) \leq \exp(-\Omega(d)).
\end{align*}

Using the polynomial notation of Lemma \ref{lem:poly} and specifically notation \eqref{poly}, as well as the fact that by Cramer's rule $\lambda_i$ are rational functions of the coordinates of $x_i$ satisfying $\lambda_i \mathrm{det}(x_2,\ldots,x_{d+1})=\mathrm{det}(\ldots, x_{i-1},-x_1,x_{i+1},\ldots)$ it suffices to show
\begin{align*} 
\mathbb{P}\left(\bigcup_{(C,C') \in \mathcal{C}, \|(C,C')\|_2 \leq 2^{3d}}\{|P_{C,C'}(x_1,\ldots,x_{d+1})| \leq |\mathrm{det}(x_2,\ldots,x_{d+1})|2^{5d}/2^N\}\right) \leq \exp(-\Omega(d)).
\end{align*}

Using the fifth part of the Lemma \ref{lem:bounds} there exists some constant $D>0$ for which it suffices to show
\begin{align*}\mathbb{P}\left(\bigcup_{(C,C') \in \mathcal{C}, \|(C,C')\|_2 \leq 2^{3d}}\{|P_{C,C'}(x_1,\ldots,x_{d+1})| \leq  2^{Dd \log d}/2^N\}\right) \leq \exp(-\Omega(d)).
\end{align*}Now since $N=\Theta( d^3( \log d)^2)$ we have $N=\omega(d\log d).$ Hence, for sufficiently large $d$ it suffices to show
\begin{align*}\mathbb{P}\left(\bigcup_{(C,C') \in \mathcal{C}, \|(C,C')\|_2 \leq 2^{3d}}\{|P_{C,C'}(x_1,\ldots,x_{d+1})| \leq  2^{-\frac{N}{2}}\}\right) \leq \exp(-\Omega(d)).
\end{align*}

By a union bound, it suffices
\begin{align}\label{eq:sum}\sum_{(C,C') \in \mathcal{C}, \|(C,C')\|_2 \leq 2^{3d}}\mathbb{P}\left(|P_{C,C'}(x_1,\ldots,x_{d+1})| \leq 2^{-\frac{N}{2}}\right) \leq 2^{-\Omega(d)}.
\end{align}Now the integer points $(C,C')$ with $\ell_2$ norm at most $2^{3d}$ are at most $2^{3d^2+d}$ as they have at most $2^{3d+1}$ choices per coordinate.
Furthermore, using the anticoncentration inequality \eqref{eq:anti} of Lemma \ref{lem:poly}, we have for any $(C,C') \in \mathcal{C}$ that it holds for some universal constant $B>0,$
\begin{align*}\mathbb{P}\left(|P_{C,C'}(x_1,\ldots,x_{d+1})| \leq 2^{-\frac{N}{2}}\right) \leq B(d+1)2^{-\frac{N}{2(d+1)}}.\end{align*}
Combining the above with the left hand side of \eqref{eq:sum}, the right hand side is at most
\begin{align*}
B (d+1) 2^{3d^2+d} 2^{-\frac{N}{2(d+1)}} =\exp( O(d^2) -\Omega(N/d))=\exp(-\Omega(d)),
\end{align*}
where we used that $N/d=\Omega(d^2 \log d )$. This completes the proof.
\end{proof}

\section{Approximation with One-Hidden-Layer ReLU Networks}
\label{app:relu-approx}
The members of the cosine function class $\sF_\gamma = \{\cos(2\pi \gamma \inner{w,x}) \mid w \in S^{d-1}\}$ consist of a composition of the univariate $2\pi$-Lipschitz, 1-periodic function $\phi(z) = \cos(2\pi z)$, and an one-dimensional linear projection $z=\gamma \inner{w,x}$. Notice that since $x \sim N(0,I_d)$, $z$ lies within the interval $[-R,R]$, where $R = \gamma \sqrt{2\log(1/\delta)}$, with probability at least $1-\delta$ due to Mill's inequality (Lemma~\ref{lem:mill}). Hence, to achieve $\eps$-squared loss over the Gaussian input distribution, it suffices for the ReLU network to uniformly approximate the univariate function $\phi(z) = \cos(2\pi z)$ on some compact interval $[-R(\gamma,\eps), R(\gamma,\eps)]$, and output 0 for all $z \in \RR$ outside the compact interval. 

The uniform approximability of univariate Lipschitz functions by the family of one-hidden-layer ReLU networks on compact intervals is well-known. To establish our results, we will use the quantitative result from~\cite{eldan2016power}, which we reproduce here as Lemma~\ref{lem:relu-uniform-approx}. We present our ReLU approximation result for the cosine function class right after, in Theorem~\ref{thm:cosine-relu-approx}.

\begin{lemma}[{\cite[Lemma 19]{eldan2016power}}]
\label{lem:relu-uniform-approx}
Let $\sigma(z)=\max\{0,z\}$ be the ReLU activation function, and fix $L, \eta,R > 0$. Let $f:\RR\rightarrow\RR$ be an $L$-Lipschitz function which is constant outside an interval $[-R,R]$. There exist scalars $a,\{\alpha_i, \beta_i\}_{i=1}^{w}$, where $w\leq 3\frac{RL}{\eta}$, such that the function
\begin{align*}
    h(x) = a + \sum_{i=1}^w \alpha_i \sigma(x - \beta_i)
\end{align*}
is $L$-Lipschitz and satisfies
\begin{align*}
    \sup_{x \in \RR} \bigl|f(x) - h(x) \bigr| \leq \eta.
\end{align*}
Moreover, one has $|\alpha_i| \leq 2 L$.
\end{lemma}

\begin{theorem}
\label{thm:cosine-relu-approx}
Let $d \in \NN$, $\gamma \ge 1$, and $\eps \in (0,1)$ be a real number. Then, the cosine function class $\sF_\gamma = \{\cos(2\pi\gamma \inner{w,x}) \mid w \in S^{d-1}\}$ can be $\eps$-approximated (in the squared loss sense) over the Gaussian input distribution $x \sim N(0,I_d)$ by one-hidden-layer ReLU networks of width at most $O\left( \gamma\sqrt{\frac{\log(1/\eps)}{\eps}}\right)$.
\end{theorem}
\begin{proof}
Let $R = \lceil \gamma \sqrt{2\log(8/\eps)} \rceil + 1/2$, and $z = \gamma \inner{w,x}$. Then, by Mill's inequality (Lemma~\ref{lem:mill}) and the fact that $R>\gamma,$
\begin{align}\label{ref:mill}
    \PP(|z| \ge R) \le \sqrt{\frac{2}{\pi}} \exp\left(-\frac{R^2}{2\gamma^2}\right) \le \frac{\eps}{8}\;.
\end{align}

Let $f : \RR \rightarrow \RR$ be a function which is equal to $\cos(2\pi z)$ on $[-R,R]$ and 0 outside the compact interval. We claim that $f$ is still $2\pi$-Lipschitz. First, note that $\cos(2\pi R) = \cos(-2\pi R) = 0$. Moreover, $f$ is $2\pi$-Lipschitz within the interval $[-R, R]$ and 0-Lipschitz in the region $|z| > R$. Hence, it suffices to consider the case when one point $z$ falls inside $[-R,R]$ and another point $z'$ falls outside the interval. Without loss of generality, assume that $z \in [-R,R]$ and $z' > R$. The same argument applies for $z' < -R$. Then,
\begin{align*}
    |f(z')-f(z)| = |f(R)-f(z)| \le 2\pi|R-z| \le 2\pi|z'-z|\;.
\end{align*}
Now set $L=2\pi, \eta = \sqrt{\eps/2}, R=\lceil \gamma\sqrt{2\log(8/\eps)}\rceil + 1/2$ in the statement of Lemma~\ref{lem:relu-uniform-approx}, and approximate $f$ with a one-hidden-layer ReLU network $g(z)$ of width at most $O\left( \gamma\sqrt{\frac{\log(1/\eps)}{\eps}}\right)$. Then,
\begin{align*}
\EE_{x \sim N(0,I_d)}[(\cos(2\pi \gamma \inner{w,x})-g(\gamma \inner{w,x}))^2] &=
    \EE_{z \sim N(0,\gamma)}[(\cos(2\pi z)-g(z))^2] \\
    &= \frac{1}{\gamma\sqrt{2\pi}}\int (\cos(2\pi z)-g(z))^2 \exp(-z^2/(2\gamma^2))dz \\
    &= \frac{1}{\gamma\sqrt{2\pi}}\int_{|z| \le R} (\cos(2\pi z)-g(z))^2 \exp(-z^2/(2\gamma^2))dz \\
    &\qquad + \frac{1}{\gamma\sqrt{2\pi}}\int_{|z| > R} (\cos(2\pi z)-g(z))^2 \exp(-z^2/(2\gamma^2))dz \\
    &\le \eta^2 + \frac{4}{\gamma\sqrt{2\pi}}\int_{|z| > R} \exp(-z^2/(2\gamma^2))dz \\
    &\le \eta^2 + 4(\eps/8) \\
    &< \eps\;,
\end{align*}
where the first inequality follows from the fact that the squared loss is bounded by $4$ for all $z \notin [-R,R]$ since $\cos(2\pi z) \in [-1,1]$ and $g(z) \in [-\eta,\eta] \subset [-1,1]$ and the second inequality uses \eqref{ref:mill}.
This completes the proof.
\end{proof}

\section{Covering Algorithm for the Unit Sphere}
\label{app:covering-algo-random}
The (very simple) randomized exponential-time algorithm for constructing an $\eps$-cover of the $d$-dimensional unit sphere $S^{d-1}$ is presented in Algorithm~\ref{alg:covering-algo}. We prove the algorithm's correctness in the following claim, which is essentially an appropriate application of the coupon collector problem.

\begin{claim}
Let $d \in \NN$ be a number, let $\eps \in (0,1)$ be a real number, and let $N = \lceil (1+4/\eps)^d \rceil $. Then, $\lceil 2 N\log N \rceil $ vectors sampled from $S^{d-1}$ uniformly at random forms an $\eps$-cover of $S^{d-1}$ with probability at least $1-\exp(-\Omega(d))$.
\end{claim}

\begin{proof}
By Lemma~\ref{lem:sphere-covering}, we know that there exists an $\eps/2$-cover of $S^{d-1}$ with size less than $N = \lceil (1+4/\eps)^d \rceil $. Let us assume for simplicity and without loss of generality, that it's size equals to $N$, by adding additional arbitrary points on the sphere to the cover if necessary. We denote this $\eps/2$-cover by $\sK$. Of course, $\sK \subseteq S^{d-1}$ by the definition of an $\eps$-cover in~\cite[Section 4.2]{vershynin2018high}.

Now, observe that any family $W$ of $M$ vectors on the sphere, say $W=\{w_1,\ldots,w_M\},$ with the property that for any $v \in \sK$ there exist $i \in [M]$ such that $\|v-w_i\|_2 \leq \epsilon/2$ is an $\epsilon$-cover of $S^{d-1}.$ Indeed, let $x \in S^{d-1}.$ Since $\sK$ is an $\epsilon/2$-cover, there exist $v \in \sK$ with $\|x-v\|_2 \leq \epsilon/2.$ Moreover, using the property of the family $W$, there exists some $i \in [M]$ for which $\|v-w_i\|_2 \leq \epsilon/2$. By triangle inequality we have $\|w_i-x\|_2 \leq \epsilon.$

Now, by definition of the $\epsilon/2$-cover it holds \begin{align*} \bigcup_{v \in \sK} \left(B(v,\epsilon/2) \cap S^{d-1}\right)=S^{d-1},\end{align*}where by $B(x,r)$ we denote the Euclidean ball in $\mathbb{R}^d$ with center $x \in \mathbb{R}^d$ and radius $r.$ Hence, denoting by $\mu$ the uniform probability measure on the sphere, by a simple union bound we conclude that for all $v \in \sK,$ $N \mu(B(v,\epsilon/2) \cap S^{d-1}) \geq 1$ or
\begin{align}\label{prop_sphere}
\mu(B(v,\epsilon/2) \cap S^{d-1}) \geq \frac{1}{N}.
\end{align} In other words, if we fix some $v \in K$ and sample a uniform point $w$ on the sphere, it holds that with probability at least $1/N$ we have $\|w-v\|_2 \leq \epsilon/2.$

Hence, the probability that $M$ random i.i.d. unit vectors $w_1,\ldots,w_M$ are all at distance more than $\eps/2$ from a fixed $v \in \sK$ is upper bounded by
\begin{align*}
    \PP\left(\bigcap_{i=1}^M \{\|u_i - v\|_2 > \eps/2 \}\right) \le (1-1/N)^m \le \exp(-m/N)\;.
\end{align*}

Now let $M=2N\log N$. By the union bound, the probability that there exists some $v \in \sK$ not covered by $M$ random unit vectors $w_1,\ldots,w_M$ is upper bounded by
\begin{align*}
    \PP\left(\bigcup_{v \in \sK}\{\|u_i - v\|_2 > \eps/2 \text{ for all }i=1,\ldots,M \}\right) \le |\sK| \cdot \exp(-M/N) \le 1/N \;.
\end{align*}

Since $N = \exp(\Omega(d))$, we conclude that $M=2N\log N$ random unit vectors form an $\eps$-cover of $S^{d-1}$ with probability $1-\exp(-\Omega(d))$. The proof is complete.
\end{proof}

\begin{algorithm}[t]
\caption{Exponential-time algorithm for constructing an $\eps$-cover of the unit sphere}
\label{alg:covering-algo}
\KwIn{A real number $\eps \in (0,1)$, and natural number $d \in \NN$.}
\KwOut{An $\eps$-cover of the unit sphere $S^{d-1}$ containing $2N \log N$ points, where $N = (1+4/\eps)^d$ with probability $1-\exp(-\Omega(d))$.}
\vspace{1mm} \hrule \vspace{1mm}
Initialize the cover $\sC = \emptyset$, and set $m = 2N\log N$. \\
\For{$i=1$ \KwTo $m$}{
    Sample $x \sim N(0,1)$ \\
    $v \leftarrow x/\|x\|_2$ \\
    Add $v \in S^{d-1}$ to $\sC$
    }
\Return $\sC$.
\end{algorithm}

\section{The Population Loss and Parameter Estimation}
\label{app:population}
Let $f(x) = \cos(2\pi \gamma \inner{w,x})$ be the target function defined on Gaussian inputs $x \sim N(0,I_d)$. In this section, we consider the proper learning setup, where we wish to learn a unit vector $w'$ such that the hypothesis $g_{w'}(x) = \cos(2\pi \gamma \inner{w',x})$ achieves small squared loss with respect to the target function $f$. Towards this goal, we define the squared loss associated with a unit vector $w' \in S^{d-1}$.

\begin{definition}
Let $d \in \NN, \gamma \ge 1$, and $w \in S^{d-1}$ be some fixed hidden direction. For any $w' \in S^{d-1}$, we define the population loss $L(w')$ of the hypothesis $g_{w'}(x)=\cos(2\pi \gamma \inner{w',x})$ with respect to $w$ by 
\begin{align}
\label{eq:cosine-population-loss}
L(w')=\mathbb{E}_{x \sim N(0,I_d)}[(\cos(2\pi \gamma \inner{w,x})-\cos(2\pi \gamma \inner{w',x}))^2]\;.
\end{align}
\end{definition}

Notice that because the cosine function is even, the population loss inherits the sign symmetry and satisfies that $L(w')=L(-w')$ for all $w' \in S^{d-1}$. Reflecting that symmetry, we obtain a Lipschitz relation between the population loss and the squared $\ell_2$ difference between $w$ and $w'$ (or $-w'$ if $\|w+w'\|_2 \le \|w-w'\|_2)$. In particular, when $\gamma$ is diverging, we can rigorously show that recovery of $w$ with $o(1/\gamma)$ $\ell_2$-error is sufficient for (properly) learning the associated cosine function with constant edge. This is formally stated in Corollary~\ref{cor:population-parameter-recovery}. We start with the following useful proposition.

\begin{proposition}\label{prop:population}
For every $w' \in S^{d-1}$ it holds
\begin{align}
\label{eq:cosine-hermite-expansion}
    L(w')=2\sum_{k \in 2 \mathbb{Z}_{\geq 0}} \frac{(2\pi \gamma)^{2k}}{k!} \exp(-4 \pi^2 \gamma^2) \left(1-\inner{w,w'}^{k}\right).
\end{align}
In particular, 
\begin{align}
    L(w') & \leq 4 \pi^2 \gamma^2 \min\{\|w-w'\|^2_2,\|w+w'\|^2_2\}.
\end{align}
\end{proposition}

\begin{proof}
Let $\{h_k\}_{k\in \ZZ_{\ge 0}}$ be the (probabilist's) normalized Hermite polynomials. We have that the pair $Z=\inner{w,x}, Z_{\rho}=\inner{w',x}$ is a bivariate pair of standard Gaussian random variables with correlation $\rho = \inner{w,w'}$. Using the fact that $h_k$'s form an orthonormal basis in Gaussian space (See item (1) of Lemma~\ref{lem:hermite}), we have by Parseval's identity that
\begin{align*}
   L(w')&=2(\mathbb{E}[\cos (2\pi \gamma Z)^2]-\mathbb{E}[\cos (2\pi \gamma Z)\cos (2\pi \gamma Z_{\rho})])\\
   &=2\sum_{k \in \mathbb{Z}}\left(\mathbb{E}[\cos (2\pi \gamma Z)h_k(Z)]^2-\mathbb{E}[\cos (2\pi \gamma Z)h_k(Z)]\mathbb{E}[\cos (2\pi \gamma Z_{\rho})h_k(Z)]\right).
\end{align*}
Using now item (2) of Lemma \ref{lem:hermite} for $\rho=1$ and for $\rho=\inner{w,w'}$, we have
\begin{align*}
   L(w')&=2\sum_{k \in \mathbb{Z}}\left(\frac{(2\pi \gamma)^{2k}}{k!} \exp(-4 \pi^2 \gamma^2)-\inner{w,w'}^{k}\frac{(2\pi \gamma)^{2k}}{k!} \exp(-4 \pi^2 \gamma^2)\right)\\
   &=2\sum_{k \in 2 \mathbb{Z}_{\geq 0}} \frac{(2\pi \gamma)^{2k}}{k!} \exp(-4 \pi^2 \gamma^2) \left(1-\inner{w,w'}^{k}\right),
\end{align*}
as we wanted for the first part.

For the second part, notice that since the summation on the right hand from Eq.~\eqref{eq:cosine-hermite-expansion} is only containing an even power of $\inner{w,w'}$ it suffices to establish the upper bound in terms of $\|w-w'\|^2_2$. The exact same argument can be used to obtain the upper bound in terms of $\|w+w'\|^2_2$, due to the observed sign symmetry of the population loss with respect to $w'$.

Now notice that using the elementary inequality that for $\alpha \in (0,1), x \geq 1$ we have $(1-a)^x \geq 1-ax$, we conclude that for all $k \geq 0$ (the case $k=0$ is trivial) it holds
\begin{align*}
  1-\inner{w,w'}^k =1-(1-\frac{1}{2}\|w-w'\|^2_2)^k \leq \frac{k}{2} \|w-w'\|^2_2\;.
\end{align*}
Hence, combining with the first part, we have 
\begin{align*}
   L(w')&\leq \sum_{k \in 2 \mathbb{Z}_{\geq 0}} k\frac{(2\pi \gamma)^{2k}}{k!} \exp(-4 \pi^2 \gamma^2) \|w-w'\|^2_2\\
   & \leq \sum_{k \in  \mathbb{Z}_{\geq 0}} k\frac{(2\pi \gamma)^{2k}}{k!} \exp(-4 \pi^2 \gamma^2) \|w-w'\|^2_2\;.
\end{align*}
Now notice that $\sum_{k \in  \mathbb{Z}_{\geq 0}} k\frac{(2\pi \gamma)^{2k}}{k!} \exp(-4 \pi^2 \gamma^2)$ is just the mean of a Poisson random variable with parameter (and mean)  equal to $4\pi^2 \gamma^2.$ Hence, the proof of the second part of the proposition is complete.
\end{proof}
The following Corollary is immediate given the above result and the item (3) of Lemma \ref{lem:hermite}.
\begin{corollary}
\label{cor:population-parameter-recovery}
Let $d \in \NN$ and $\gamma=\gamma(d)=\omega(1)$. For any $w' \in S^{d-1}$ which satisfies $\min\{\|w-w'\|^2_2,\|w+w'\|^2_2\} \leq \frac{1}{16\pi^2 \gamma^2}$ and sufficiently large $d$,
\begin{align*}
    L(w')\leq \mathrm{Var}(\cos(2\pi \gamma \inner{w,x}))-1/12\;.
\end{align*}
\end{corollary}

\begin{proof}
Using our condition and $w'$ and the second part of the Proposition \ref{prop:population} we conclude
\begin{align*}
   L(w')\leq \frac{1}{4}\;.
\end{align*}
Now using item (3) of Lemma \ref{lem:hermite} we have that for large values of $d$ (since $\gamma=\omega(1)$), it holds
\begin{align*}
   \frac{1}{3}\leq \mathrm{Var}(\cos(2\pi \gamma \inner{w,x}))\;.
\end{align*}
The result follows from combining the last two displayed inequalities.
\end{proof}

\section{Optimality of \texorpdfstring{$d+1$}{d+1} samples for exact recovery under norm priors}
\label{optim_sample_app}

In this appendix, we argue that $d+1$ samples are necessary in order to obtain exact recovery with probability $1-\exp(-\Omega(d))$, 
irrespective of any estimation procedure. 
Since our upper bound holds for arbitrary $w/\|w\|_2 \in {S}^{d-1},$ and arbitrary $1 \leq \gamma =\|w\|_2=\mathrm{poly}(d)$, 
it suffices to prove a lower bound for \emph{some} distributional assumption on $\gamma$ and $w/\|w\|_2$ which respects these constraints. 
Hence, for our lower bound, we assume a uniform prior on the direction $w/\|w\|_2 \in S^{d-1}$, and assume that $\gamma=\|w\|_2>0$ is distributed independently of $w$ according to a probability density $q_\gamma$ which satisfies the following assumption.
\begin{assumption}
\label{assump:basicapp} For some $B>\sqrt{2}$ and $C>0$, the function $q_{\gamma}: \mathbb{R} \rightarrow [0,\infty)$ satisfies that
$q_{\gamma}(t) t^{-d+1}$ is non-increasing for $t \in [1, B]$, and  $\int_{\sqrt{2}}^B q_{\gamma}(t) dt \geq C$. 
\end{assumption}

We now state our lower bound, restating Theorem \ref{thm:IToptim} for convenience.
\begin{theorem}
\label{thm:IToptimapp}
Consider $d\geq 2$ samples $\{(x_i, y_i=|\inner{x_i, w}|) \}_{i=1\ldots d}$, in which the $x_i$'s are drawn i.i.d. from $N(0,I_d)$, and $w$ is drawn from two independent variables: $w/\|w\|$ uniformly distributed in $S^{d-1}$ and $\|w\|$ distributed with density satisfying Assumption \ref{assump:basicapp}. Let $\mathcal{A}$ be any estimation procedure (deterministic or randomized) that takes as input $\{(x_i, y_i)\}_{i=1,\ldots,d}$ and outputs $w' \in \mathbb{R}^d$. Then with probability $\omega(d^{-2})$ it holds $w' \not \in \{-w,w\}.$
\end{theorem}

\begin{proof}
The key idea of the proof will be to establish that with probability $\omega(d^{-2})$ over the draws of the data $\{ x_i \}_{i=1,\ldots,d}$ and the hidden vector $w$, the following event occurs: There exist a pair of antipodal solutions $\{-w',w'\}$ different from $\pm w$, such that the posterior probability measure $p(\tilde{w}~|~ \{(x_i, y_i)\}_{i=1,\ldots,d})$ over any possible hidden vector $\tilde{w} \in \mathbb{R}^d$ 
satisfies $p( \{-w',w'\}~|~ \{(x_i, y_i)\}) \geq p( \{- w, w \}| ~|~\{(x_i, y_i)\})$. In this event, the MAP estimator will thus fail to exactly recover $\{- {w}, {w} \}$ at least with probability $1/2$ (over the randomness of the algorithm). Finally, using the  optimality of the Maximum-a-Posteriori Bayes estimator in minimizing the probability of error, the result follows.

Let $X=(x_i)_{i=1\ldots d} \in \mathbb{R}^{d \times d}$, be the matrix where for $i=1,\ldots,d$ with $i$-th \emph{row}
equal to $x_i^\top$, and $X^{-1}$ its inverse (which exists with probability $1$ since the determinant of a squared matrix with i.i.d. Gaussian entries is non-zero almost surely \cite{caron2005zero}). Furthermore, let $y = (y_i)_{i=1\ldots d} \in \mathbb{R}^d$ the vector of the labels.
Let us introduce binary variables $\varepsilon \in \{-1,1\}^d$, and the associated matrix
\begin{align*}
    A_\varepsilon := X^{-1} \mathrm{diag}(\varepsilon) X, \;.
\end{align*}
where by $\mathrm{diag}(\varepsilon)$ we refer to the $d \times d$ diagonal matrix with the vector $\varepsilon$ on the diagonal.

We say that a $w' \in \mathbb{R}^d$ is a feasible solution if for all $i=1,\ldots, d$ it holds that $|\inner{ x_i, w'} | = y_i$. Notice that if $w'$ is a feasible solution, then for any $\varepsilon \in \{-1,1\}^d$, 
$A_\varepsilon w'$ is also a feasible solution. This follows since for each $i=1,\ldots,d$ it holds by definition $x_i^\top X^{-1}=e_i^{\top}$, where $e_i$ is the $i$-th standard basis vector, and therefore $x_i^\top A_{\epsilon}=\varepsilon_i x_i^\top$. Hence we have
\begin{align*}
|x_i^\top A_\varepsilon w'|= | \varepsilon_i x_i^\top w' | = y_i\;.
\end{align*} 
On the other hand, if $w'$ is a feasible solution, then there exists $\varepsilon \in \{-1,1\}^d,$ for which for all $i=1,\ldots, d$, it holds $\inner{x_i, w'} = \varepsilon_i y_i$. Therefore, using the definition of $y_i$ and the already established properties of $A_{\varepsilon}$,
\begin{align*} 
\inner{x_i, w'} = \varepsilon_i y_i= x_i^\top \varepsilon_i {w} =x_i^\top A_{\varepsilon}w \;.
\end{align*}

Hence, $X(w'-A_{\varepsilon}{w})=0$. As $X$ is invertible almost surely, we conclude that $w'=A_{\varepsilon}{w}.$ Combining the above, we conclude that the set of feasible solutions is almost surely the set $$\mathcal{B}_{{w}}=\{A_{\varepsilon}{w} | \varepsilon \in \{-1,1\}^d\}.$$ Of course, this set includes ${w}$ when $\varepsilon=\mathbf{1}$ is the all-one vector, and $-{w}$ when $\varepsilon=-\mathbf{1}$ is the all-minus-one vector. Furthermore, from the almost sure linear independence of all $x_i, i=1,\ldots, d+1$, and that $w$ is drawn independent of $X$, we conclude that for all $\varepsilon \not \in \{-\mathbf{1},\mathbf{1}\}$ it holds almost surely that $A_{\varepsilon}{w} \not \in \{-{w},{w}\}.$

Now consider the joint density of the setup in this notation (where we recall that $\tilde{w} \in \mathbb{R}^d$ denotes the generic vector to be recovered, while $w$ is the actual draw of the prior), which decomposes as
\begin{align*} p(X, \tilde{w}, y) = p_X(X) \cdot p_{\tilde{w}}(\tilde{w}) \cdot p( y~|~X, \tilde{w})~, X \in \mathbb{R}^{d \times d}, \tilde{w} \in \mathbb{R}^d, y \in \mathbb{R}^d~.\end{align*}
Notice that since we work under the noiseless assumption it holds 
$p(y~|~X, \tilde{w}) = \delta\left(y- |X \tilde{w}|\right)$,
where by a slight abuse of notation for a vector $v \in \mathbb{R}^d$ we denote by $|v| \in \mathbb{R}^d$ the vector with elements $|v_i|, i=1,\ldots,d$.
Further recall that in this notation we sample a hidden ${w} \sim p_{\tilde{w}}$ and independently a matrix  $X \sim p_X$. We observe the vector of labels $y = |X w|$ and $X$. The posterior probability $p( \tilde{w} ~|~X, y)$ is therefore 
\begin{align}
\label{eq:posteriorprob}
p( \tilde{w} ~|~X, y) = \frac{p(X, \tilde{w}, y)}{p(X, y)} \propto p_{\tilde{w}}(\tilde{w}) \cdot p( y~|~X, \tilde{w})~.    
\end{align}
From our previous argument, we know that this posterior distribution is  necessarily supported in the set $\mathcal{B}_{{w}}$ of $2^d$ points of the form $(X^{-1} \cdot \mathrm{diag}(\varepsilon)) y$ for any $\varepsilon \in \{-1,1\}^d$, which include ${w}$. Denoting by $\delta(\tilde{w})$ the Dirac unit mass at $\tilde{w}$, we have
\begin{align}
\label{eq:bubu}
p( \tilde{w} ~|~X, y) = \frac{1}{Z}\sum_{w' \in \mathcal{B}_{{w}}} \alpha_{X, y}(w') \delta(\tilde{w} - w')~,    
\end{align}
for some normalizing constant $Z$ and some coefficients $\alpha_{X, y}(\varepsilon)$ that we now determine. 
We evaluate the posterior distribution over $\tilde{w}$ from (\ref{eq:posteriorprob}) using the coarea formula~\cite{maly2003co}: Given an arbitrary test function $\phi \in C^\infty_c(\mathbb{R}^d)$, and $F: \mathbb{R}^d \to \mathbb{R}^d$ defined as $F(u):= | X u |$, we have
\begin{eqnarray}
Z \int_{\mathbb{R}^d} p( \tilde{w}~|~X, y) \phi(\tilde{w}) d\tilde{w} &=& \int_{\mathbb{R}^d} p_{\tilde{w}}( \tilde{w}) \delta( y - F(\tilde{w})) \phi(\tilde{w}) d\tilde{w} \\ 
&=& \int_{\mathbb{R}^d} \left(\int_{ F^{-1}(z) } \delta(y - z) p_{\tilde{w}}( u) \phi(u) |DF(u)|^{-1} d\mathcal{H}_0(u) \right) dz \\
&=& \int_{\mathbb{R}^d} \delta(y- z) \left(\int_{ \mathcal{B}_z } p_{\tilde{w}}( u) \phi(u) |DF(u)|^{-1} d\mathcal{H}_0(u) \right) dz \\
&=& \sum_{w' \in \mathcal{B}_w} p_{\tilde{w}}( w' ) \phi( w' ) |\mathrm{det}(X)|^{-1} ~,
\end{eqnarray}
where $d\mathcal{H}_0$ is the $0$-th dimensional Hausdorff measure. 
From (\ref{eq:bubu}) we also have that 
\begin{align*}
 \int_{\mathbb{R}^d} p( \tilde{w}~|~X, y) \phi(\tilde{w}) d\tilde{w} = \sum_{w' \in \mathcal{B}_w} \alpha_{X,y}(w') \phi( w')\;,
\end{align*}
hence we deduce that the weights in (\ref{eq:bubu}) satisfy
\begin{align*}
\forall~\varepsilon~,~    \alpha_{X, y}(X^{-1} \cdot \mathrm{diag}(\varepsilon) y) = p_{\tilde{w}}( X^{-1} \cdot \mathrm{diag}(\varepsilon) y ) |\mathrm{det}(X)|^{-1}~.
\end{align*}
By plugging $y = |X {w}| = \mathrm{diag}(\varepsilon^*) X {w}$ for the sign coefficients $\varepsilon^*_i = \mathrm{sign}( \inner{x_i, w} )$, and recalling the definition of $A_\varepsilon$, we conclude that the posterior distribution over the hidden vector $\tilde{w}$ satisfies almost surely

\begin{align*} 
p( \tilde{w} ~|~X, y)= \left\{
\begin{array}{ll}
      \frac{1}{{Z}} p_{\tilde{w}}(\tilde{w}) & \tilde{w} \in \mathcal{B}_{{w}} \\
      0 & \tilde{w} \not \in \mathcal{B}_{{w}}\\
\end{array} 
\right.
\end{align*}
where ${Z}:=\sum_{\tilde{w} \in \mathcal{B}_{{w}}} p_{\tilde{w}}(\tilde{w})$.

Now to prove the desired result, based on the folklore optimality of the Maximum-A-Posteriori (MAP) estimator in minimizing probability of failure of exact recovery (see Lemma \ref{lem:mapoptim} for completeness) it suffices to prove that with probability $\omega(d^{-2})$ there exists  ${w}' \in \mathcal{B}_{{w}} \setminus \{-{w}, {w}\}$ such that 
\begin{align}
\label{eq:b00cond}
p_{\tilde{w}}({w}')\geq p_{\tilde{w}}({w})~.    
\end{align}
Indeed, recall that since $p_{\tilde{w}}$ is rotationally invariant, we have $p_{\tilde{w}}(\tilde{w}) = p_{\tilde{w}}(-\tilde{w})$ for any $\tilde{w}$, therefore (\ref{eq:b00cond}) immediately implies $p_{\tilde{w}}(\pm {w}')\geq p_{\tilde{w}}(\pm {w}).$ Hence, the  MAP estimator (and therefore any estimator) fails to exactly recover an element of $\{w,-w\}$ with probability  $\omega(d^{-2})$, as we wanted.

Now, using a standard change of variables to spherical coordinates, for all $\tilde{w} \in \mathbb{R}^d$ the density of the prior equal to
$p_{\tilde{w}}(\tilde{w}) = q_\gamma(\| \tilde{w}\|_2 ) \| \tilde{w}\|_2^{-d+1}$.  In particular, based on Assumption~\ref{assump:basic} it suffices to prove that with probability $\omega(d^{-2})$ there exists a ${w}'  \in \mathcal{B}_{{w}} \setminus \{-{w}, {w}\}$ such that $1 \leq \|{w}'\|_2<\|{w}\|_2,$ or equivalently there exists  $\varepsilon \in \{-1,1\}^d \setminus \{- \mathbf{1}, \mathbf{1}\}$ such that
\begin{align}
\label{eq:mainidea}
 1 \leq \| A_\varepsilon {w} \|_2< \|{w} \|_2 ~.   
\end{align}

We establish (\ref{eq:mainidea}) by actually studying only one such $\varepsilon,$ potentially the simplest choice, which we call  $\varepsilon^{(1)}$ where $\varepsilon^{(1)}_1 = -1$ and $\varepsilon^{(1)}_j = +1$ for $j=2,\ldots,d$. This is accomplished by the following key lemma:
\begin{lemma}
\label{lem:ITlem_new}
Suppose $X \in \mathbb{R}^{d \times d}$ has i.i.d. $N(0,I_d)$ entries, and ${w}$ is drawn independently of $X$, such that $w/\|w\|_2$ is drawn from the uniform measure of $S^{d-1}$ and its norm $\|w\|_2$ is independent of $w/\|w\|_2$ and distributed according to a density $q_{\gamma}$ satisfying Assumption~\eqref{assump:basic}. Set also $A_{\varepsilon^{(1)}}=X^{-1}\mathrm{diag}(\varepsilon^{(1)})X$. Then with probability greater than $ \omega(d^{-2})$, it holds
\begin{eqnarray}
1  \leq   \| A_{\varepsilon^{(1)}} {w} \|_2 < \|{w}\|_2 ~. \label{eq:ITpart1} 
\end{eqnarray}
\end{lemma}
This lemma thus proves (\ref{eq:mainidea}) and the failure of the MAP estimator with probability $\omega(d^{-2})$. 

 We conclude the proof by formally stating and using the optimality of the MAP estimator in terms of minimizing the error probability, by relating it to a standard error correcting setup. From our previous argument, we can reduce ourselves to decoders that operate in the discrete set $\mathcal{B}_{{w}}$, since any $\tilde{w}$ outside this set will be different from $\pm {w}$ almost surely. 
\begin{lemma}
\label{lem:mapoptim}
Suppose $\mathcal{X}$ is a discrete set, and let ${x}^* \in \mathcal{X}$ be an element to be recovered, with posterior distribution $p(x | y )$, $x\in \mathcal{X}$, after having observed the output $y=g(x^*)$. Then, any estimator producing $\hat{x}=\hat{x}(y)$ will incur in an error probability $\mathbb{P}(\hat{x} \neq x^*)$ at least $1 - \max_x p(x|y)$, with equality if $\hat{x}$ is the Maximum-A-Posterior (MAP) estimator which outputs $\arg \max_x p(x|y).$
\end{lemma}
We apply the Lemma \ref{lem:mapoptim} for $\mathcal{X}$ containing all the pairs of antipodal elements of $\mathcal{B}_w$, that is $\mathcal{X}=\{ \{w',-w'\} : w' \in \mathcal{B}_{w}\}$ and $x^*=\{w,-w\}.$ As we have established that the MAP estimator fails to exactly recover $x^*$ with probability $\omega(d^{-2})$ this completes the proof.

\end{proof}
\subsection{Proof of Lemma \ref{lem:ITlem_new}}
\begin{proof}

If $e_1$ denotes the first standard basis vector, observe that by elementary algebra,
\begin{align}
\label{eq:basic0}
A_{\varepsilon^{(1)}} = X^{-1} \left( I_d - 2 e_1 e_1^\top \right) X = I_d - 2 \tilde{x}_1 x_1 ~,    
\end{align}
where $x_1^\top$ is the first row of $X$ and $\tilde{x}_1$ is the first column of $X^{-1}$. 

We need a spectral decomposition of matrices of the form $A = I_d - 2 u v^\top$, which is provided in the following lemma:
\begin{lemma}
\label{lem:basiclem}
Let $\eta \in \mathbb{R}$ and $A = I_d - 2 \eta u v^\top \in \mathbb{R}^{d \times d}$, with $\|u \|_2 = \|v \|_2=1$, and $\alpha = \inner{u, v}$. 
Then 
$A^\top A$ has the eigenvalue $1$ with multiplicity $d-2$, 
and two additional eigenvalues $\lambda_{1}, \lambda_2$ with multiplicity $1$ given by
\begin{align}
\label{eq:basic1}
\lambda_{1} = 1 + 2\eta \left( \eta - \alpha - \sqrt{\eta^2 + 1 - 2\eta \alpha} \right) ~,~ \lambda_{2} = 1 + 2\eta \left( \eta - \alpha + \sqrt{\eta^2 + 1 - 2\eta \alpha} \right) ~.    
\end{align}
In particular, $\lambda_{\mathrm{min}}(A^\top A) = \lambda_1 < 1$ and $\lambda_{\mathrm{max}}(A^\top A) = \lambda_2 >1$ whenever $\eta > 0$ and $|\alpha| <1 $.
\end{lemma}

From (\ref{eq:basic0}), we now apply Lemma \ref{lem:basiclem}. By noting that $\inner{x_1, \tilde{x}_1} = 1$ since $X X^{-1} = I_d$, note that the lemma applies for $A_{\varepsilon^{(1)}}$ with parameters
\begin{align*}
\alpha = \left\langle \frac{x_1}{\|x_1\|_2}, \frac{\tilde{x}_1}{\|\tilde{x}_1\|_2} \right\rangle =  \frac{1}{\| x_1 \|_2 \cdot \| \tilde{x}_1 \|_2}~\text{, and } \eta = \| x_1 \|_2 \cdot \| \tilde{x}_1 \|_2\;.\end{align*} 
Since $|\alpha| \in (0,1]$ by Cauchy-Schwarz and
and $\alpha \eta=1$, it follows that $\eta \geq 1$ and the eigenvalues of $A_{\varepsilon^{(1)}}^\top A_{\varepsilon^{(1)}}$ are
$\left(\lambda_{\mathrm{min}}(A_{\varepsilon^{(1)}}^\top A_{\varepsilon^{(1)}}), 1, \ldots,1, \lambda_{\mathrm{max}}(A_{\varepsilon^{(1)}}^\top A_{\varepsilon^{(1)}})\right) $, with 
\begin{eqnarray}
\label{eq:basic2}
\lambda_{\mathrm{min}}(A_{\varepsilon^{(1)}}^\top A_{\varepsilon^{(1)}}) &=&  1 + 2 \eta\left(\eta - \alpha - \sqrt{\eta^2 - 1 }\right)  = -1 + 2\eta^2 - 2\eta \sqrt{\eta^2-1}\\
\lambda_{\mathrm{max}}(A_{\varepsilon^{(1)}}^\top A_{\varepsilon^{(1)}})&=& 1 + 2\eta \left(\eta - \alpha + \sqrt{\eta^2 -1 }\right) = -1 + 2\eta^2 + 2\eta \sqrt{\eta^2-1} ~.
\end{eqnarray}
In fact, we claim that $|\alpha|<1$ with probability $1$, which by Lemma \ref{lem:basiclem} implies that 
\begin{align}
\label{eq:cuc2}
\lambda_{\mathrm{min}}(A_{\varepsilon^{(1)}}^\top A_{\varepsilon^{(1)}})<1<\lambda_{\mathrm{max}}(A_{\varepsilon^{(1)}}^\top A_{\varepsilon^{(1)}})\;. 
\end{align}
Indeed, recalling from Lemma \ref{lem:basiclem} that by definition $\alpha = \inner{\frac{{x}_1}{\|{x}_1\| },\frac{\tilde{x}_1}{\|\tilde{x}_1\| }}$ with $\tilde{x}_1 = (X^\top X)^{-1} x_1$, first observe that $|\alpha|<1$ almost surely. Indeed, $|\alpha|=1$ iff $\tilde{x}_1$ and $x_1$ are colinear, that is for some scalar $\lambda$ it holds $(X^\top X)^{-1} x_1 = \lambda x_1$, which in particular implies that $x_1$ is an eigenvector of $(X^\top X)^{-1}$, or equivalently of $X^\top X$. 
Letting $y_i = x_i^\top x_1$, this means that
\begin{align*}
    \lambda x_1 =(X^\top X)x_1 = \left(\sum_i x_i x_i^\top \right) x_1 = \sum_i x_i y_i \;.
\end{align*}
Since $X$ has rank $d$ almost surely, $\{x_i\}_{i=1\ldots d}$ are linearly independent almost surely, which in turn implies that $y_i = \inner{x_1, x_i} = 0$ for $i \neq 1$ almost surely. This is a $0$-probability event since the $x_i$'s are continuously distributed and independent of each other. 

In what follows to ease notation we denote $\varepsilon^{(1)}$ simply by $\varepsilon$ and in particular $A_{\varepsilon^{(1)}}$ simply by $A_{\varepsilon}$. In the following lemma we establish that $\eta \lesssim d^2$ with probability close to $1$. The proof of this fact is given in Section~\ref{sec:auxiliary}. More precisely, we claim the following:
\begin{lemma}
\label{lem:etabounds}
There exist constants $C>0$ and $d_0 > 0$ such that for any $d \geq d_0$,
\begin{align*}
    \mathbb{P}\left(\eta \leq C d^2 \right) \geq 1 - 1/d\;.
\end{align*}
\end{lemma}
 
We shall now establish (\ref{eq:ITpart1}) building from Lemma \ref{lem:etabounds}. We first relate the spectrum of $A_\varepsilon$ with the probability that $ \| A_\varepsilon {w} \|_2 < \|{w} \|_2$ or equivalently $ \left\| A_\varepsilon \frac{{w}}{\|{w} \|_2} \right\|_2 < 1$. Let $\check{w}:=w/\|w\|$, so 
 ${w} = \gamma \check{w}$, with
 $\check{w} \in {S}^{d-1}$ uniformly distributed, and independent from $\gamma$. 
We claim that with respect to the randomness of $\check{w}$ but conditioning on $X$ it holds \begin{align}
\label{eq:ITlem_rr} 
\mathbb{P}_{\check{w}} ( \| A_\varepsilon \check{w} \| < 1 ) = \frac{2}{\pi} \arcsin\left(\sqrt{\frac{1-\lambda_\mathrm{min}(A_\varepsilon^\top A_\varepsilon)}{\lambda_\mathrm{max}(A_\varepsilon^\top A_\varepsilon)-\lambda_{\mathrm{min}}(A_\varepsilon^\top A_\varepsilon)}} \right)~.  
\end{align}
Indeed, assuming without loss of generality that the two eigenvectors of $A_\varepsilon^\top A_\varepsilon$ associated with the distinct eigenvalues $\lambda_{\min}(A_\varepsilon^\top A_\varepsilon)$ and $\lambda_{\max}(A_\varepsilon^\top A_\varepsilon)$ are respectively $e_1$ and $e_2$, the first two standard basis vectors, we have that \begin{align*}\|A_\varepsilon \check{w} \|_2^2 =  \lambda_{\mathrm{min}}(A_\varepsilon^\top A_\varepsilon) \check{w}_1^2 + \lambda_{\mathrm{max}}(A_\varepsilon^\top A_\varepsilon) \check{w}_2^2 + \sum_{i>2} \check{w}_i^2~,\end{align*}
and therefore, using the uniform distribution on ${S}^{d-1}$ of $\check{w}$, it holds
\begin{eqnarray}
\label{eq:cuc1}
\mathbb{P}_{\check{w}} ( \| A_\varepsilon \check{w} \|_2 < 1 ) &=&  \mathbb{P}_{\check{w}} ( \| A_\varepsilon \check{w} \|_2^2 \leq \|\check{w}\|^2 ) \nonumber \\
&=& \mathbb{P}_{\check{w}} ( \lambda_{\mathrm{min}}(A_\varepsilon^\top A_\varepsilon) \check{w}_1^2 + \lambda_{\mathrm{max}}(A_\varepsilon^\top A_\varepsilon) \check{w}_2^2 \leq \check{w}_1^2 + \check{w}_2^2 ) \nonumber \\
&=& \mathbb{P}_{\check{w}} \left( \lambda_{\mathrm{min}}(A_\varepsilon^\top A_\varepsilon) \frac{\check{w}_1^2}{\check{w}_1^2 + \check{w}_2^2} + \lambda_{\mathrm{max}}(A_\varepsilon^\top A_\varepsilon) \frac{\check{w}_2^2}{\check{w}_1^2 + \check{w}_2^2} \leq 1 \right) \nonumber \\
&=& \mathbb{P}_{\theta \sim U[0,2\pi]} \left( \lambda_{\mathrm{min}}(A_\varepsilon^\top A_\varepsilon) \cos(\theta)^2 + \lambda_{\mathrm{max}}(A_\varepsilon^\top A_\varepsilon) \sin(\theta)^2 \leq 1 \right)~,
\end{eqnarray}
where the last equality follows since the marginal of $\check{w}$ corresponding to the first two coordinates is also rotationally invariant. 

From the last identity of (\ref{eq:cuc1}) and (\ref{eq:cuc2}), we verify that 
\begin{eqnarray*}
 & \mathbb{P}_{\theta \sim U[0,2\pi]}& \left( \lambda_{\mathrm{min}}(A_\varepsilon^\top A_\varepsilon) \cos(\theta)^2 + \lambda_{\mathrm{max}}(A_\varepsilon^\top A_\varepsilon) \sin(\theta)^2 \leq 1 \right) \\
 &=& \frac{1}{2\pi} \int_{0}^{2\pi} \one\left[\lambda_{\mathrm{min}}(A_\varepsilon^\top A_\varepsilon) \cos(\theta)^2 + \lambda_{\mathrm{max}}(A_\varepsilon^\top A_\varepsilon) \sin(\theta)^2 \leq 1 \right] d\theta \\
 &=& \frac{2}{\pi} \int_{0}^{\pi/2} \one\left[\lambda_{\mathrm{min}}(A_\varepsilon^\top A_\varepsilon) \cos(\theta)^2 + \lambda_{\mathrm{max}}(A_\varepsilon^\top A_\varepsilon) \sin(\theta)^2 \leq 1 \right] d\theta \\
 &=& \frac{2}{\pi} \theta^*~,
\end{eqnarray*}
where $\theta^*$ is the only solution in $(0, \pi/2)$ of 
\begin{align}
\label{eq:cuc3}
\lambda_{\mathrm{min}}(A_\varepsilon^\top A_\varepsilon) \cos(\theta)^2 + \lambda_{\mathrm{max}}(A_\varepsilon^\top A_\varepsilon) \sin(\theta)^2 = 1~.    
\end{align}
From (\ref{eq:cuc3}) we obtain directly (\ref{eq:ITlem_rr}), as claimed.

Now, the quantity $\rho:=\frac{1-\lambda_\mathrm{min}(A_\varepsilon^\top A_\varepsilon)}{\lambda_\mathrm{max}(A_\varepsilon^\top A_\varepsilon)-\lambda_{\mathrm{min}}(A_\varepsilon^\top A_\varepsilon)}$, expressed in terms of $\alpha=1/\eta$ and $\eta$ becomes
\begin{align*}
\rho=\frac{1-\lambda_\mathrm{min}(A_\varepsilon^\top A_\varepsilon)}{\lambda_\mathrm{max}(A_\varepsilon^\top A_\varepsilon)-\lambda_{\mathrm{min}}(A_\varepsilon^\top A_\varepsilon)} =  \frac{-\eta^2 + \eta \sqrt{\eta^2-1}+1}{2\eta \sqrt{\eta^2-1} }~,
\end{align*}
and satisfies $0 \leq \rho=\rho(\eta) < 1$ almost surely.  
Denoting 
\begin{align*}
    f(\eta):= \arcsin \left(\sqrt{\rho} \right)\;,
\end{align*}
we verify that $f'(\eta) < 0$ for $\eta \geq 1$. In order to leverage Lemma \ref{lem:etabounds}, we consider the event that $\eta \leq C_2 d^2$. We can lower bound $f(\eta)$ as follows.
First, observe that $t \mapsto \arcsin(\sqrt{t})$ is non-decreasing in $t\in (0, 1)$, thus 
\begin{align*}
f(\eta) \geq \arcsin\left(\sqrt{\frac{\eta(\sqrt{\eta^2 -1} - \sqrt{\eta^2} ) +1 }{2\eta^2}}\right)\;,
\end{align*}
since 
\begin{align*}
    \frac{-\eta^2 + \eta \sqrt{\eta^2-1}+1}{2\eta \sqrt{\eta^2-1} } \geq  \frac{-\eta^2 + \eta \sqrt{\eta^2-1}+1}{2\eta^2 }= \frac{\eta(\sqrt{\eta^2 -1} - \sqrt{\eta^2} ) +1 }{2\eta^2}\;.
\end{align*}
Moreover, since $\sqrt{t+1} - \sqrt{t} = \frac{1}{2\sqrt{t}} + O(t^{-3/2})$, we have that 
\begin{align*}
\frac{\eta(\sqrt{\eta^2 -1} - \sqrt{\eta^2} ) +1 }{2\eta^2} = \frac{3}{4}\eta^{-2} + O(\eta^{-4})\;,
\end{align*}
which, combined with the fact that $\arcsin(t) = t + O(t^3)$ for $|t| \le 1$, leads to 
\begin{align*}
    f(\eta) \geq \frac{3}{4}\eta^{-1} + O(\eta^{-2})\;.
\end{align*}
Finally, using Lemma \ref{lem:etabounds} and the definition of $f(\eta)$, we obtain that 
\begin{align*}
    \mathbb{P}_{\check{w}}(\| A_\varepsilon \check{w} \| \leq 1 ) \geq \frac{6}{4\pi C_2}d^{-2} + O(d^{-4})    
\end{align*}
with probability (over $X$) greater than $1/2$.
Since $X$ and $w$ are independent, we conclude that
\begin{align}
\label{eq:ITlem_r1}
    \mathbb{P}_{X,\check{w}}(\| A_\varepsilon \check{w} \| \leq 1 ) \geq \frac{1}{2} \left(\frac{6}{4\pi C_2}d^{-2} + O(d^{-4})\right)= C_4 d^{-2} + O(d^{-4})~~,
\end{align}
where $C_4$ is a constant.

Now we show that 
\begin{align*}
\mathbb{P}_{\check{w}} \left(\| A_\varepsilon \check{w} \|_2^2 \geq 1 - 1/\sqrt{d}\right) \geq 1-\exp\left(-\Omega(\sqrt{d})\right) \;.
\end{align*}
Recall that $\check{w}$ is distributed uniformly on the sphere $S^{d-1}$, and that all eigenvalues of $A_{\varepsilon}^{\top} A_{\varepsilon}$ are all greater or equal to 1, except for $\lambda_\mathrm{min}$. Assuming without loss of generality that $e_1$ is the eigenvector corresponding to $\lambda_{\mathrm{min}}$, we have for any $\check{w} \in S^{d-1}$,
\begin{align*}
    \|A_\varepsilon \check{w}\|_2^2 \ge 1-\check{w}_1^2\;.
\end{align*}
Let $H$ be the hemisphere $H=\{\check{w}_1 \le 0 \mid \check{w} \in S^{d-1}\}$. By the classic isoperimetric inequality for the unit sphere $S^{d-1}$~\cite[Chapter 1]{ledoux2001concentration}, the measure of the $r$-neighborhood of $H$, which we denote by $H_r = \{u \in S^{d-1} \mid \mathrm{dist}(u,H) \le r\}$, satisfies
\begin{align*}
    \mathbb{P}_{\check{w}}(H_r) = \mathbb{P}_{\check{w}}(\check{w}_1 \le r) \ge 1-\exp(-(d-1)r^2/2)\;.
\end{align*}
An analogous inequality holds for the event $\{\check{w}_1 \ge -r\}$ by the sign symmetry of the distribution of $\check{w}$. Plugging in $r=d^{-1/4}$, It follows that
\begin{align*}
\mathbb{P}_{\check{w}} \left(\| A_\varepsilon \check{w} \|^2 \geq 1 - 1/\sqrt{d}\right) 
&\geq \mathbb{P}_{\check{w}} \left(1-\check{w}_1^2 \geq 1 - 1/\sqrt{d}\right) \\
&= \mathbb{P}_{\check{w}} \left(|\check{w}_1| \le 1/d^{1/4}\right) \\
&\geq 1-\exp\left(-\Omega(\sqrt{d})\right)\;.
\end{align*}

Therefore, combining the above with (\ref{eq:ITlem_r1}) using the union bound, we obtain
\begin{align}
\label{eq:ITlem66}
    \mathbb{P}_{X,\check{w} }\left(\sqrt{1 -  d^{-1/2}} \leq \| A_\varepsilon {\check{w}} \| \leq 1  \right) \geq C_4 d^{-2} + O(d^{-4}) - \exp(-\Omega(\sqrt{d})) = C_4 d^{-2} + O(d^{-4}) ~.
\end{align}
Finally, since $B>\sqrt{2}$ and $\sqrt{1-d^{-1/2}} \ge 1/\sqrt{2}$, we have 
\begin{align}
\label{eq:ITlem67}
\mathbb{P}_{\tilde{w}}( \gamma \sqrt{1 -  d^{-1/2}} \geq 1 ) &=& \mathbb{P}_{\tilde{w}}\left( \gamma \geq \frac{1}{\sqrt{1 -  d^{-1/2}}} \right) \nonumber \\ 
&=& \int_{\frac{1}{\sqrt{1 -  d^{-1/2}}}}^B {q_\gamma}(v) dv := Q_s
\end{align}

Since ${w} = \gamma \check{w}$, where $\check{w}$ is uniformly distributed in $\mathcal{S}^{d-1}$ and $\gamma$ is independent of $\check{w}$, we conclude by assembling (\ref{eq:ITlem66}) and (\ref{eq:ITlem67}) that

\begin{align*}
    \mathbb{P}_{X,{w}} \left(1 \leq \| A_\varepsilon {w} \| \leq \|{w}\|  \right) \geq (C_4 d^{-2} + O(d^{-4})) Q_s = C_5 d^{-2} + O(d^{-4})~,
\end{align*}
since $Q_s \geq Q_{1/\sqrt{2}} \geq C$ 
for $d\geq 2$ thanks to Assumption \ref{assump:basic}. This concludes the proof of (\ref{eq:ITpart1}).
\end{proof}

\subsection{Auxiliary Lemmas}
\label{sec:auxiliary}

\begin{proof}[Proof of Lemma \ref{lem:mapoptim}]
Observe that
\begin{align*}
    \mathbb{P}( \hat{x} \neq x^*) = 1- \mathbb{P}(\hat{x} = x^*) = 1 - p(\hat{x} | y) \geq 1 - \max_x p( x| y)\;,
\end{align*}
with equality if $\hat{x}$ is the Maximum-a-Posteriori estimator. 
\end{proof}

\begin{proof}[Proof of Lemma \ref{lem:basiclem}]
First notice that we can reduce to a two-by-two matrix, since the directions orthogonal to both $u$ and $v$ clearly belong to an eigenspace of eigenvalue $1$. 
The result follows directly by computing the characteristic equation $\det[ A^\top A - \lambda I] = 0$.
\end{proof}

\begin{proof}[Proof of Lemma \ref{lem:etabounds}]
First, observe that since the law of $X$ is rotationally invariant, we can 
assume without loss of generality that $x_1$ is proportional to $e_1^\top$, the first standard basis vector. Using the Schur complement, we have
\begin{align}
\label{eq:simpler}
X = \begin{pmatrix} 
\|x_1\|_2 & 0 \\
v & \bar{X} 
\end{pmatrix}~,\text{ and }~ X^{-1} = \begin{pmatrix} 
\|x_1\|_2^{-1} & 0 \\
b & \bar{X}^{-1} 
\end{pmatrix}~,
\end{align}
where $v$ is the $(d-1)$-dimensional vector given by $v_i = \|x_1\|_2^{-1} \inner{x_1, x_{i+1}} = x_{i+1,1} \sim N(0, 1)$, $\bar{X}$ is a $(d-1) \times (d-1)$ matrix whose entries are drawn i.i.d. from
$N(0,1)$, and $b = - \|x_1\|_2^{-1} \bar{X}^{-1} v$. 
Observe that $\bar{X}$ and $v$ are independent, since the choice of basis depends only on $x_1$. The coordinates of $v$ are independent as well for the same reason.
It follows that
\begin{align}
\label{eq:basicschur}
\| \tilde{x}_1 \|_2^2 
&= \|x_1\|_2^{-2}\left(1 + \| \bar{X}^{-1} v\|_2^2\right)  \nonumber \\
&\le \|x_1\|_2^{-2}\left(1 + \|\bar{X}^{-1}\|^2\cdot \|v\|_2^2\right)\;, 
\end{align}
where $\|\bar{X}^{-1}\| = \max_{u \in S^{d-1}} \|\bar{X}^{-1}u\|_2$ is the operator norm of $\bar{X}^{-1}$. 
Now let $\alpha$ be a fixed constant, which will be specified later. Additionally, assume that $d$ is sufficiently large so that $\alpha d^4 \ge 2$. From Eq.~\eqref{eq:basicschur}, we have that 
\begin{align}
\mathbb{P}\{ \eta^2 \ge \alpha d^4\} 
&\leq \mathbb{P}\left\{  \|x_1\|_2^2 \left( \|x_1\|_2^{-2}\left(1 + \|\bar{X}^{-1}\|^2\cdot \|v\|_2^2\right)  \right) \ge \alpha d^4  \right\} \nonumber \\
&= \mathbb{P} \left\{    1+ \|\bar{X}^{-1}\|^2\cdot\|v\|_2^2  \ge \alpha d^4 \right\} \nonumber \\
&\le \mathbb{P} \left\{    \|\bar{X}^{-1}\|^2\cdot\|v\|_2^2  \ge \alpha d^4/2 \right\} \nonumber \\
&= \mathbb{P} \left\{    \|\bar{X}^{-1}\|\cdot\|v\|_2  \ge \sqrt{\alpha/2}\cdot d^2 \right\} \label{eq:inverse-norm-split-case} \;.
\end{align}
To upper bound Eq.~\eqref{eq:inverse-norm-split-case}, we use the fact that $\bar{X}^{-1}$ and $v$ are independent, and split the event into two cases: $\{\|v\|_2 \ge \sqrt{d}/2\}$ and $\{\|v\|_2 < \sqrt{d}/2\}$. By~\cite[Theorem 3.1.1]{vershynin2018high}, we know that there exists a constant $C_1 > 0$ such that
\begin{align*}
    \mathbb{P}\left\{\|v\|_2 < \sqrt{d}/2\right\} \le \exp(-C_1 \cdot d)\;.
\end{align*}
Moreover, by~\cite[Theorem 1.2]{szarek1991condition}, we have that for sufficiently large $d$, there exists a universal constant $C_2 > 0$ such that for any $t > 0$,
\begin{align*}
    \mathbb{P}\left\{ \|\bar{X}^{-1}\| \ge t \sqrt{d} \right\} \leq C_2/t\;.
\end{align*}
By setting $\alpha = 2C_2^2$ and $d$ sufficiently large so that $\exp(-C_1 d) \le 1/(2d)$, we have
\begin{align}
\mathbb{P} \Big\{    \|\bar{X}^{-1}\|\cdot\|v\|_2  &\ge \sqrt{\alpha/2}\cdot d^2 \Big\} \nonumber \\
&\le \mathbb{P} \left\{    \|\bar{X}^{-1}\|\cdot\sqrt{d}/2  \ge \sqrt{\alpha/2}\cdot d^2 \right\} 
\cdot \mathbb{P} \left\{\|v\|_2 > \sqrt{d}/2\right\} + \mathbb{P}\left\{\|v\|_2 \le \sqrt{d}/2\right\} \nonumber \\
&\le \mathbb{P} \left\{    \|\bar{X}^{-1}\|\cdot\sqrt{d}/2  \ge \sqrt{\alpha/2}\cdot d^2 \right\} + \exp(-C_1 d) \nonumber \\
&= \mathbb{P} \left\{    \|\bar{X}^{-1}\|  \ge \sqrt{2\alpha}\cdot d^{3/2} \right\} + \exp(-C_1 d) \nonumber \\
&\le C_2/(\sqrt{2\alpha} \cdot d) + \exp(-C_1 d) \nonumber \\
&\le 1/d \nonumber \;.
\end{align}
Therefore,
\begin{align*}
\mathbb{P}\{ \eta \ge \sqrt{2}C_2 \cdot d^2\} = \mathbb{P}\{ \eta^2 \ge 2C_2^2 \cdot d^4\} \le \mathbb{P} \Big\{    \|\bar{X}^{-1}\|\cdot\|v\|_2  \ge C_2\cdot d^2 \Big\} \le 1/d \;.
\end{align*}
\end{proof}

\section{Auxiliary Results}
\label{aux_app}

\subsection{The Periodic Gaussian}

\begin{definition}
\label{def:periodic-Gaussian}
Let $\Psi_s(z) : [-1/2,1/2) \rightarrow \RR_+$ be the periodic Gaussian density function defined by
\begin{align}
    \Psi_s(z) := \sum_{k=-\infty}^\infty \frac{1}{s\sqrt{2\pi}} \exp\Bigg(-\frac{1}{2} \Big(\frac{z-k}{s}\Big)^2\Bigg)\;. \nonumber
\end{align}
We refer to the parameter $s$, the standard deviation of the Gaussian before periodicization, as the ``width'' of the periodic Gaussian $\Psi_s$.
\end{definition}

\begin{remark}
\label{rem:periodic-gaussian}
For intuition, we can consider two extreme settings of the width $s$. If $s \ll 1$, then $\Psi_s$ is close in total variation distance to the Gaussian of standard deviation $s$ since the tails outside $[-1/2,1/2)$ will be very light. On the other hand, if $s \gg 1$, then $\Psi_s$ is close in total variation distance to the uniform distribution on $[0,1)$. This intuition is formalized in Claim~\ref{claim:gaussian-mod1}.
\end{remark}

The Gaussian distribution on $\RR$ satisfies the following tail bound called Mill's inequality. 
\begin{lemma}[Mill's inequality {\cite[Proposition 2.1.2]{vershynin2018high}}]
\label{lem:mill}
Let $z \sim N(0,1)$. Then for all $t > 0$, we have
\begin{align*}
    \PP(|z| \ge t) = \sqrt{\frac{2}{\pi}}\int_t^\infty e^{-x^2/2}dx \le \frac{1}{t}\cdot \sqrt{\frac{2}{\pi}} e^{-t^2/2}\;.
\end{align*}
\end{lemma}

The Poisson summation formula, stated in Lemma~\ref{lem:poisson-sum} below, will be useful in our calculations. We first define the dual of a lattice $\Lambda$ to make the formula easier to state.
\begin{definition}
\label{def:dual-lattice}
The dual lattice of a lattice $\Lambda$, denoted by $\Lambda^*$, is defined as
\begin{align*}
    \Lambda^* = \{ y \in \mathbb{R}^d \mid \inner{x,y} \in \mathbb{Z}\; \text{ for all } x \in \Lambda\}
    \; .
\end{align*}\end{definition}
A key property of the dual lattice is that if $B$ is a basis of $\Lambda$ then $(B^T)^{-1}$ is a basis of $\Lambda^*$; in particular, $\det(\Lambda^*) = \det(\Lambda)^{-1}$, where $\det(\Lambda)$ is defined as $\det(\Lambda) = \det(B)$ (the determinant of a lattice is basis-independent)~\cite[Chapter 1]{micciancio2002complexity}.

For ``nice'' functions defined any lattice, the following formula holds~\cite[Theorem 2.3]{ebeling1994lattices}.
\begin{lemma}[Poisson summation formula] For any lattice $\Lambda \subset \RR^d$ and any function $f: \RR^d \rightarrow \mathbb{C}$ satisfying some ``niceness'' assumptions\footnote{For our purposes, it suffices to know that the Gaussian function of any variance $s > 0$ satisfies this niceness assumption. Precise conditions can be found in~\cite[Theorem 2.3]{ebeling1994lattices}.},
\label{lem:poisson-sum}
\begin{align*}
    \sum_{x\in \Lambda}f(x) = \det(\Lambda^*)\cdot \sum_{y \in \Lambda^*}\widehat{f}(y)
    \;,
\end{align*}
where $\widehat{f}(y)= \int_{\RR^d} f(x)e^{-2\pi i \inner{y,x}}dx$, and $\Lambda^*$ is the dual lattice of $\Lambda$.
\end{lemma}

Note that by the properties of the Fourier transform, for a fixed $c \in \RR^d$
\begin{align*}
    \sum_{x \in \Lambda + c} f(x) = \sum_{x \in \Lambda}f(x+c) = \det(\Lambda^*) \sum_{y \in \Lambda^*} \exp(-2\pi i \inner{c,y})\cdot \widehat{f}(y)\;.
\end{align*}

\begin{claim}[{Adapted from \cite[Claim 2.8.1]{NSDthesis}}]
\label{claim:gaussian-mod1}
For any $s > 0$ and any $z \in [-1/2,1/2)$ the periodic Gaussian density function $\Psi_s(z)$ satisfies
\begin{align}
    \Psi_s(z) \le \frac{1}{s\sqrt{2\pi}} \left(1+2(1+s^2)e^{-1/(2s^2)}\right)\;. \nonumber
\end{align}
and
\begin{align}
    |\Psi_s(z) - 1| \le 2(1+1/(4\pi s)^2)e^{-2\pi^2s^2}\;. \nonumber
\end{align}
\end{claim}

\begin{proof}
We first derive an expression for $\Psi_s(0)$ using the Poisson summation formula. Note that the Fourier transform of $f(y) = \exp(-y^2/2)$ is given by $\widehat{f}(u) = \sqrt{2\pi}\cdot \exp(-2\pi^2u^2)$. Moreover, viewing $\ZZ$ as a one-dimensional lattice, the determinant of the dual lattice $((1/s)\ZZ)^* = s\ZZ$ is $s$. Hence,
\begin{align}
    \Psi_s(0) &= \frac{1}{s\sqrt{2\pi}} \sum_{y \in (1/s)\ZZ} \exp(-y^2/2) \nonumber \\
    &= \frac{\det(s\ZZ)\sqrt{2\pi}}{s\sqrt{2\pi}} \cdot \sum_{u \in s\ZZ} \exp(-2\pi^2u^2) \nonumber \\
    &= \sum_{u \in s\ZZ} \exp(-2\pi^2u^2) \;. \label{eq:centered-periodic-gaussian}
\end{align}
We now observe that $\Psi_s(z) \le \Psi_s(0)$ for any $z \in [-1/2,1/2)$. This can again be shown using the Poisson summation formula as follows.
\begin{align*}
    \Psi_s(z) &= \frac{1}{s\sqrt{2\pi}} \sum_{y \in (1/s)\ZZ+z/s} \exp(-y^2/2) \\
    &= \sum_{u \in s\ZZ} \exp(-2\pi i uz/s) \cdot \exp(-2\pi^2u^2) \\
    &\le \sum_{u \in s\ZZ} |\exp(-2\pi i uz/s)| \cdot \exp(-2\pi^2u^2) \\
    &\le \sum_{u \in s\ZZ} \exp(-2\pi^2u^2) \\
    &= \Psi_s(0)\;.
\end{align*}
Hence, it suffices to upper bound $\Psi_s(0)$ and show a lower bound for $\Psi_s(z)$ for all $z \in [-1/2,1/2)$. For the first upper bound, we use Mill's inequality (Lemma~\ref{lem:mill}) to obtain
\begin{align*}
    \Psi_s(0) &= \frac{1}{s\sqrt{2\pi}} \sum_{y \in (1/s)\ZZ} \exp(-y^2/2) \\
    &\le \frac{1}{s\sqrt{2\pi}} \left(1 + 2\exp(-1/(2s^2)) + 2 \int_1^\infty \exp(-x^2/(2s^2))dx \right) \\
    &\le \frac{1}{s\sqrt{2\pi}} \left(1 + 2(1+s^2)\exp(-1/(2s^2))\right)\;.
\end{align*}

For the second upper bound, we use Eq.~\eqref{eq:centered-periodic-gaussian} and Mill's inequality to obtain
\begin{align*}
    \Psi_s(0) &= \sum_{u \in s\ZZ} \exp(-2\pi^2 u^2) \\
    &= 1+ \sum_{u \in s\ZZ\setminus\{0\}} \exp(-2\pi^2 u^2) \\
    &= 1 + 2\sum_{k = 1}^\infty \exp(-2\pi^2 s^2k^2) \\
    &\le 1 + 2\exp(-2\pi^2 s^2) + 2\int_{1}^\infty \exp(-2\pi^2s^2 x^2)dx \\ 
    &\le 1 + 2(1 + 1/(4\pi s)^2) \exp(-2\pi^2 s^2)\;.
\end{align*}
For the lower bound on $\Psi_s(z)$, we use the Poisson summation formula and Mill's inequality again to obtain
\begin{align*}
    \Psi_s(z) &= \sum_{u \in s\ZZ}\exp(-2\pi i zu/s)\cdot \exp(-2\pi^2 u^2) \\
    &= 1+ \sum_{u \in s\ZZ\setminus\{0\}}\exp(-2\pi i zu/s)\cdot \exp(-2\pi^2 u^2) \\
    &\ge 1 - 2\sum_{k = 1}^\infty |\exp(-2\pi i z k)| \cdot \exp(-2\pi^2 s^2k^2) \\
    &\ge 1 - 2\left(\exp(-2\pi^2 s^2) + \int_{1}^\infty \exp(-2\pi^2s^2 x^2)dx\right) \\ 
    &\ge 1 - 2(1+ 1/(4\pi s)^2) \exp(-2\pi^2 s^2)\;.
\end{align*}
\end{proof}

\subsection{Auxiliary Lemmas for the Constant Noise Regime}

\begin{lemma}\label{lem:arccos}
 Fix some $\tau \in (0,1]$. Then, for $\arccos :[-1,1] \rightarrow [0,\pi]$ it holds that 
 \begin{align*}
     \sup_{x,y \in [-1,1], |x-y| \leq \tau} |\arccos(x)-\arccos(y)| \leq \arccos(1-\tau).
 \end{align*}
\end{lemma}

\begin{proof}
Let us fix some arbitrary $\xi \in [0,\tau]$ and consider the function $G(x)=\arccos(x)-\arccos(x+\xi).$ Given the fact that $\arccos$ is decreasing, it suffices to show that $|G(x)| \leq \arccos(1-\tau)$ for all $x \in [-1,1-\xi].$ By direct computation it holds 
\begin{align*}G'(x)&=-\frac{1}{\sqrt{1-x^2}}+\frac{1}{\sqrt{1-(x+\xi)^2}}\\
&=\frac{\xi (2x+\xi)}{\sqrt{1-x^2}\sqrt{1-(x+\xi)^2}(\sqrt{1-x^2}+\sqrt{1-(x+\xi)^2})}.
\end{align*}
Hence, the function $G$ decreases until $x=-\xi/2$ and increases beyond this point. Consequently, $G$ obtains its global maximum at one the endpoints of $[-1,1-\xi].$ But since $\cos(\pi-a)=-\cos(a)$ for all $a \in \mathbb{R}$ it also holds for all $b \in [-1,1]$ $\arccos(-b)+\arccos(b)=\pi$. Hence, \begin{align*}G(-1)=\pi-\arccos(-1+\xi)=\arccos(1-\xi)=G(1-\xi).\end{align*} Therefore, \begin{align*} G(x) \leq \arccos(1-\xi) \leq \arccos(1-\tau).\end{align*}The proof is complete.
\end{proof}

\subsection{Auxiliary Lemmas for the Exponentially Small Noise Regime}\label{app_aux_LLL}

\begin{lemma}\label{lem:trunc_2}[Restated Lemma \ref{lem:trunc}]
Suppose $n \leq C_0d$ for some constant $C_0>0$ and $s \in \mathbb{R}^n$ satisfies for some $m \in \mathbb{Z}^n$ that $|\inner{m,s}| = \exp(-\Omega( (d \log d)^3))$. Then for some sufficiently large constant $C>0$, if $N=\lceil d^3 (\log d)^2 \rceil$ there is an  $m' \in \mathbb{Z}^{n+1}$ which is equal with $m$ in the first $n$ coordinates, satisfies $\|m'\|_2 \leq C d^{\frac{1}{2}}\|m\|_2$ and is an integer relation for the $(s_1)_N,\ldots,(s_n)_N, 2^{-N}.$ 
\end{lemma}

\begin{proof}
We start with noticing that since $N=o((d \log d)^3)$ we have
\begin{align*}
    |\inner{m,s}| \leq \exp(-\Omega( (d \log d)^3))=O(2^{-N})\;.    
\end{align*}
Hence, since for any real number $x$ we have $|x-(x)_N| \leq 2^{-N},$ it holds
\begin{align*}
    \sum_{i=1}^n m_i (s_i)_N &=\sum_{i=1}^n m_i s_i +O(\sum_{i=1}^n m_i 2^{-N})\\
    &=O(2^{-N}) +O(\sum_{i=1}^n |m_i| 2^{-N})\\
    &=O(\sum_{i=1}^n|m_i| 2^{-N}).
\end{align*} 
Now observe that the number $\sum_{i=1}^n m_i (s_i)_N$ is a rational number of the form $a/2^N, a \in \mathbb{Z}.$ Hence using the last displayed equation we can choose some integer $m'_{n+1}$ with
\begin{align*}
    \sum_{i=1}^n m_i (s_i)_N =m'_{n+1} 2^{-N}.
\end{align*} 
for which using Cauchy-Schwartz and $n=O(d)$ it holds
\begin{align*}
    |m'_{n+1}|=O(\|m\|_1)=O(\sqrt{n}\|m\|_2)=O(\sqrt{d}\|m\|_2).
\end{align*} 
Hence $m'=(m_1,\ldots,m_n,-m'_{n+1})$ is an integer relation for $(s_1)_N, \ldots (s_n)_N, 2^{-N}$. On top of that 
\begin{align*}
    \|m'\|^2_2 \leq \|m\|^2_2+O(d\|m\|^2_2)=O(d\|m\|^2_2).
\end{align*}
This completes the proof.
\end{proof}

\begin{lemma}[Restated Lemma \ref{lem:bounds}]\label{lem:bounds_v2}
 Suppose that  $\gamma \leq d^{Q}$ for some $Q>0$. For some hidden direction $w \in S^{d-1}$ we observe $d+1$ samples of the form $(x_i,z_i), i=1,\ldots,d+1$ where for each $i$, $x_i$ is a sample from  $N(0, I_d)$ samples, and \begin{align*} z_i=\cos(2 \pi  (\gamma \langle w,x_i \rangle)) +\xi_i,\end{align*} for some unknown and arbitrary $\xi_i \in \mathbb{R}$ satisfying $|\xi_i| \leq \exp(-  (d \log d)^3).$ Denote by $X \in \mathbb{R}^{d \times d}$ the random matrix with columns given by the $d$ vectors $x_2,\ldots,x_{d+1}$. With probability $1-\exp(-\Omega(d))$ the following properties hold.
\begin{itemize}

\item[(1)] \begin{align*} \max_{i=1,\ldots,d+1} \|x_i\|_2 \leq 10\sqrt{d}.
\end{align*}
    \item[(2)] \begin{align*}
    \min_{i=1,\ldots,d+1} |\sin(2\pi \gamma \langle x_i,w \rangle)| \geq 2^{-d}.
    \end{align*}
    \item[(3)] For all $i=1,\ldots,d+1$ it holds $z_i \in [-1,1]$ and \begin{align*}
   z_i=\cos (2 \pi( \gamma \langle x_i, w\rangle+\xi'_i)),
     \end{align*}
     for some $\xi'_i \in \mathbb{R}$ with $|\xi'_i|=\exp(-\Omega( (d \log d)^3)).$ 
    \item[(4)] The matrix $X$ is invertible. Furthermore,  
    \begin{align*}
    \|X^{-1}x_1\|_{\infty} =O( 2^{\frac{d}{2}}\sqrt{d}).
    \end{align*}
    \item[(5)] 
    \begin{align*}
    0<|\mathrm{det}(X)|=O(\exp(d\log d)).
    \end{align*}

\end{itemize}
\end{lemma}

\begin{proof}

For the first part, notice that for each $i=1,2,\ldots,d+1,$ the quantity $\|x_i\|^2_2$ is distributed like a $\chi^2(d)$ distribution with $d$ degrees of freedom. Using standard results on the tail of the $\chi^2$ distribution (see e.g. \cite[Chapter 2]{wainwright_2019}) we have for each $i,$ 
\begin{align*}
\mathbb{P}\left( \|x_1\|_2 \geq 10\sqrt{d} \right)= \exp(-\Omega(d)).
\end{align*} Hence,
\begin{align*}
    \mathbb{P}\left( \bigcup_{i=1}^{d+1} \|x_i\|_2 \geq 10\sqrt{d} \right) \leq (d+1)    \mathbb{P}\left( \|x_1\|_2 \geq 10\sqrt{d} \right) =O(d \exp^{-\Omega(d)})=\exp(-\Omega(d)),
\end{align*}

For the second part, first notice that for large $d$ the following holds: if for some $\alpha \in \mathbb{R}$ we have $|\sin (\alpha)| \leq 2^{-d}$ then for some integer $k$ it holds $|\alpha-k\pi| \leq 2^{-d+1} .$ Indeed, by substracting an appropriate integer multiple of $\pi$ we have $\alpha-k \pi \in [-\pi/2,\pi/2].$ Now by applying the mean value theorem for the branch of arcsin defined with range $[-\pi/2,\pi/2]$ we have that \begin{align*}
|\alpha- k \pi|=|\arcsin(\sin \alpha)-\arcsin(0)| \leq \frac{1}{\sqrt{1-\xi^2}} |\sin \alpha| \leq \frac{1}{1-\xi^2} 2^{-d} \end{align*} for some $\xi $ with $|\xi| \leq |\sin \alpha| \leq  2^{-d}.$ Hence, using the bound on $\xi$ we have
\begin{align*}
    |\alpha-k \pi| \leq \frac{1}{1-2^{-2d}} 2^{-d} \leq 2^{-d+1}\;.    
\end{align*}

Using the above observation, we have that if for some $i$ it holds $|\sin(2\pi \gamma \langle x_i,w \rangle)| \leq 2^{-d}$ then for some integer $k \in \mathbb{Z}$ it holds $|\langle x_i, w \rangle - \frac{k}{2\gamma}| \leq \frac{1}{\gamma}2^{-d}.$ Furthermore, since by Cauchy-Schwartz and the first part with probability $1-\exp(-\Omega(d))$ we have \begin{align*} |\langle x_i, w \rangle| \leq \|x_i\| \leq 10 \sqrt{d},\end{align*}  it suffices to consider only the integers $k$ satisfying $|k| \leq 10 \gamma \sqrt{d},$ with probability $1-\exp(-\Omega(d))$.
Hence, 
\begin{align*}
    \mathbb{P}\left( \bigcup_{i=1}^{d+1} |\sin (2\pi \gamma \langle x_i,w \rangle)| \leq 2^{-d} \right) &\leq \mathbb{P}\left( \bigcup_{i=1}^{d+1} \bigcup_{k: |k| \leq 10 \gamma \sqrt{d}} | \langle x_i,w \rangle - \frac{k}{2\gamma}| \leq \frac{1}{\gamma}2^{-d} \right)\\
    & \leq 20 d \sqrt{d} \gamma  \sup_{k \in \mathbb{Z}}\mathbb{P}\left( | \langle x_1,w \rangle - k/2\gamma| \leq \frac{1}{\gamma}2^{-d} \right) \\
    & \leq 40 d \sqrt{d}  2^{-d}\\
    &=\exp(-\Omega(d)),
\end{align*}where we used the fact that $\langle x_1,w \rangle$ is distributed as a standard Gaussian, and that for a standard Gaussian $Z$ and for any interval $I$ of any interval of length $t$ it holds $\mathbb{P}(Z \in I) \leq \frac{1}{\sqrt{2\pi}} t \leq t$.

For the third part, notice that from the second part for all $i=1,\ldots,d+1$ it holds \begin{align*}1-\cos^2 (2\pi \gamma \langle x_i,w \rangle)=\sin^2 (2\pi \gamma \langle x_i,w \rangle)=\Omega(2^{-2d})\end{align*} with probability $1-\exp(-\Omega(d)).$ Hence, since $\|\xi\|_{\infty} \leq \exp(- (d \log d)^3)$ we have that for all $i=1,\ldots,d+1$ it holds \begin{align*} z_i=\cos (2 \pi \gamma \langle x_i, w\rangle))+\xi_i \in [-1,1],\end{align*}with probability $1-\exp(-\Omega(d)).$ Hence,  the existence of $\xi'_i$ follows by the fact that image of the cosine is the interval $[-1,1]$.  Now by mean value theorem we have
\begin{align*}\xi_i=\cos (2 \pi( \gamma \langle x_i, w\rangle+\xi'_i))-\cos (2 \pi \gamma \langle x_i, w\rangle))= 2\pi \gamma \xi'_i \sin(2\pi \gamma t) \end{align*} for some $t \in (\langle x_i, w\rangle-|\xi_i|,\langle x_i, w\rangle+|\xi_i|).$ By the 1-Lipschitzness of the sine function, the second part and the exponential upper bound on the noise we can immediately conclude \begin{align*} |\sin(2\pi \gamma t)| \geq \sin(2\pi \gamma \inner{x_i,w})-|\xi_i| =\Omega(2^{-d}),\end{align*} with probability $1-\exp(-\Omega(d)).$ Hence it holds $|\xi'_i| \Omega(2^{-d}) \leq |\xi_i|$ and therefore
\begin{align*}|\xi'_i|  \leq 2^{d} |\xi_i|=\exp(-\Omega( (d \log d))^3)\end{align*}with probability $1-\exp(-\Omega(d)).$

For the fourth part, for the fact that $X$ is invertible, consider its determinant, that is the random variable $\det(X)$. The determinant is non-zero almost surely, i.e. $\det(X) \not =0 $ almost surely. This follows from the fact that the determinant is a non-zero polynomial of the entries of $X$, e.g. for $X=I_d$ it equals one, hence, using folklore results as all entries of $X$ are i.i.d. standard Gaussian it is almost surely non-zero \cite{caron2005zero}. Now, using standard results on the extreme singular values of $X$, such as \cite[Equation (3.2)]{rudelson_ICM}, we have that $\sigma_{\max}(X^{-1})=1/\sigma_{\min}(X) \leq 2^{d},$ with probability $1-\exp(-\Omega(d)).$ In particular, using also the first part, it holds \begin{align*}\|X^{-1}x_1\|_{\infty} \leq \|X^{-1}x_1\|_2 \leq \sqrt{\sigma_{\max}(X^{-1})} \|x_1\|_2 \leq 2^{\frac{d}{2}} \sqrt{d},\end{align*}with probability $1-\exp(-\Omega(d)).$

For the fifth part, notice that the determinant is non-zero from the fourth part.

For the upper bound on the determinant, we apply Hadamard's inequality \cite{hadamard_ineq} and part 1 of the Lemma to get that
\begin{align*} |\mathrm{det}(x_2,\ldots,x_{d+1})| \leq \prod_{i=2}^{d+1} \|x_i\|_2 \leq (10\sqrt{d})^d= O(\exp(d\log d)),
\end{align*}with probability $1-\exp(-\Omega(d)).$

\end{proof}

\subsection{Auxiliary Lemmas for the Population Loss}

Fix some hidden direction $w \in S^{d-1}.$
Recall that for any $w' \in S^{d-1}$, we denote by 
\begin{align*}
    L(w')=\mathbb{E}_{x \sim N(0,I_d)}[(\cos(2\pi \gamma \inner{w,x})-\cos(2\pi \gamma \inner{w',x}))^2]\;.
\end{align*}

\begin{lemma}\label{lem:hermite}
Let us consider the (probabilist's) normalized Hermite polynomials on the real line $\{h_k\}_{k \in \ZZ_{\ge 0}}$. The following identities hold for $Z \sim N(0,1)$.
\begin{itemize}
    \item[(1)] For all  $k, \ell \in \mathbb{Z}_{\geq 0}$
    \begin{align*}
        \mathbb{E}[h_k(Z)h_{\ell}( Z)]=\one[k=\ell]\;.
    \end{align*}
    \item[(2)] Let $Z_{\rho}$ be a standard Gaussian which is $\rho$-correlated with $Z$. Then, for all $\gamma>0, k \in \mathbb{Z}_{\geq 0}$, 
    \begin{align*}
        \mathbb{E}[h_k(Z)\cos(2\pi \gamma Z_{\rho})]=(-1)^{k/2} \rho^k\frac{(2\pi \gamma)^{k}}{\sqrt{k!}} \exp(-2 \pi^2 \gamma^2)\cdot \one[k \in 2 \mathbb{Z}_{\geq 0}]\;.
    \end{align*}
    \item[(3)]The performance of the trivial estimator, which always predicts 0, equals
    \begin{align*}
        \mathrm{Var}(\cos(2\pi \gamma Z))&=\sum_{k \in 2 \mathbb{Z}_{\geq 0}\setminus \{0\}} \frac{(2\pi \gamma)^{2k}}{k!} \exp(-4 \pi^2 \gamma^2)=\frac{1}{2}+O(\exp(-\Omega(\gamma^2)))\;.
    \end{align*}
\end{itemize}
\end{lemma}

\begin{proof}
The first part follows from the standard property that the family of normalized Hermite polynomials form a complete orthonormal basis of $L^2(N(0,1))$~\cite[Proposition B.2]{kunisky2019notes}.

For the second part, recall the basic fact that we can set $Z_{\rho}=\rho Z+\sqrt{1-\rho^2}W$ for some $W$ standard Gaussian independent from $Z$. Using~\cite[Proposition 2.10]{kunisky2019notes}, we get
\begin{align*}
    \mathbb{E}[h_k(Z)\cos(2\pi \gamma Z_{\rho})]& = \mathbb{E}[h_k(Z)\cos(2\pi \gamma (\rho Z+\sqrt{1-\rho^2}W)]\\
    &=\frac{1}{\sqrt{k!}}\mathbb{E}\left[\frac{d^k}{dZ^k}\cos(2\pi \gamma (\rho Z+\sqrt{1-\rho^2}W)\right]\\
    &=(-1)^{k/2}(2\pi \rho \gamma)^k \frac{1}{\sqrt{k!}} \mathbb{E}[\cos(2\pi \gamma (\rho Z+\sqrt{1-\rho^2}W)]\cdot \one(k \in 2 \mathbb{Z}_{\geq 0})\\
    &\quad +(-1)^{(k+1)/2}(2\pi \rho \gamma)^k \frac{1}{\sqrt{k!}} \mathbb{E}[\sin(2\pi \gamma (\rho Z+\sqrt{1-\rho^2}W)]\cdot \one(k \not \in 2 \mathbb{Z}_{\geq 0})\\
    &=(-1)^{k/2}(2\pi \rho \gamma)^k \frac{1}{\sqrt{k!}} \mathbb{E}[\cos(2\pi \gamma (\rho Z+\sqrt{1-\rho^2}W)]\cdot \one(k \in 2 \mathbb{Z}_{\geq 0})\\
    &=(-1)^{k/2}(2\pi \rho \gamma)^k \frac{1}{\sqrt{k!}} \mathbb{E}[\cos(2\pi \gamma Z)]\cdot \one(k \in 2 \mathbb{Z}_{\geq 0})\\
    &=(-1)^{k/2}(2\pi \rho \gamma)^k \frac{1}{\sqrt{k!}} \exp(-2\pi^2 \gamma^2)\cdot \one(k \in 2 \mathbb{Z}_{\geq 0})\;,
\end{align*}
where (a) in the third to last line we used that the $\sin$ is an odd function and therefore when $k$ is odd the corresponding term is zero, (b) in the second to last line we used that $Z_\rho$ follows the same standard Gaussian law as $Z$ and, (c) in the last line we used the characteristic function of the standard Gaussian to conclude that for any $t>0$,
\begin{align*}
    \mathbb{E}[\cos(tZ)]=\mathrm{Re}[\mathbb{E}[e^{itZ}]]=e^{-t^2/2}\;.    
\end{align*}

For the third part, notice that by applying the result from part (1) and the result from part (2) (for $\rho =1)$ it holds,
\begin{align*}
     \mathrm{Var}(\cos(2\pi \gamma Z))&=\sum_{k \in  \mathbb{Z}_{\geq 0}\setminus\{0\}} \mathbb{E}[\cos(2\pi \gamma Z)h_k(Z)]^2\\
     &=\sum_{k \in 2 \mathbb{Z}_{\geq 0}\setminus \{0\}} \frac{(2\pi \gamma)^{2k}}{k!} \exp(-4 \pi^2 \gamma^2)\\
     &=\sum_{k \in 2 \mathbb{Z}_{\geq 0}} \frac{(2\pi \gamma)^{2k}}{k!} \exp(-4 \pi^2 \gamma^2)-\exp(-4 \pi^2 \gamma^2)\\
     &=\sum_{k \geq 0} \frac{1}{2}\cdot \frac{(2\pi \gamma)^{2k}}{k!} \exp(-4 \pi^2 \gamma^2) (1+(-1)^k)-\exp(-4 \pi^2 \gamma^2)\\
       &= \frac{1}{2}\left(\sum_{k \geq 0} \frac{(4\pi^2 \gamma^2)^{k}}{k!} \exp(-4 \pi^2 \gamma^2) +\sum_{k \geq 0} \frac{(-4\pi^2 \gamma^2)^{k}}{k!} \exp(-4 \pi^2 \gamma^2)\right) -\exp(-4 \pi^2 \gamma^2)\\
       &=\frac{1}{2}+\frac{1}{2}\exp(-8\pi^2\gamma^2)-\exp(-4 \pi^2 \gamma^2)\\
       &=\frac{1}{2}+O(\exp(-\Omega(\gamma^2)))\;.
     \end{align*}
\end{proof}

\end{document}